\RequirePackage[svgnames,table]{xcolor}
\PassOptionsToPackage{numbers,sort&compress}{natbib}
\documentclass[11pt,letterpaper,logo]{yalearxiv}

\usepackage{mathtools}
\usepackage{nicefrac}
\usepackage{cancel}
\usepackage{subcaption}
\usepackage{multirow}
\usepackage{algorithm}
\usepackage{algpseudocode}
\usepackage{comment}
\usepackage{natbib}

\setcitestyle{numbers,square,sort&compress}
\hypersetup{
  colorlinks=true,
  linkcolor=blue!50!black,
  citecolor=blue!50!black,
  urlcolor=blue!50!black,
  hyperfootnotes=false
}

\newtheorem{theorem}{Theorem}
\theoremstyle{remark}

\newtcbtheorem[number within=section]{thm}{Theorem}
{colback=blue!2,colframe=blue!50!black,fonttitle=\bfseries,coltitle=white,boxrule=0.5pt}{thm}
\newtcbtheorem[use counter from=thm]{lem}{Lemma}
{colback=gray!4,colframe=gray!60!black,fonttitle=\bfseries,coltitle=white,boxrule=0.5pt}{lem}
\newtcbtheorem[use counter from=thm]{prop}{Proposition}
{colback=purple!2,colframe=purple!60!black,fonttitle=\bfseries,coltitle=white,boxrule=0.5pt}{prop}
\newtcbtheorem[use counter from=thm]{defn}{Definition}
{colback=green!2,colframe=green!40!black,fonttitle=\bfseries,coltitle=white,boxrule=0.5pt}{defn}
\newtcbtheorem[use counter from=thm]{asm}{Assumption}
{colback=orange!2,colframe=orange!60!black,fonttitle=\bfseries,coltitle=white,boxrule=0.5pt}{asm}

\makeatletter
\newcommand{\LongState}[1]{%
  \State \begin{minipage}[t]{\dimexpr\linewidth-\ALG@tlm\relax}
  #1
  \end{minipage}%
}
\makeatother

\newcommand{\answerTODO}[1][]{\textcolor{red}{\bfseries [TODO]}}
\newcommand{\justificationTODO}[1][]{\textcolor{red}{\bfseries [TODO]}}

\title{RTPO: Reverse-Turn Policy Optimization for\\Stabilizing Agentic RL Training}
\runningtitle{RTPO: Reverse-Turn Policy Optimization}
\keywords{multi-turn reinforcement learning, agentic workflows, tool-using agents, turn-level credit assignment}

\author{%
Yugu Li$^{1}$, Zehong Cao$^{1,*}$, Jianglin Qiao$^{2}$, and Siyi Hu$^{3}$\\
$^{1}$School of CSIT, Adelaide University, Adelaide, SA 5000, Australia\\
$^{2}$ACFR, The University of Sydney, Camperdown, NSW 2050, Australia\\
$^{3}$School of EECMS, Curtin University, Bentley, WA 6102, Australia\\
\texttt{yugu.li@adelaide.edu.au};
\texttt{jimmy.cao@adelaide.edu.au};
\texttt{jianglin.qiao@sydney.edu.au};
\texttt{siyi.hu@curtin.edu.au}
}
\correspondingauthor{Zehong Cao (\texttt{jimmy.cao@adelaide.edu.au})}
\begin{document}

\begin{abstract}
\vspace{-1mm}
{\centering\section*{Abstract}\par}
Training multi-turn agentic workflows with reinforcement learning (RL) enables large language models to perform complex reasoning, use external tools, and conduct iterative search beyond single-turn settings. Yet multi-turn RL training remains highly unstable, often causing severe performance degradation as the number of turns increases. Through theoretical analysis, we identify three tightly coupled sources of instability: rollout–training context mismatch, weak turn-level credit assignment under sparse terminal rewards, and asynchronous policy drift when short and long trajectories are optimized under different policy versions. We show that these issues share a common structural origin in flattened trajectory optimization and address them through a unified reverse-turn formulation. We propose Reverse-Turn Policy Optimization (RTPO), which organizes multi-turn rollouts as sparse reverse trees and performs turn-level policy updates in temporal reverse order, aligning each decision with its downstream continuation. RTPO enables causally consistent turn-level credit assignment and on-policy continuation to control asynchronous drift. We provide theoretical guarantees showing that RTPO eliminates context mismatch and asynchronous drift under the proposed turn-level formulation, reduces credit bias, and converges to recursive optimality. Experiments on multi-turn agentic RL benchmarks show that RTPO improves upon trajectory- and turn-level baselines by 21.50\% and 10.76\%, respectively, highlighting its potential to support more stable training for tool-using agents.
\end{abstract}

\maketitle

\section{Introduction}\label{introduction}

Reinforcement learning (RL) has become a central paradigm for post-training large language models (LLMs), especially when supervision comes from final outcome rewards rather than dense token-level labels. Through outcome-based optimization, RL encourages behaviors such as planning, self-reflection, and verification, achieving strong results in single-turn mathematical reasoning~\citep{shao2024deepseekmath,yang2024qwen2,guo2025deepseek,li2026ssvpo} and code generation~\citep{le2022coderl,shojaee2023executionbased,team2025kimi}. Motivated by these advances, recent work extends RL to multi-turn agentic workflows, especially Tool-Integrated Reasoning (TIR)~\citep{yao2023react,gao2023pal,shinn2023reflexion,chen2026toward,chang2026thor,liu2026roga}, where agents iteratively reason, call tools, receive feedback, and refine their behavior across turns.

Despite this promise, \textbf{RL training for multi-turn workflows remains unstable}: models that improve on short workflows often degrade as the number of turns increases. This degradation is not merely due to longer sequences; it is amplified by temporal dependencies across turns, where each decision reshapes the context, environment state, and future decision distribution. Existing methods partially mitigate this issue but provide limited analysis of its causes. Trajectory-level methods such as PPO~\citep{schulman2017proximal}, GRPO~\citep{shao2024deepseekmath}, and GSPO \citep{zheng2025group} retain the flattened trajectory paradigm. Turn-decomposition methods such as SeeUPO~\citep{hu2026seeupo} improve update granularity but still condition on flattened histories. Tree-based methods such as TreeGRPO~\citep{ji2026tree} and ARPO~\citep{dong2026agentic} use shared prefixes or branching rollouts, but their advantage estimates are not fully aligned with each turn’s causal contribution to downstream continuation. Overall, the sources of instability remain underexplored at the training-pipeline level, and existing methods (see \textbf{Appendix~\ref{related work}} for details) do not jointly address their shared origin.

In this paper, \textbf{we provide a theoretical analysis that identifies three coupled sources of instability}: context mismatch between rollout and training, which breaks consistency between generated and optimized turn-level contexts; weak turn-level credit assignment, where sparse terminal rewards obscure the contribution of individual decisions; and asynchronous policy drift, where short and long trajectories are optimized under different versions of an evolving policy. Although these issues arise from different components of the training pipeline, we show that they share a common structural origin. We illustrate these sources in \textbf{Figure~\ref{fig:1}} and analyze them formally in Sec.~\ref{theory}. 
\vspace{-3pt}

\begin{center}

\setlength{\fboxsep}{4pt}
\fbox{%
\begin{minipage}{0.94\linewidth}
\setlength{\parskip}{0pt}
\setlength{\parindent}{0pt}
\textbf{Core Challenge.}
Motivated by our theoretical analysis, we ask: How can we improve multi-turn agentic RL performance by stabilizing turn-wise training with rollout--training consistency, turn-level credit assignment, and controlled asynchronous policy drift?
\end{minipage}
}

\end{center}

\vspace{-3pt}
To address this challenge, \textbf{we propose Reverse-Turn Policy Optimization (RTPO), a policy optimization framework for stabilizing multi-turn agentic RL training}, with theoretical details provided in Sec.~\ref{sec:method}. The key idea is to formalize sampled interactions as sparse trees and optimize turn-level policies in temporal reverse order, propagating continuation values from later turns to earlier ones through the reverse optimality guarantee. By constructing sibling continuations for each turn, RTPO estimates turn-level advantages under matched downstream conditions. This removes context inconsistency induced by flattened trajectory optimization, reduces turn-level credit bias through causal action alignment, and controls asynchronous policy drift through on-policy continuation. 

 \textbf{Our contributions are fourfold}:
(i) We identify coupled sources of instability in multi-turn agentic RL: context mismatch, weak turn-level credit, and asynchronous policy drift.
(ii) We provide a theoretical analysis showing that these sources share a common structural origin in flattened trajectory optimization. 
(iii) We propose RTPO, a reverse-turn policy optimization framework with sparse reverse trees, turn-level on-policy updates, and theoretical guarantees on recursive optimality, context consistency, and reduced credit bias.
(iv) We validate RTPO on multi-turn agentic RL benchmarks, where it outperforms strong baselines, including GRPO, TreeGRPO, ARPO, and SeeUPO, while further stabilizing the training pipeline.

\begin{figure}[h]
    \centering
    \includegraphics[width=1.0\linewidth]{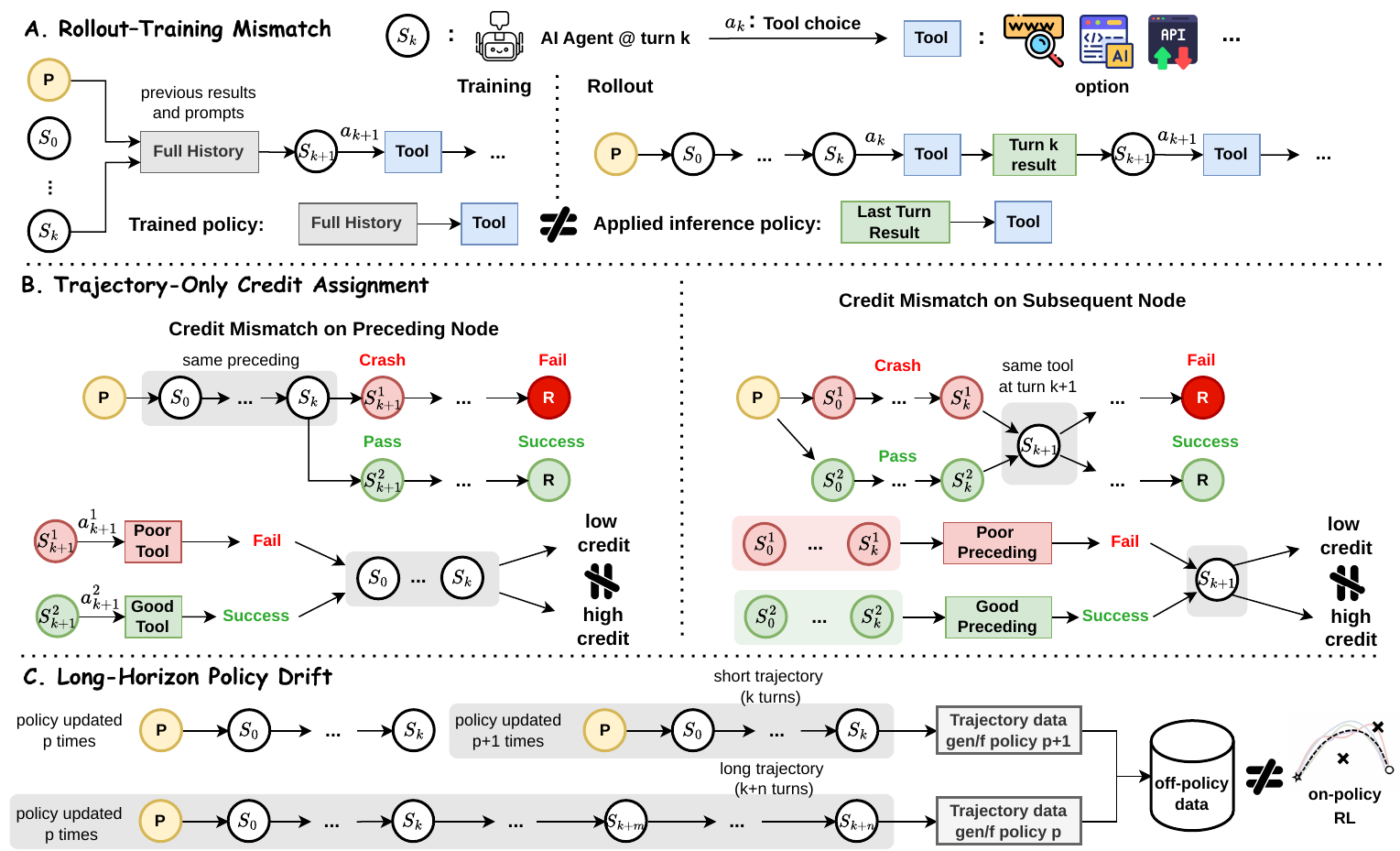}
       \caption{The identified training instability in multi-turn agentic RL arises from: (A) rollout--training context mismatch, (B) trajectory-only credit assignment, and (C) long-horizon policy drift.}
       \vspace{-10pt}
    \label{fig:1}
\end{figure}

\section{Theoretical Analysis: Training Instability}
\label{theory}
We argue that the instability of multi-turn RL training stems from \textit{rollout--training mismatch}: rollouts are typically generated under truncated or summarized contexts, while training recomputes likelihood ratios under concatenated full-history contexts. This discrepancy biases policy optimization and worsens over long horizons. Moreover, the flattened full-history formulation provides only trajectory-level credit, causing terminal rewards to obscure individual turns. When trajectories are generated asynchronously, this mismatch further induces policy drift, as long trajectories generated under an older policy may be optimized after the policy has already been updated by shorter trajectories. 

Building on these observations and insights, we model a multi-turn interaction as a hierarchical Markov decision process (H-MDP) \citep{10.5555/2074094.2074120}. Given an initial prompt $q$, turn $k$ consists of a model response $l_k$ and environment feedback $f_k$, producing the trajectory $(q,l_0,f_0,\ldots,l_{n-1},f_{n-1})$. The turn-level state is $S_k=(q,l_0,f_0,\ldots,l_{k-1},f_{k-1})$, and the turn-level action is the response $l_k$. Each response is generated autoregressively as $l_k=(a_{k,1},\ldots,a_{k,T_k})$, where the token-level state is $s_{k,t}=(S_k,a_{k,<t})$. Existing PPO- and GRPO-style methods \citep{shao2024deepseekmath,zheng2025group,yue2025vapo,yu2026dapo} typically flatten the full interaction into a single token sequence and optimize the resulting trajectory as follows:
\begin{equation}
\label{eq:flat-short}
J^{\mathrm{flat}}(\theta)
=
\mathbb{E}
\left[
\frac{1}{G}\sum_{i=1}^{G}
\frac{1}{\sum_t m_{i,t}}
\sum_t m_{i,t}
\min\!\left(
\rho_{i,t} A_i,
\operatorname{clip}(\rho_{i,t},1-\epsilon,1+\epsilon) A_i
\right)
\right],
\end{equation}
where $\rho_{i,t}=\pi_\theta(a_{i,t}\mid x_{i,<t})/
\pi_{\theta_{\mathrm{old}}}(a_{i,t}\mid x_{i,<t})$ is the token-level importance-sampling (IS) ratio and $A_i$ is a trajectory-level advantage assigned uniformly to all unmasked tokens in trajectory $g_i$ ($g_i$ belongs to a group of $G$ trajectories). Full preliminaries are provided in Appendix~\ref{full-theory}.

\subsection{Rollout--Training Mismatch}
\label{sec:truncation}
As illustrated in \textbf{Figure~\ref{fig:1}-A}, in multi-turn interactions, rollouts are often generated under a truncated or summarized context $\phi(\bar{x}_k)$, while training recomputes token probabilities under the full flattened history $\bar{x}_k$. Thus, the IS ratio used in training differs from the true IS ratio induced by the rollout distribution:
\begin{equation}
\label{eq:is-mismatch-short}
\rho_{k,t}^{\mathrm{flat}}
=
\frac{\pi_\theta(a_{k,t}\mid \bar{x}_k)}
     {\pi_{\theta_{\mathrm{old}}}(a_{k,t}\mid \phi(\bar{x}_k))}
\neq
\frac{\pi_\theta(a_{k,t}\mid \phi(\bar{x}_k))}
     {\pi_{\theta_{\mathrm{old}}}(a_{k,t}\mid \phi(\bar{x}_k))}
=
\rho_{k,t}^{\mathrm{true}} .
\end{equation}
Because the denominator in $\rho_{k,t}^{\mathrm{flat}}$ does not match the distribution that actually sampled the token, the resulting policy-gradient estimate is biased. This mismatch becomes more severe in later turns as the omitted history grows. Moreover, when $\phi$ is non-injective, distinct full-history states can collapse into the same truncated observation, inducing state aliasing and restricting optimization to an observation-induced policy class whose optimum may be strictly below the full-history optimum. 

\noindent\fcolorbox{gray}{gray!10}{
\parbox{0.96\linewidth}{
See \textbf{Appendix~\ref{sec:append_truncation}} for the full theoretical analysis of rollout--training mismatch.
}}

\subsection{Trajectory-Only Credit Assignment}
\label{sec:credit}

Flattened training assigns a single trajectory-level advantage to all turns, even though different turns may contribute unequally to the final outcome, as shown in \textbf{Figure~\ref{fig:1}-B}. For trajectory $g_i$, the population trajectory advantage can be decomposed at turn $k$ as
\begin{equation}
\label{eq:traj-decompose-short}
R_i-\mathbb{E}[R]
=
\underbrace{R_i-Q^\pi(S_{i,k},l_{i,k})}_{\text{downstream stochasticity}}
+
\underbrace{A^\pi(S_{i,k},l_{i,k})}_{\text{true turn advantage}}
+
\underbrace{V^\pi(S_{i,k})-\mu_R}_{\text{upstream state effect}} .
\end{equation}
Here, $R_i$ denotes the final return of trajectory $g_i$, $Q^\pi(S_{i,k}, l_{i,k}) = \mathbb{E}[R \mid S_{i,k}, l_{i,k}]$, $V^\pi(S_{i,k}) = \mathbb{E}[R \mid S_{i,k}]$, and $\mu_R = \mathbb{E}[R]$. This decomposition shows that the true turn-level advantage is only one component of the trajectory-level signal. When upstream state effects or downstream stochasticity dominate, the sign of the trajectory advantage may disagree with the true turn-level advantage, leading to incorrect or even reversed policy updates. Similarly, group-relative baselines provide valid local comparisons only when trajectories share the same turn-level state. Under cross-state grouping, the advantage estimator incurs additional bias $\mathrm{Bias}_{\mathrm{cross}} = \frac{G-1}{G}\left(V^\pi(S_{i,k})-\mu_R\right)$, determined by the upstream trajectory without causal connection to the current action. 

\noindent\fcolorbox{gray}{gray!10}{
\parbox{0.96\linewidth}{
See \textbf{Appendix~\ref{sec:append_credit}} for the full theoretical analysis of trajectory-only credit assignment issues.
}}

\subsection{Long-Horizon Policy Drift}
\label{PPO_clipping}

In asynchronous training with parallel sampled trajectories (see \textbf{Figure~\ref{fig:1}-C}), shorter trajectories may complete and update the policy while longer trajectories are still being generated under an older policy. When these longer trajectories are later optimized, they become off-policy with respect to the current model. The unbiased correction would require the full-trajectory importance weight $\omega_i = \prod_{t=1}^{T_i} \rho_{i,t}$. However, the PPO or GRPO method applies clipping independently at the token level, so the resulting correction differs from the true full-trajectory IS ratio:
\begin{equation}
\label{eq:clip-product-short}
\prod_{t=1}^{T_i}
\operatorname{clip}(\rho_{i,t},1-\epsilon,1+\epsilon)
\neq
\prod_{t=1}^{T_i}\rho_{i,t}
=
\omega_i .
\end{equation}
Therefore, token-level clipping cannot faithfully correct long-horizon policy drift. The discrepancy compounds with trajectory length, while using the exact full-trajectory ratio is impractical because its variance grows rapidly with $T_i$. This explains why standard PPO- or GRPO-style training methods can become unstable or even collapse in long-horizon agentic RL. 

\noindent\fcolorbox{gray}{gray!10}{
\parbox{0.96\linewidth}{
The full theoretical analysis of long-horizon policy drift is provided in \textbf{Appendix~\ref{append_PPO_clipping}}.
}}

\section{Method: Reverse-Turn Policy Optimization (RTPO)}
\label{sec:method}

To address the training instability revealed by our theoretical analysis in Sec.~\ref{theory}, \textit{we propose Reverse-Turn Policy Optimization (RTPO) with theoretical guarantees for stabilizing agentic RL training}, as shown in \textbf{Figure~\ref{fig:2}}. RTPO first models multi-turn interaction as a turn-boundary MDP and defines an independent sub-policy optimization objective for each turn (Sec.~\ref{sec:turn-opt}). This formulation enables reverse-order training to mitigate the mismatch between rollout and training contexts. RTPO then develops sparse tree rollouts based on maximum-value decomposition (Sec.~\ref{sec:tree-credit}) to estimate true turn-level advantages, enabling causally consistent turn-level credit assignment for the case of cross-trajectory comparison without state bias. Finally, RTPO designs an on-policy continuation mechanism (Sec.~\ref{sec:onpolicy}) that eliminates the need for downstream IS-ratio correction under PPO clipping, thereby addressing policy drift induced by asynchronous turns.

\begin{figure}[b!]
    \centering
\includegraphics[width=\linewidth]{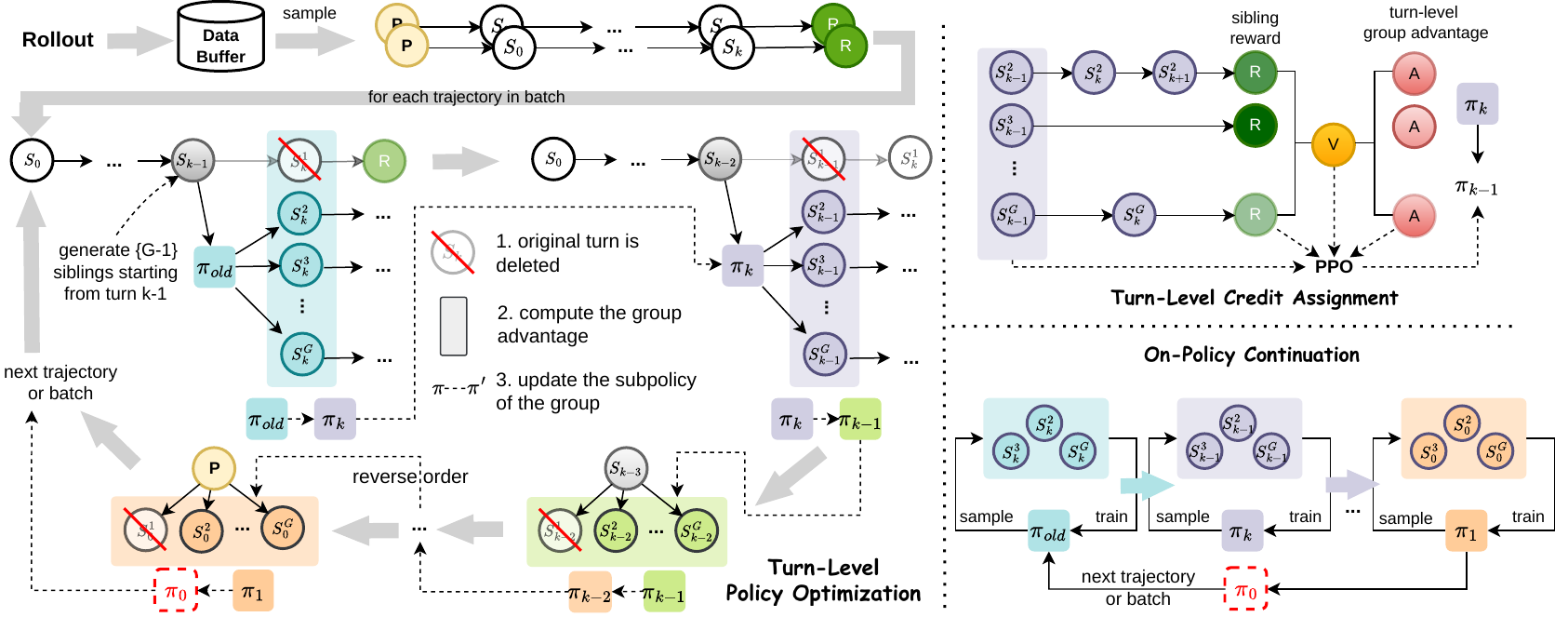}
\caption{Overview of RTPO in agentic RL training. After rollout, RTPO performs reverse-order \textit{turn-level policy optimization} for each trajectory in the batch. Starting from the final turn $k$, the rollout (old) policy $\pi_{\mathrm{old}}$ generates $G{-}1$ sibling rollouts from the same turn boundary to estimate a group advantage and update the training (turn-level) policy to $\pi_k$. The procedure then proceeds backward through turns $k{-}1,k{-}2,\ldots,0$, where each updated turn policy as $\pi_{k-1},\pi_{k-2},\ldots,\pi_0$. To generate state-matched siblings for \textit{turn-level credit assignment}, we design a sparse tree structure that assigns group-relative advantages to individual sibling rollouts at each turn and propagates optimization backward across the trajectory. Finally, RTPO applies \textit{on-policy continuation} to coordinate asynchronous short- and long-trajectory updates from $\pi_k$ to $\pi_0$, reducing policy drift.}
    \label{fig:2}
\end{figure}
 
\subsection{Turn-Level Policy Optimization}
\label{sec:turn-opt}

\paragraph{Turn-boundary MDP for rollout-training match.}
We model a $K$-turn agent episode as a turn-boundary MDP $\mathcal{M}=\langle \bar{\mathcal{X}},\mathcal{A}_H,P_H,R_H,\gamma_H\rangle$, where the augmented state $\bar{x}_k=(S_k,k)$ encodes the interaction history $S_k$ and the turn index $k$. At the turn-$k$ boundary, the agent selects a macro-action $u_k\equiv l_k\in\mathcal{A}_{H,k}(S_k)$, corresponding to the complete turn-$k$ response $l_k=(a_{k,1},\ldots,a_{k,T_k})$. Executing $u_k$ consumes $\tau_k=T_k$ token steps, after which the environment returns to the external tool feedback $f_k$, and the interaction history is updated as $S_{k+1}=S_k\circ(l_k,f_k)$. To define the conditioning context received by the model at turn $k$, we let $c_k=\psi(S_k)$, where $\psi$ can be the identity map, a truncation operator, or a summarization operator. The turn-level policy is factorized as $\pi_\theta=(\pi_{\theta,0},\ldots,\pi_{\theta,K-1})$, where each sub-policy is autoregressive at the token level: $\pi_{\theta,k}(u_k | c_k) = \prod_{t=1}^{T_k} \pi_\theta(a_{k,t} | c_k, a_{k,<t})$.

This turn-boundary MDP decomposition factorizes the episode policy into $K$ turn-level sub-policies, with each sub-policy mapping the conditioning context $c_k=\psi(S_k)$ to a complete response $l_k$. Here, $c_k$ is kept identical between rollout and training. During rollout, RTPO records the exact context $c_k$ received by the model, including any truncation or summarization, together with the corresponding old-policy log-probabilities. During training, the same $c_k$ is used as input, and the loss is computed only over the output tokens in $l_k$. Hence, the denominator of the IS ratio is evaluated under the same conditioning context as in rollout:

\begin{equation}
\label{eq:is-match}
\rho_{k,t}^{\mathrm{RTPO}} = \frac{\pi_\theta(a_{k,t} | c_k, a_{k,<t})}{\pi_{\theta_{\mathrm{old}}}(a_{k,t} | c_k, a_{k,<t})}
\end{equation}
 
\paragraph{Per-turn policy optimization under reverse-order training.}
Training proceeds in reverse order through the turns $k = K{-}1, K{-}2, \ldots, 0$. After turn $k$ is completed, $\pi_{\theta,k}$ is frozen (e.g., subsequent turns produce no gradients for turn-$k$ tokens). This reverse ordering ensures that when turn $k$ is trained, the downstream policies in each turn $\pi_{\theta,k+1:K-1}$ have been optimized and fixed.

During reverse-order training for turn-$k$, the environment is forked from the trunk trajectory's boundary state $S_k$ to generate $G{-}1$ sibling rollouts in the sparse tree. Each sibling rollout $j$ receives a turn-level advantage $A_{j,k}^{H}$ (the details of sibling rollout generation are provided in Sec.~\ref{sec:tree-credit}). The policy optimization objective for turn $k$ is then defined as:
\begin{equation}
\label{eq:per-turn-obj}
J_k(\theta)=\frac{1}{G{-}1}\sum_{j=1}^{G-1}\frac{1}{T_{j,k}}\sum_{t=1}^{T_{j,k}}\min\!\Big(\rho_{j,k,t}\, A_{j,k}^{H},\;\operatorname{clip}\big(\rho_{j,k,t},\, 1{-}\epsilon,\, 1{+}\epsilon\big)\, A_{j,k}^{H}\Big)
\end{equation}
where $\rho_{j,k,t} = \pi_\theta(a_{j,k,t} | c_k, a_{j,k,<t}) / \pi_{\theta_{\mathrm{old}}}(a_{j,k,t} | c_k, a_{j,k,<t})$ is the token IS ratio and $T_{j,k}$ is the number of tokens generated by sibling rollout $j$ at turn $k$. Only these $G{-}1$ sibling rollouts are used for gradient updates; the trunk trajectory is excluded. 

In \textbf{Theorem~1}, we show that RTPO has local and global convergence guarantees through reverse-order turn-level policy optimization.

\begin{tcolorbox}[
    colback=gray!5,
    colframe=black!60,
    title=Theorem 1: Convergence to Recursive Optimality,
    fonttitle=\bfseries,
    breakable
]

Under standard assumptions, finite state and macro-action spaces, Robbins–Monro step sizes, sufficient per-turn exploration, frozen downstream policies, and bounded rewards, reverse-order turn-level optimization satisfies:

\begin{enumerate}[label=(\alph*),nosep]
 \item \textbf{Per-turn convergence.} For each turn $k$, fixing downstream policies reduces optimization to single-step Q-learning with a stationary continuation target; hence, per-turn policy $\pi_k$ converges under the stated assumptions.
 \item \textbf{Recursive optimality.} Under reverse-order training, applying the above argument recursively from turn $K{-}1$ to turn $0$ yields a sequence of per-turn policies that is recursively optimal.
 \item \textbf{Global optimality.} If the turn-level macro-action spaces are complete, i.e., the per-turn policy class can represent any globally feasible trajectory-level policy, then recursive optimality is equivalent to global optimality over the full trajectory.
\end{enumerate}

\end{tcolorbox}

The proof of Theorem~1 is provided in \textbf{Appendix~\ref{app:convergence-proof}}.

\subsection{Turn-Level Credit Assignment}
\label{sec:tree-credit}
\paragraph{Turn-level advantage function.}
We construct sparse tree rollouts at turn-level boundaries, elevating the advantage granularity from trajectory-only reward to turn-level credit while ensuring that all compared rollouts share the same state and are free from state bias. In turn $k$, $G{-}1$ siblings independently generate turn-$k$ responses from the shared boundary state $S_k$ and continue to the terminal. The turn-level advantage function for sibling $j$ is: $A^H_{j,k} = \hat{Q}_{j,k} - \hat{V}_k$. 

Here, all siblings share the same state $S_k$, so advantage differences can only arise from two sources: different action choices at turn $k$ and independent downstream sampling noise. The downstream noise has zero mean ($\mathbb{E}[\xi_{\mathrm{down}} | S_k, l_{j,k}] = 0$) and introduces no systematic bias; the upstream state term $\xi_{\mathrm{up}} = V^\pi(S_k) - \bar{V}$ from Eq.~(\ref{eq:traj-decompose-short}) is exactly zero, because all siblings share $S_k$. Furthermore, $A^H_{j,k}$ is assigned only to the output tokens of turn $k$; the shared prefix $S_k$ serves as the prompt input but does not enter the loss, so each token appears exactly once in the training batch.
 
\paragraph{Turn-level value estimation.}
To identify turn-level values, we separate the local reward at each turn from the downstream completion value over the full trajectory. Based on the MAXQ principle \citep{dietterich2000hierarchical}, we define the following action-value decomposition for turn $k$: $\tilde{Q}_k^\pi(S_k, u_k) = r_k + \gamma^{\tau_k} F_k^\pi(S_{k+1})$. $r_k$ is the immediate reward at turn $k$ ($r_k = 0$ for $k < K{-}1$ under sparse rewards), and $F_k^\pi(S_{k+1})$ is the downstream continuation value representing the expected cumulative return from turn $k{+}1$ onward under policy $\pi$, with base case $F_{K-1}^\pi \equiv 0$. Then, the terminal reward $R_j$ obtained by sibling $j$ rolling out the policy from $S_k$ to the terminal is a single Monte Carlo sample of the above:
\begin{equation}
\label{eq:q-mc}
\hat{Q}_{j,k} = R_j = r_{j,k} + \gamma^{\tau_k} \hat{F}_{j,k}
\end{equation}
where $\hat{F}_{j,k}$ is a single sample of $F_k^\pi(S_{k+1})$. $R_j$ is an unbiased estimator of $\tilde{Q}_k^\pi$. Note that the value estimation quality of $\hat{Q}_{j,k}$ depends on $\hat{F}_{j,k}$, i.e., the quality of the sampled downstream continuation value. If the downstream policy is not yet optimized, even a strong turn-$k$ action may still produce $R_j=0$ due to downstream errors, thereby contaminating the turn-level advantage with downstream noise. The case where $\hat{F}_{j,k}$ is generated by an already optimized downstream policy is addressed in Sec.~\ref{sec:onpolicy}. After estimating the action values at each turn, we compute a state-specific value baseline from the sibling rollouts. Specifically, we define $\hat{V}_k$ as the mean of the estimated $Q$-values across the $G{-}1$ siblings: $\hat{V}_k = \frac{1}{G{-}1}\sum_{j=1}^{G-1} R_j$.

In \textbf{Theorem~2}, we show that RTPO obtains accurate turn-level credit without bias from upstream and downstream effects.

\begin{tcolorbox}[
    colback=gray!5,
    colframe=black!60,
    title=Theorem~2: Causally Consistent Turn-Level Advantage Estimation,
    fonttitle=\bfseries,
    breakable
]

Consider turn $k$ and suppose that the $G{-}1$ sibling rollouts are forked from the same boundary state $S_k$, with each sibling rollout $j\in\{1,\ldots,G{-}1\}$ independently sampling a turn-$k$ macro-action $u_{j,k}\equiv l_{j,k}$. Let $A^H_{j,k}$ denote the turn-level advantage assigned to sibling rollout $j$. Then, under bounded rewards and independent sibling sampling, the turn-level advantage estimator satisfies the following properties:
\begin{enumerate}[label=(\alph*),nosep]
    \item \textbf{Local unbiasedness up to finite-group bias.}
    Conditional on $S_k$, the expectation of $A^H_{j,k}$ is proportional to the true turn-level advantage, up to a finite-group bias of order $O(1/G)$. Because all comparisons are made from the same boundary state, the estimator removes the upstream state-contamination term that appears in cross-trajectory comparisons.

    \item \textbf{Reduced value estimation error.}
    The mean squared error of $A^H_{j,k}$ is lower than that of trajectory advantage estimation whenever cross-state value variance is non-zero. The improvement gap is governed by the variance of values across different boundary states, which can be large in long-horizon, multi-turn tasks.

    \item \textbf{State-matched causal actions.}
    Since all siblings share the same boundary state $S_k$, differences in their returns are causally attributable to the sampled turn-$k$ macro-actions $u_{j,k}$, rather than to upstream trajectory differences. The resulting advantage is assigned only to turn-$k$ output tokens, excluding prefix tokens from the gradient.
\end{enumerate}
\end{tcolorbox}

The proof of Theorem~2 is provided in \textbf{Appendix~\ref{app:advantage-proof}}.

\subsection{On-Policy Continuation}
\label{sec:onpolicy}

\paragraph{On-policy evolution.}
Asynchronous turn updates can induce policy drift that is not fully corrected by per-token IS clipping. In principle, one could correct this drift using the full trajectory-level IS product, but its variance grows exponentially with the horizon length, making it unstable for long-horizon multi-turn training. This drift directly affects the downstream continuation value $F_k$ introduced in Sec.~\ref{sec:tree-credit}: under a stale or mismatched downstream policy, $\hat{F}_{j,k}$ can systematically deviate from the true continuation value $F_k^{\pi_{\theta_{>k}}}$, while existing IS-based corrections are insufficient to remove this deviation. To avoid this issue, RTPO re-generates sibling continuations on-policy at each turn. Let $\theta_0$ denote the parameters at the start of the policy evolution. During reverse-order training, the parameters are updated sequentially across turns. By the time optimization reaches turn $k$, the policies for downstream turns $K{-}1,\ldots,k{+}1$ have already been updated; we denote the resulting current parameters by $\theta_{>k}$. Thus, under $\pi_{\theta_{>k}}$, the downstream turns $k{+}1,\ldots,K{-}1$ use the optimized continuation policy, whereas turn $k$ and all upstream turns remain to be optimized. Note that the trunk is a complete trajectory generated at the start of the policy evolution using $\pi_{\theta_0}$. It does not participate in gradient updates, nor does it enter the computation of $\hat{V}_k$. Its sole role is to provide boundary states $S_k$ and environment snapshots $\mathrm{snap}_k$ as anchoring points for sibling forking. The trunk's policy nature affects which states $S_k$ are visited during training (state coverage), but does not affect the correctness of advantage estimation, since all siblings contributing to the estimate are generated on-policy.

\paragraph{On-policy sibling generation.}
At the start of turn $k$, RTPO synchronizes the latest parameters $\theta_{>k}$ to the inference engine, forks the environment from $\mathrm{snap}_k$, and generates $G{-}1$ siblings using $\pi_{\theta_{>k}}$. Each sibling rollout $j$ generates a turn-$k$ response, then continues with $\pi_{\theta_{>k}}$ through turns $k{+}1$ to $K{-}1$, obtaining terminal reward $R_j$. Since the sampling policy equals the current policy, the trajectory-level IS weight is identically one:
$\omega_j=\prod_{h=k+1}^{K-1}\prod_{t=1}^{T_{j,h}}\frac{\pi_{\theta_{>k}}(a_{j,h,t} \mid s_{j,h,t})}{\pi_{\theta_{>k}}(a_{j,h,t} \mid s_{j,h,t})}\equiv 1 .$

Since turn-level value estimation is $\hat{Q}_{j,k} = r_{j,k} + \gamma^{\tau_k} \hat{F}_{j,k}$ from Sec.~\ref{sec:tree-credit}, the sibling's downstream continuation is performed under the current policy $\pi_{\theta_{>k}}$, the sample $\hat{F}_{j,k}$ is an unbiased draw from $F_k^{\pi_{\theta_{>k}}}(S_{j,k+1})$. Therefore:
\begin{equation}
\label{eq:q-onpolicy}
\hat{Q}_{j,k}=r_{j,k}+\gamma^{\tau_k}\hat{F}_{j,k}^{\pi_{\theta_{>k}}}
\end{equation}
is an unbiased Monte Carlo estimate of $\tilde{Q}_k^{\pi_{\theta_{>k}}}(S_k, u_{j,k})$.
The entire value estimate requires no IS correction, is unaffected by clip truncation, and is free of multiplicative variance explosion.
 
In \textbf{Theorem~3}, we show that RTPO conducts on-policy continuation to avoid policy drift and further reduces advantage-estimation errors.

\begin{tcolorbox}[
    colback=gray!5,
    colframe=black!60,
    title=Theorem 3: On-Policy Continuation under Asynchronous Turns,
    fonttitle=\bfseries,
    breakable
]

In the reverse-order training procedure of RTPO, at the start of each asynchronous turn $k$, sibling rollouts $j\in\{1,\ldots,G{-}1\}$ are generated from the shared boundary state $S_k$ using the current downstream policy $\pi_{\theta_{>k}}$ and are then continued to termination under the same policy. Then the following two properties hold:
\begin{enumerate}[label=(\alph*),nosep]
    \item \textbf{Drift-free on-policy continuation.}
    Each sibling's terminal return $R_j$ provides an unbiased Monte Carlo estimate of the turn-level Q-value under the current downstream policy $\pi_{\theta_{>k}}$. Because the sampling policy and the evaluation policy coincide throughout the sibling continuation, the trajectory-level importance-sampling weight $\omega_j$ is identically one, and no trajectory-level IS correction is required.

    \item \textbf{Dynamic error reduction in advantage estimation.}
    As reverse-order training improves the downstream policies, the continuation value estimates $\hat{F}_{j,k}^{\pi_{\theta_{>k}}}$ become aligned with the current optimized downstream policy rather than a stale policy. Under bounded binary or normalized rewards, when the induced success probabilities move away from the high-uncertainty region around one-half, the variance of the Monte Carlo Q-value estimator decreases, thereby improving the signal-to-noise ratio of the resulting turn-level advantage estimates $A^H_{j,k}$.
\end{enumerate}
\end{tcolorbox}

The proof of Theorem~3 is provided in \textbf{Appendix~\ref{app:onpolicy-proof}}.

\section{Experimental Results and Analysis}\label{experiment}

\textbf{Experimental Setting.} 
We use Qwen3-8B~\citep{yang2025qwen3} as the backbone for the main experiments on RTPO and all baselines, enabling Qwen3's thinking mode for all multi-turn agentic RL rollouts. We compare RTPO with trajectory-level methods, including GRPO~\citep{shao2024deepseekmath}, and turn/tree-level credit methods, including ARPO~\citep{dong2026agentic}, TreeGRPO~\citep{ji2026tree}, and SeeUPO~\citep{hu2026seeupo}. We consider 12 experiments covering comprehensive mathematical and knowledge reasoning benchmarks, where agents perform multi-turn RL training with external calculation and web-search tools. All methods are implemented in VeRL \citep{sheng2024hybridflow}, with vLLM \citep{kwon2023efficient} for rollout generation and FSDP \citep{zhao2023pytorch} for distributed training. We evaluate training performance and stability using task accuracy, tool-call statistics, log-probability comparisons, turn-level credit effects, and the hit rates of on-policy versus off-policy outputs. Detailed experimental setup, including base models, baselines, datasets, configurations, and evaluation metrics, is provided in \textbf{Appendix~\ref{Experimental}}. Implementation details, including pseudocode, training details, and the Qwen3 chat template, are provided in \textbf{Appendix \ref{imple datials}}.


\vspace{-2pt}
\subsection{Main Results}
\label{results:main}

We compare RTPO against the trajectory-level method GRPO and the turn-level method SeeUPO, and additionally include the untuned vanilla model as a non-RL reference. \textbf{Table~\ref{tab:math_knowledge_toolcall_results}} reports both accuracy and the corresponding number of tool calls on each benchmark. We evaluate overall performance on mathematical reasoning and knowledge-intensive question answering across three difficulty tiers: easy (GSM8K), medium (AMC23, MATH500), and hard (AIME24, AIME25, OE-Math, HotpotQA, and 2Wiki). Our RTPO achieves the best performance across all eight benchmarks, improving overall accuracy over vanilla by 66.78\%, outperforming GRPO (+21.50\%) and SeeUPO (+10.76\%). Further discussion and insights on tool calls are provided in \textbf{Appendix~\ref{dis for main}}.

\begin{table*}[!ht]
\centering
\scriptsize
\setlength{\tabcolsep}{2.5pt}
\renewcommand{\arraystretch}{0.95}
\resizebox{\textwidth}{!}{
\begin{tabular}{l|cc|cc|cc|cc|cc|cc|cc|cc|cc}
\toprule
\multirow{3}{*}{\textbf{Method}}
& \multicolumn{12}{c|}{\textbf{M: Mathematical Reasoning}} 
& \multicolumn{4}{c|}{\textbf{K: Knowledge Reasoning}}
& \multicolumn{2}{c}{\textbf{Overall}} \\
\cmidrule(lr){2-13}
\cmidrule(lr){14-17}
\cmidrule(lr){18-19}
& \multicolumn{2}{c|}{\textbf{AIME24}}
& \multicolumn{2}{c|}{\textbf{AIME25}}
& \multicolumn{2}{c|}{\textbf{AMC23}}
& \multicolumn{2}{c|}{\textbf{MATH500}}
& \multicolumn{2}{c|}{\textbf{GSM8K}}
& \multicolumn{2}{c|}{\textbf{OE-Math}}
& \multicolumn{2}{c|}{\textbf{HotpotQA}}
& \multicolumn{2}{c|}{\textbf{2Wiki}}
& \multicolumn{2}{c}{\textbf{}} \\
& \textbf{Acc} & \textbf{Calls}
& \textbf{Acc} & \textbf{Calls}
& \textbf{Acc} & \textbf{Calls}
& \textbf{Acc} & \textbf{Calls}
& \textbf{Acc} & \textbf{Calls}
& \textbf{Acc} & \textbf{Calls}
& \textbf{Acc} & \textbf{Calls}
& \textbf{Acc} & \textbf{Calls}
& \textbf{Acc} & \textbf{Calls} \\
\midrule
Vanilla
& 3.33 & 17
& 13.33 & 23
& 12.50 & 66
& 51.00 & 341
& 76.95 & 731
& 24.33 & 453
& \underline{54.52} & 204
& 59.83 & 224
& 36.97$_{\textcolor{red}{+00.00\%}}$
& 2059$_{(1631,428)}$ \\

GRPO
& 10.00 & 12
& 20.00 & 8
& 57.50 & 18
& 82.33 & 94
& 93.86 & 59
& 50.74 & 29
& 53.39 & 214
& 61.85 & 226
& 53.71$_{\textcolor{red}{+45.28\%}}$
& 660$_{(220,440)}$ \\

SeeUPO
& \underline{20.00} & 8
& \underline{23.33} & 6
& \underline{67.50} & 14
& \underline{84.20} & 87
& \underline{94.47} & 87
& \underline{53.86} & 147
& 54.27 & 254
& \underline{63.78} & 253
& \underline{57.68}$_{\textcolor{red}{+56.02\%}}$
& 856$_{(349,507)}$ \\

RTPO
& \textbf{33.33} & 4
& \textbf{26.67} & 2
& \textbf{67.50} & 9
& \textbf{86.20} & 168
& \textbf{94.84} & 482
& \textbf{55.49} & 106
& \textbf{64.89} & 613
& \textbf{64.35} & 867
& \textbf{61.66}$_{\textcolor{red}{+66.78\%}}$
& 2251$_{(771,1480)}$ \\
\bottomrule
\end{tabular}
}
\caption{Performance comparison across eight tool-use agentic RL benchmarks. \textit{Acc} denotes Pass@1 for mathematical tasks and best-span F1 for knowledge tasks, while \textit{Calls} denotes the number of tool calls rounded to the nearest integer, such as Python and web-search calls. \textit{Overall} subscripts indicate relative changes in accuracy compared with the vanilla model, along with total tool calls (M, K). }
\label{tab:math_knowledge_toolcall_results}
\vspace{-10pt}
\end{table*}

\subsection{Rollout--Training Consistency Analysis}\label{result:rtc}

\label{sec:stability}

To examine RTPO’s training-time advantage, we measure at each step the geometric-mean ratio between training-stage and rollout-stage token \textit{logprobs}. This ratio captures rollout--training consistency: values near $1$ indicate matched conditioning distributions, while deviations suggest that the full-history training policy favors outputs different from those sampled during rollout. The Kullback–Leibler (KL) divergence further quantifies the distributional gap, with smaller values indicating stronger rollout--training consistency. \textbf{Figure~\ref{fig:two_results}} shows the log-probability ratio and KL divergence during training. On mathematical tasks, RTPO stays stable at $1.0$, while SeeUPO fluctuates around $1.0$ and GRPO recovers from $0.89$ to $0.97$ after $30$ steps. On knowledge tasks, RTPO again remains at $1.0$, whereas SeeUPO and GRPO recover only from around $0.80$ to $0.89$ and $0.83$ after $14$ steps. RTPO also achieves the lowest average KL divergence, indicating more consistent rollout-training contexts than the baselines.

\begin{figure}[!ht]
    \centering
    \begin{subfigure}{0.24\linewidth}
        \centering
        \includegraphics[width=\linewidth]{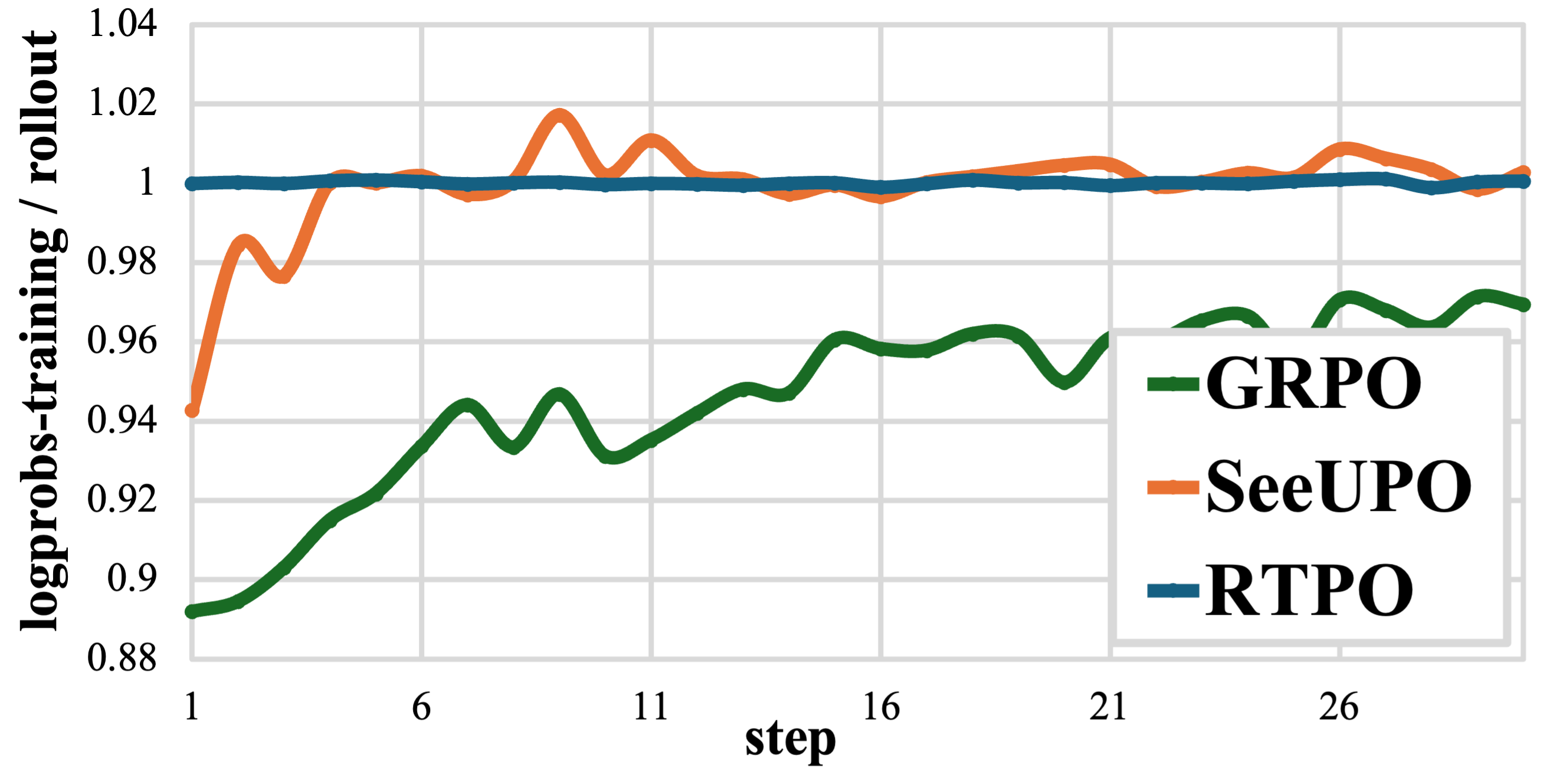}
        \caption{Rollout-train ratio [M]}
        \label{fig:math_ratio}
    \end{subfigure}
    \hfill
    \begin{subfigure}{0.24\linewidth}
        \centering
        \includegraphics[width=\linewidth]{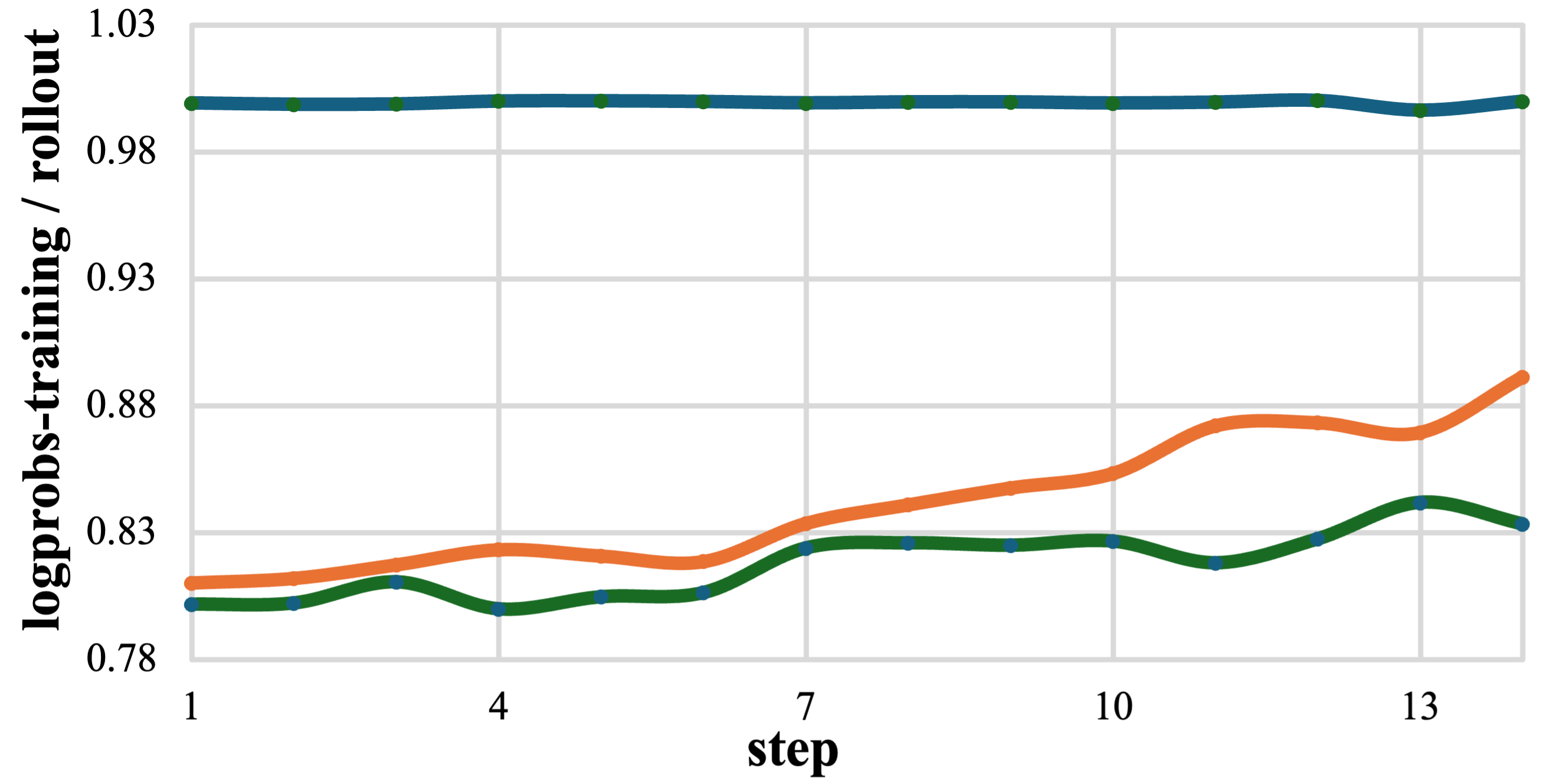}
        \caption{Rollout-train ratio [K]}
        \label{fig:search_ratio}
    \end{subfigure}
    \hfill
    \begin{subfigure}{0.24\linewidth}
        \centering
        \includegraphics[width=\linewidth]{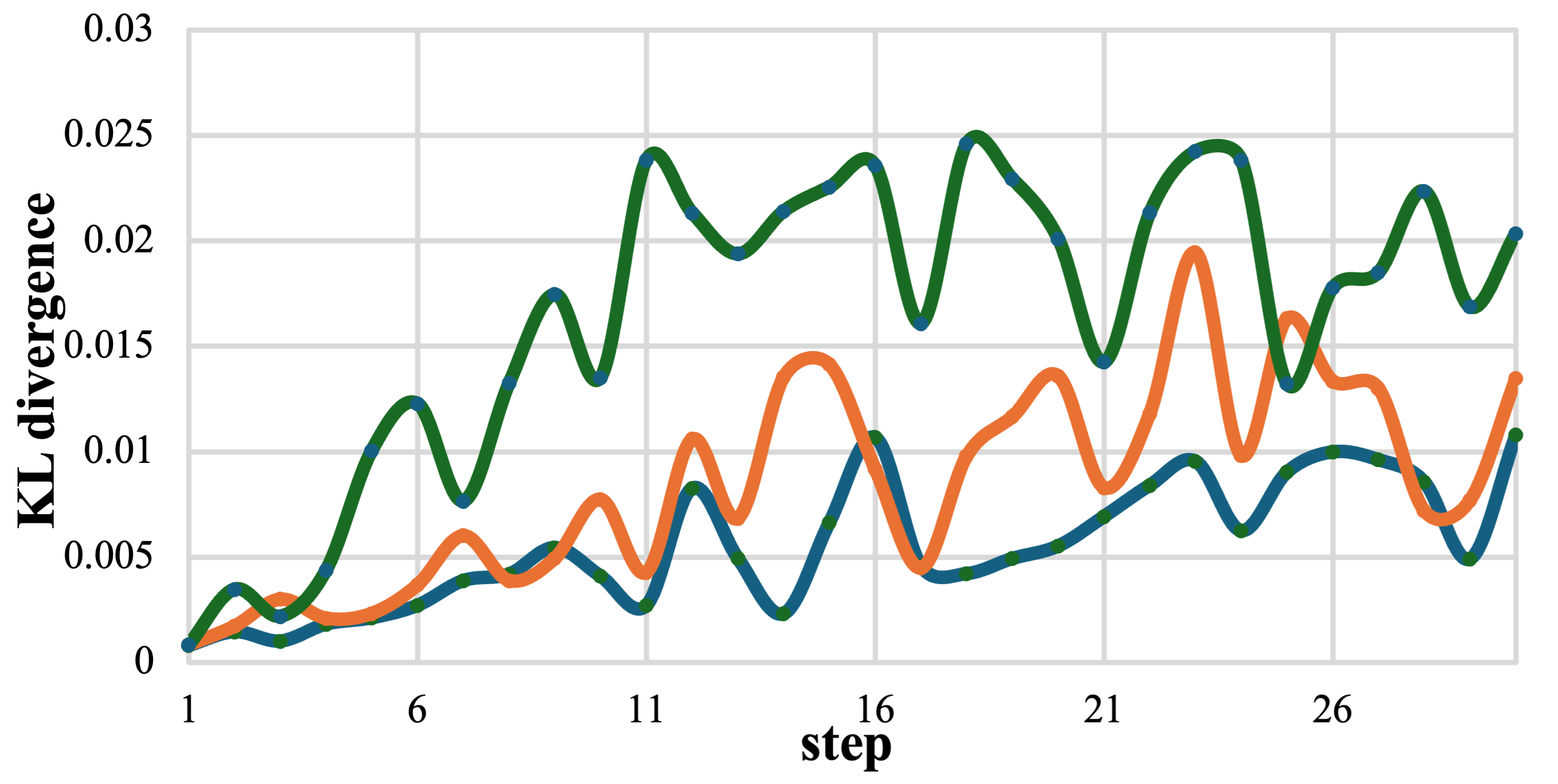}
        \caption{KL divergence (M)}
        \label{fig:math_kl}
    \end{subfigure}
    \hfill
    \begin{subfigure}{0.24\linewidth}
        \centering
        \includegraphics[width=\linewidth]{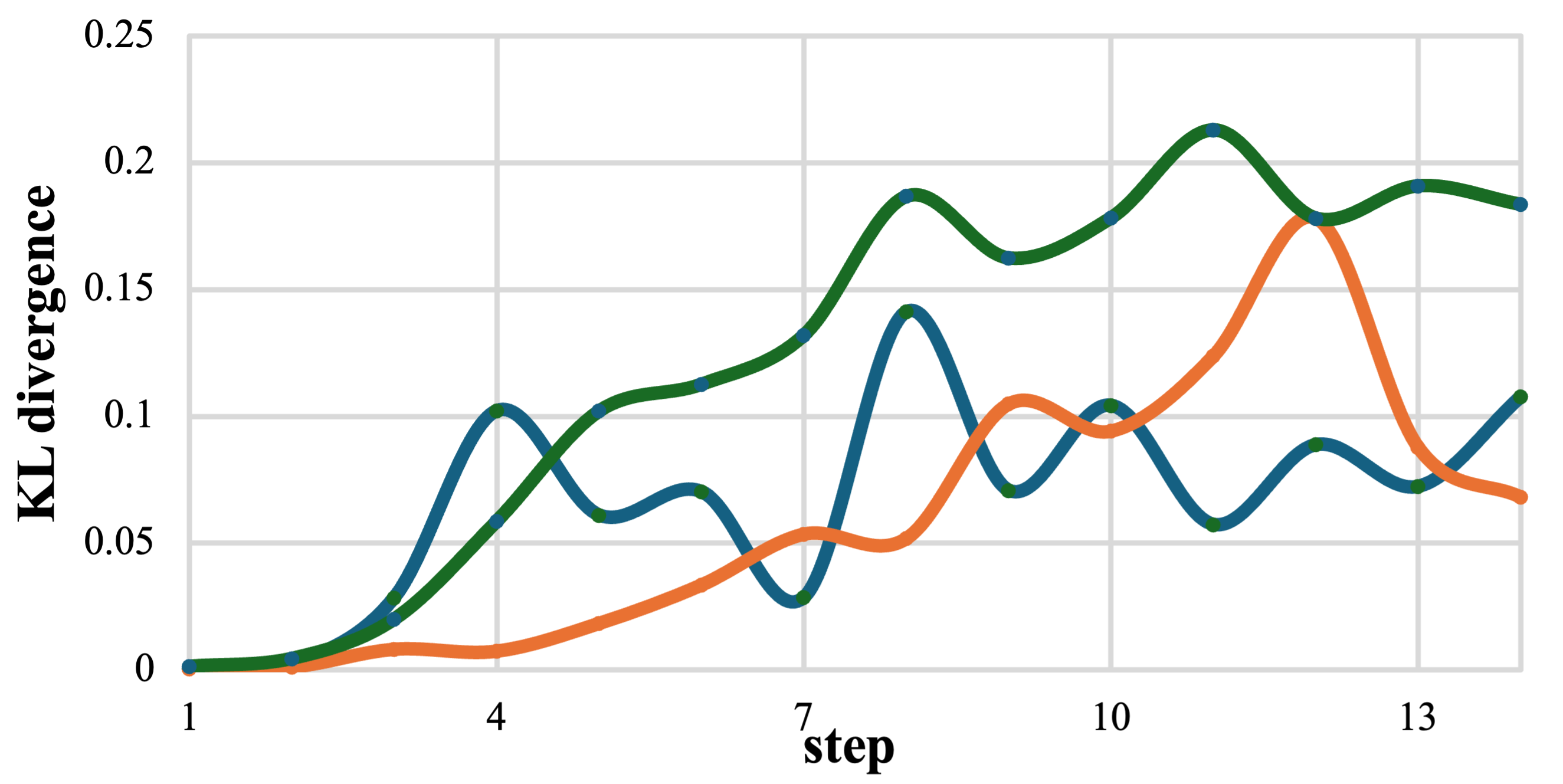}
        \caption{KL divergence (K)}
        \label{fig:search_kl}
    \end{subfigure}
    \caption{Rollout--training consistency comparison using log-probability ratios and KL divergence on mathematical (M) and knowledge (K) reasoning tasks. Zoom in for better visualization.}
    \label{fig:two_results}
    \vspace{-5pt}
\end{figure}

A noteworthy observation is that, despite using different rollout and training contexts, baseline ratios still drift slowly toward $1$. This resembles the bootstrapped alignment mechanisms in DAgger~\citep{ross2011reduction} and SCoRe~\citep{kumartraining}. However, unlike RTPO’s structural consistency, this empirical alignment is incomplete, task-dependent, noisy, and consumes additional optimization budget. We provide further discussion and insights in \textbf{Appendix \ref{dis for rtc}}.

\subsection{Effect of Turn-Level Credit Assignment}\label{result:ca}

\label{sec:credit-diagnostic}

To isolate the effect of RTPO’s turn-level credit signal, we feed the full interaction history as the rollout input for all methods, controlling for the rollout–training context mismatch presented in Sec~\ref{sec:stability}. Under this control setting, performance differences mainly reflect the effect of credit assignment (CA) on advantage estimation. Because full interaction histories are used, the accuracy scores in Table~\ref{tab:math_results} are substantially higher than those in Table~\ref{tab:math_knowledge_toolcall_results}. RTPO-CA achieves the highest overall average across the four mathematical reasoning benchmarks, as shown in \textbf{Table~\ref{tab:math_results}}. It achieves the best results on AMC23 (Pass@1 $93.33$, Pass@4 $100.0$), AIME25 (Pass@1 $70.00$, Pass@4 $76.67$), and MATH500 (Pass@1 $86.60$, Pass@4 $89.60$), while matching the best Pass@4 on AIME24 ($80.00$). These results show that RTPO improves both solution coverage and the preference for correct answers, with the largest gain on the harder AIME25 benchmark. This supports the value of turn-level credit assignment: by forking sibling rollouts from the same boundary state $S_k$, RTPO forms a local baseline and attributes advantage directly to the current turn decision.

\vspace{-5pt}
\begin{table*}[!ht]
\centering
\scriptsize
\setlength{\tabcolsep}{9pt}
\begin{tabular}{lcccccccccc}
\toprule
\multirow{2}{*}{\textbf{Method}}
& \multicolumn{2}{c}{\textit{AMC23}}
& \multicolumn{2}{c}{\textit{AIME24}}
& \multicolumn{2}{c}{\textit{AIME25}}
& \multicolumn{2}{c}{\textit{MATH500}}
& \multicolumn{2}{c}{\textit{Overall}} \\
\cmidrule(lr){2-3}
\cmidrule(lr){4-5}
\cmidrule(lr){6-7}
\cmidrule(lr){8-9}
\cmidrule(lr){10-11}
& P@1 & P@4
& P@1 & P@4
& P@1 & P@4
& P@1 & P@4
& P@1 & P@4 \\
\midrule
ARPO
& 90.00 & 97.50
& \textbf{73.33} & 73.33
& 63.33 & 73.33
& 86.00 & 88.80
& 78.15 & 83.23 \\

TreeGRPO
& 92.50 & 96.67
& 63.33 & \textbf{80.00}
& 56.67 & 70.00
& 86.20 & 89.40
& 74.67 & 84.00 \\

SeeUPO
& 92.50 & 97.50
& \textbf{73.33} & \textbf{80.00}
& 53.33 & 70.00
& 86.20 & 89.20
& 76.33 & 84.18 \\

RTPO-CA
& \textbf{93.33} & \textbf{100.0}
& 66.67 & \textbf{80.00}
& \textbf{70.00} & \textbf{76.67}
& \textbf{86.60} & \textbf{89.60}
& \textbf{79.13} & \textbf{86.55} \\
\bottomrule
\end{tabular}
\caption{Mathematical reasoning performance under controlled turn-level credit assignment (CA). Overall reports the average across benchmarks. Bold indicates the best result.}
\label{tab:math_results}
\vspace{-15pt}
\end{table*}

\subsection{Policy Drift Correction}

\label{sec:ablation-onpolicy}

To isolate the effect of on-policy continuation (Sec~\ref{sec:onpolicy}), we compare the default RTPO with an off-policy variant that reuses downstream continuations generated by $\pi_{\theta_0}$ during the initial rollout and corrects staleness using a clamped trajectory-level IS weight; see Appendix~\ref{appd2} for details. We evaluate both variants on four knowledge-reasoning deep-search benchmarks (GAIA, WebWalkerQA, HLE, and XBench) using output-hit accuracy. As shown in \textbf{Table~\ref{tab:on_off_policy_hit}}, default RTPO outperforms the off-policy variant on GAIA ($+5.83\%$, WebWalkerQA ($+3.50\%$), and XBench ($+7.00\%$). On HLE, the difference is negligible ($-0.33\%$), indicating a near tie. This pattern supports \textit{Theorem~3(b)}: on-policy continuation is most beneficial when downstream policies change substantially during reverse-order training. GAIA, WebWalkerQA, and XBench involve longer retrieval and interaction horizons, where stale-rollout IS correction can introduce clamp-truncation bias that on-policy re-sampling avoids. In contrast, HLE is more closed-ended and often requires shorter search horizons, leading to smaller gains; additional results in \textbf{Appendix \ref{add subsets}} further support this interpretation. In addition, we discuss the limitations, future work, and broader impacts of RTPO in \textbf{Appendix~\ref{limits}, \ref{impact}}.

\vspace{-8pt}

\begin{table}[!ht]
\centering
\small
\begin{tabular}{lrrrrrr}
\toprule
\textbf{Benchmark} & QS & Off-Policy Hit & Off-Policy Rate & On-Policy Hit & On-Policy Rate & \textbf{$\Delta$} \\
\midrule
\textit{GAIA}      & 103  & 24  & 23.30\% & 30  & 29.13\% & \textbf{+5.83\%} \\
\textit{XBench}    & 100  & 8   & 8.00\%  & 15  & 15.00\% & \textbf{+7.00\%} \\
\textit{WebWalker} & 200  & 12  & 6.00\%  & 19  & 9.50\%  & \textbf{+3.50\%} \\
\textit{HLE}       & 2096 & 261 & 12.45\% & 254 & 12.12\% & -0.33\% \\
\bottomrule
\end{tabular}
\vspace{2pt}
\caption{Output-hit comparison between on- and off-policy RTPO variants on knowledge tasks. Values are rounded to the nearest integer. $\Delta$ denotes the change from off-policy to on-policy hit rates.}
\label{tab:on_off_policy_hit}
\vspace{-15pt}
\end{table}

\section{Concluding Remarks}

This work identifies rollout--training mismatch as a fundamental source of instability in multi-turn agentic RL, particularly in tool-augmented mathematical reasoning and deep-search tasks. We provide a theoretical analysis showing how existing training pipelines produce unstable optimization signals and propose RTPO as a principled framework to address this issue. RTPO integrates rollout--training consistency, turn-level credit assignment, and on-policy continuation within a unified training pipeline. Supported by theoretical guarantees and empirical results, RTPO improves multi-turn optimization stability and provides a promising direction for training long-horizon tool-using agents.

\section*{Acknowledgements}

This study is partially supported by Australian Research Council (ARC)
Projects DE220100265, DP250103612, and LP240200636.

\newpage
\bibliographystyle{unsrtnat}
\bibliography{main}

@inproceedings{xuesimpletir,
  title={SimpleTIR: End-to-End Reinforcement Learning for Multi-Turn Tool-Integrated Reasoning},
  author={Xue, Zhenghai and Zheng, Longtao and Liu, Qian and Li, Yingru and Zheng, Xiaosen and MA, Zejun and An, Bo},
  booktitle={First Workshop on Multi-Turn Interactions in Large Language Models},
  year={2026}
}

@article{team2025kimi,
  title={Kimi k1. 5: Scaling reinforcement learning with llms},
  author={Team, Kimi and Du, Angang and Gao, Bofei and Xing, Bowei and Jiang, Changjiu and Chen, Cheng and Li, Cheng and Xiao, Chenjun and Du, Chenzhuang and Liao, Chonghua and others},
  journal={arXiv preprint arXiv:2501.12599},
  year={2025}
}

@article{yang2024qwen2,
  title={Qwen2. 5-math technical report: Toward mathematical expert model via self-improvement},
  author={Yang, An and Zhang, Beichen and Hui, Binyuan and Gao, Bofei and Yu, Bowen and Li, Chengpeng and Liu, Dayiheng and Tu, Jianhong and Zhou, Jingren and Lin, Junyang and others},
  journal={arXiv preprint arXiv:2409.12122},
  year={2024}
}

@article{guo2025deepseek,
  title={Deepseek-r1: Incentivizing reasoning capability in llms via reinforcement learning},
  author={Guo, Daya and Yang, Dejian and Zhang, Haowei and Song, Junxiao and Wang, Peiyi and Zhu, Qihao and Xu, Runxin and Zhang, Ruoyu and Ma, Shirong and Bi, Xiao and others},
  journal={arXiv preprint arXiv:2501.12948},
  year={2025}
}

@inproceedings{yao2023react,
  title={ReAct: Synergizing Reasoning and Acting in Language Models},
  author={Yao, Shunyu and Zhao, Jeffrey and Yu, Dian and Du, Nan and Shafran, Izhak and Narasimhan, Karthik and Cao, Yuan},
  booktitle={International Conference on Learning Representations},
  year={2023}
}

@inproceedings{gao2023pal,
  title={Pal: Program-aided language models},
  author={Gao, Luyu and Madaan, Aman and Zhou, Shuyan and Alon, Uri and Liu, Pengfei and Yang, Yiming and Callan, Jamie and Neubig, Graham},
  booktitle={International conference on machine learning},
  year={2023},
  organization={PMLR}
}

@article{schulman2017proximal,
  title={Proximal policy optimization algorithms},
  author={Schulman, John and Wolski, Filip and Dhariwal, Prafulla and Radford, Alec and Klimov, Oleg},
  journal={arXiv preprint arXiv:1707.06347},
  year={2017}
}

@article{yang2025qwen3,
  title={Qwen3 technical report},
  author={Yang, An and Li, Anfeng and Yang, Baosong and Zhang, Beichen and Hui, Binyuan and Zheng, Bo and Yu, Bowen and Gao, Chang and Huang, Chengen and Lv, Chenxu and others},
  journal={arXiv preprint arXiv:2505.09388},
  year={2025}
}

@inproceedings{ross2011reduction,
  title={A reduction of imitation learning and structured prediction to no-regret online learning},
  author={Ross, St{\'e}phane and Gordon, Geoffrey and Bagnell, Drew},
  booktitle={Proceedings of the Fourteenth International Conference on Artificial Intelligence and Statistics},
  year={2011},
}

@inproceedings{kumartraining,
  title={Training Language Models to Self-Correct via Reinforcement Learning},
  author={Kumar, Aviral and Zhuang, Vincent and Agarwal, Rishabh and Su, Yi and Co-Reyes, John D and Singh, Avi and Baumli, Kate and Iqbal, Shariq and Bishop, Colton and Roelofs, Rebecca and others},
  booktitle={The Thirteenth International Conference on Learning Representations},
  year={2025}
}

@article{allen2021learning,
  title={Learning markov state abstractions for deep reinforcement learning},
  author={Allen, Cameron and Parikh, Neev and Gottesman, Omer and Konidaris, George},
  journal={Advances in Neural Information Processing Systems},
  year={2021}
}

@article{lu2025scaling,
  title={Scaling llm multi-turn rl with end-to-end summarization-based context management},
  author={Lu, Miao and Sun, Weiwei and Du, Weihua and Ling, Zhan and Yao, Xuesong and Liu, Kang and Chen, Jiecao},
  journal={arXiv preprint arXiv:2510.06727},
  year={2025}
}

@inproceedings{chen2026toward,
    title={Toward Effective Tool-Integrated Reasoning via Self-Evolved Preference Learning},
    author={Yifei Chen and Guanting Dong and Zhicheng Dou},
    booktitle={The Fourteenth International Conference on Learning Representations},
    year={2026},
}

@article{JMLR:v25:23-0488,
  author  = {Yifan Zhong and Jakub Grudzien Kuba and Xidong Feng and Siyi Hu and Jiaming Ji and Yaodong Yang},
  title   = {Heterogeneous-Agent Reinforcement Learning},
  journal = {Journal of Machine Learning Research},
  year    = {2024},
}

@article{rafailov2023direct,
  title={Direct preference optimization: Your language model is secretly a reward model},
  author={Rafailov, Rafael and Sharma, Archit and Mitchell, Eric and Manning, Christopher D and Ermon, Stefano and Finn, Chelsea},
  journal={Advances in Neural Information Processing Systems},
  year={2023}
}

@ARTICLE{6145622,
  author={Browne, Cameron B. and Powley, Edward and Whitehouse, Daniel and Lucas, Simon M. and Cowling, Peter I. and Rohlfshagen, Philipp and Tavener, Stephen and Perez, Diego and Samothrakis, Spyridon and Colton, Simon},
  journal={IEEE Transactions on Computational Intelligence and AI in Games}, 
  title={A Survey of Monte Carlo Tree Search Methods}, 
  year={2012},
}

@article{yang2025treerpo,
  title={Treerpo: Tree relative policy optimization},
  author={Yang, Zhicheng and Guo, Zhijiang and Huang, Yinya and Liang, Xiaodan and Wang, Yiwei and Tang, Jing},
  journal={arXiv preprint arXiv:2506.05183},
  year={2025}
}

@article{cao2026treeadv,
  title={TreeAdv: Tree-Structured Advantage Redistribution for Group-Based RL},
  author={Cao, Lang and Ruan, Hui and Li, Yongqian and Chao, Peng and Ning, Wu and Song, Haonan and Chen, Renhong and Li, Yitong},
  journal={arXiv preprint arXiv:2601.03703},
  year={2026}
}

@article{wang2025ragen,
  title={Ragen: Understanding self-evolution in llm agents via multi-turn reinforcement learning},
  author={Wang, Zihan and Wang, Kangrui and Wang, Qineng and Zhang, Pingyue and Li, Linjie and Yang, Zhengyuan and Jin, Xing and Yu, Kefan and Nguyen, Minh Nhat and Liu, Licheng and others},
  journal={arXiv preprint arXiv:2504.20073},
  year={2025}
}

@inproceedings{
cao2026atpo,
title={ATPO: Adaptive tree policy optimization for Multi-turn medical dialogue},
author={Ruike Cao and Shaojie Bai and Fugen Yao and Liang Dong and Jian Xu and Li Xiao},
booktitle={The Fourteenth International Conference on Learning Representations},
year={2026}
}

@article{watkins1992q,
  title={Q-learning},
  author={Watkins, Christopher JCH and Dayan, Peter},
  journal={Machine learning},
  year={1992},
}

@article{tsitsiklis1994asynchronous,
  title={Asynchronous stochastic approximation and Q-learning},
  author={Tsitsiklis, John N},
  journal={Machine learning},
  year={1994},
}

@article{10.1023/A:1007678930559,
author = {Singh, Satinder and Jaakkola, Tommi and Littman, Michael L. and Szepesv\'{a}ri, Csaba},
title = {Convergence Results for Single-Step On-PolicyReinforcement-Learning Algorithms},
year = {2000},
journal = {Machine Learning}

}

@incollection{bertsekas2025neuro,
  title={Neuro-dynamic programming},
  author={Bertsekas, Dimitri P},
  booktitle={Encyclopedia of optimization},
  year={2025},
}

@ARTICLE{6796861,
  author={Jaakkola, Tommi and Jordan, Michael I. and Singh, Satinder P.},
  journal={Neural Computation}, 
  title={On the Convergence of Stochastic Iterative Dynamic Programming Algorithms}, 
  year={1994},
}

@inproceedings{10.5555/1777826.1777833,
author = {Coulom, R\'{e}mi},
title = {Efficient selectivity and backup operators in Monte-Carlo tree search},
year = {2006},
booktitle = {Proceedings of the 5th International Conference on Computers and Games}
}

@article{article,
author = {Silver, David and Huang, Aja and Maddison, Christopher and Guez, Arthur and Sifre, Laurent and Driessche, George and Schrittwieser, Julian and Antonoglou, Ioannis and Panneershelvam, Veda and Lanctot, Marc and Dieleman, Sander and Grewe, Dominik and Nham, John and Kalchbrenner, Nal and Sutskever, Ilya and Lillicrap, Timothy and Leach, Madeleine and Kavukcuoglu, Koray and Graepel, Thore and Hassabis, Demis},
year = {2016},
title = {Mastering the game of Go with deep neural networks and tree search},
journal = {Nature},
}

@inproceedings{tian2026seear,
title={{SEEA}-R1: Tree-Structured Reinforcement Fine-Tuning for Self-Evolving Embodied Agents},
author={Wanxin Tian and Shijie Zhang and Kevin Zhang and Xiaowei Chi and Chun-Kai Fan and Junyu Lu and Yulin Luo and Qiang Zhou and Yiming Zhao and Ning Liu and Siyu Lin and Zhiyuan Qin and Xiaozhu Ju and Shanghang Zhang and Jian Tang},
booktitle={The Thirty-ninth Annual Conference on Neural Information Processing Systems},
year={2026},
}

@article{robbins1951stochastic,
  title={A stochastic approximation method},
  author={Robbins, Herbert and Monro, Sutton},
  journal={The Annals of Mathematical Statistics},
  year={1951},
}

@article{wu2025resum,
  title={Resum: Unlocking long-horizon search intelligence via context summarization},
  author={Wu, Xixi and Li, Kuan and Zhao, Yida and Zhang, Liwen and Ou, Litu and Yin, Huifeng and Zhang, Zhongwang and Yu, Xinmiao and Zhang, Dingchu and Jiang, Yong and others},
  journal={arXiv preprint arXiv:2509.13313},
  year={2025}
}

@article{lindenbauer2025complexity,
  title={The Complexity Trap: Simple Observation Masking Is as Efficient as LLM Summarization for Agent Context Management},
  author={Lindenbauer, Tobias and Slinko, Igor and Felder, Ludwig and Bogomolov, Egor and Zharov, Yaroslav},
  journal={arXiv preprint arXiv:2508.21433},
  year={2025}
}

@article{wang2025openhands,
  title={The openhands software agent sdk: A composable and extensible foundation for production agents},
  author={Wang, Xingyao and Rosenberg, Simon and Michelini, Juan and Smith, Calvin and Tran, Hoang and Nyst, Engel and Malhotra, Rohit and Zhou, Xuhui and Chen, Valerie and Brennan, Robert and others},
  journal={arXiv preprint arXiv:2511.03690},
  year={2025}
}

@inproceedings{chang2026thor,
    title={{THOR}: Tool-Integrated Hierarchical Optimization via {RL} for Mathematical Reasoning},
    author={Qikai Chang and Zhenrong Zhang and Pengfei Hu and Jun Du and Jiefeng Ma and Yicheng Pan and Jianshu Zhang and Quan Liu and Jianqing Gao},
    booktitle={The Fourteenth International Conference on Learning Representations},
    year={2026},
}

@inproceedings{
liu2026roga,
title={{ROGA}: Scaling Generalist Agents for Office Productivity Tasks via Tool Generation},
author={Mugeng Liu and Xiaojun Ma and Yuhang Xie and Qin Chen and Xuanzhe Liu and Yun Ma},
booktitle={The Fourteenth International Conference on Learning Representations},
year={2026},
}

@inproceedings{ji2026tree,
title={Tree Search for {LLM} Agent Reinforcement Learning},
author={Yuxiang Ji and Ziyu Ma and Yong Wang and Guanhua Chen and Xiangxiang Chu and Liaoni Wu},
booktitle={The Fourteenth International Conference on Learning Representations},
year={2026},
}

@inproceedings{
dong2026agentic,
title={Agentic Reinforced Policy Optimization},
author={Guanting Dong and Hangyu Mao and Kai Ma and Licheng Bao and Yifei Chen and Zhongyuan Wang and Zhongxia Chen and Jiazhen Du and Huiyang Wang and Fuzheng Zhang and Guorui Zhou and Yutao Zhu and Ji-Rong Wen and Zhicheng Dou},
booktitle={The Fourteenth International Conference on Learning Representations},
year={2026},
}

@inproceedings{
yu2026memagent,
title={MemAgent: Reshaping Long-Context {LLM} with Multi-Conv {RL}-based Memory Agent},
author={Hongli Yu and Tinghong Chen and Jiangtao Feng and Jiangjie Chen and Weinan Dai and Qiying Yu and Ya-Qin Zhang and Wei-Ying Ma and Jingjing Liu and Mingxuan Wang and Hao Zhou},
booktitle={The Fourteenth International Conference on Learning Representations},
year={2026}
}

@article{
shojaee2023executionbased,
title={Execution-based Code Generation using Deep Reinforcement Learning},
author={Parshin Shojaee and Aneesh Jain and Sindhu Tipirneni and Chandan K. Reddy},
journal={Transactions on Machine Learning Research},
year={2023},
}

@inproceedings{
li2026ssvpo,
title={{SSVPO}: Effective Step-Level Credit Assignment for {RL} Training of Language Models},
author={Yugu Li and Zehong Cao and Jianglin Qiao and Siyi Hu},
booktitle={The Fourteenth International Conference on Learning Representations},
year={2026},
}

@article{shinn2023reflexion,
  title={Reflexion: Language agents with verbal reinforcement learning},
  author={Shinn, Noah and Cassano, Federico and Gopinath, Ashwin and Narasimhan, Karthik and Yao, Shunyu},
  journal={Advances in Neural Information Processing Systems},
  year={2023}
}

@article{le2022coderl,
  title={Coderl: Mastering code generation through pretrained models and deep reinforcement learning},
  author={Le, Hung and Wang, Yue and Gotmare, Akhilesh Deepak and Savarese, Silvio and Hoi, Steven Chu Hong},
  journal={Advances in Neural Information Processing Systems},
  year={2022}
}

@InProceedings{pmlr-v108-abel20a,
  title = 	 {Value Preserving State-Action Abstractions},
  author =       {Abel, David and Umbanhowar, Nate and Khetarpal, Khimya and Arumugam, Dilip and Precup, Doina and Littman, Michael},
  booktitle = 	 {Proceedings of the Twenty Third International Conference on Artificial Intelligence and Statistics},
  year = 	 {2020},
}

@article{yue2025vapo,
  title={Vapo: Efficient and reliable reinforcement learning for advanced reasoning tasks},
  author={Yue, Yu and Yuan, Yufeng and Yu, Qiying and Zuo, Xiaochen and Zhu, Ruofei and Xu, Wenyuan and Chen, Jiaze and Wang, Chengyi and Fan, TianTian and Du, Zhengyin and others},
  journal={arXiv preprint arXiv:2504.05118},
  year={2025}
}

@article{dietterich2000hierarchical,
  title={Hierarchical reinforcement learning with the MAXQ value function decomposition},
  author={Dietterich, Thomas G},
  journal={Journal of artificial intelligence research},
  year={2000}
}

@inproceedings{yu2026dapo,
    title={{DAPO}: An Open-Source {LLM} Reinforcement Learning System at Scale},
    author={Qiying Yu and Zheng Zhang and Ruofei Zhu and Yufeng Yuan and Xiaochen Zuo and YuYue and Weinan Dai and Tiantian Fan and Gaohong Liu and Juncai Liu and LingJun Liu and Xin Liu and Haibin Lin and Zhiqi Lin and Bole Ma and Guangming Sheng and Yuxuan Tong and Chi Zhang and Mofan Zhang and Ru Zhang and Wang Zhang and Hang Zhu and Jinhua Zhu and Jiaze Chen and Jiangjie Chen and Chengyi Wang and Hongli Yu and Yuxuan Song and Xiangpeng Wei and Hao Zhou and Jingjing Liu and Wei-Ying Ma and Ya-Qin Zhang and Lin Yan and Yonghui Wu and Mingxuan Wang},
    booktitle={The Thirty-ninth Annual Conference on Neural Information Processing Systems},
    year={2026},
}

@inproceedings{10.5555/2074094.2074120,
author = {Hauskrecht, Milos and Meuleau, Nicolas and Kaelbling, Leslie Pack and Dean, Thomas and Boutilier, Craig},
title = {Hierarchical solution of Markov decision processes using macro-actions},
year = {1998},
booktitle = {Proceedings of the Fourteenth Conference on Uncertainty in Artificial Intelligence},
}

@article{sheng2024hybridflow,
  title   = {HybridFlow: A Flexible and Efficient RLHF Framework},
  author  = {Guangming Sheng and Chi Zhang and Zilingfeng Ye and Xibin Wu and Wang Zhang and Ru Zhang and Yanghua Peng and Haibin Lin and Chuan Wu},
  year    = {2024},
  journal = {arXiv preprint arXiv: 2409.19256}
}

@inproceedings{kwon2023efficient,
  title={Efficient Memory Management for Large Language Model Serving with PagedAttention},
  author={Woosuk Kwon and Zhuohan Li and Siyuan Zhuang and Ying Sheng and Lianmin Zheng and Cody Hao Yu and Joseph E. Gonzalez and Hao Zhang and Ion Stoica},
  booktitle={Proceedings of the ACM SIGOPS 29th Symposium on Operating Systems Principles},
  year={2023}
}

@article{zhao2023pytorch,
  title={Pytorch fsdp: experiences on scaling fully sharded data parallel},
  author={Zhao, Yanli and Gu, Andrew and Varma, Rohan and Luo, Liang and Huang, Chien-Chin and Xu, Min and Wright, Less and Shojanazeri, Hamid and Ott, Myle and Shleifer, Sam and others},
  journal={arXiv preprint arXiv:2304.11277},
  year={2023}
}

@article{cobbe2021gsm8k,
  title={Training Verifiers to Solve Math Word Problems},
  author={Cobbe, Karl and Kosaraju, Vineet and Bavarian, Mohammad and Chen, Mark and Jun, Heewoo and Kaiser, Lukasz and Plappert, Matthias and Tworek, Jerry and Hilton, Jacob and Nakano, Reiichiro and Hesse, Christopher and Schulman, John},
  journal={arXiv preprint arXiv:2110.14168},
  year={2021}
}

@article{hendrycksmath2021,
  title={Measuring Mathematical Problem Solving With the MATH Dataset},
  author={Dan Hendrycks and Collin Burns and Saurav Kadavath and Akul Arora and Steven Basart and Eric Tang and Dawn Song and Jacob Steinhardt},
  journal={Thirty-Fifth Annual Conference on Neural Information Processing Systems},
  year={2021}
}

@misc{maa_aime_2024,
  title      = {American Invitational Mathematics Examination ({AIME})},
  author     = {{Mathematical Association of America}},
  publisher  = {Mathematical Association of America},
  year       = {2024},
}

@misc{maa_aime_2025, title = {American Invitational Mathematics Examination ({AIME})}, author = {{Mathematical Association of America}}, publisher = {Mathematical Association of America}, year = {2025} }

@misc{maa_amc_2023, title = {American Mathematics Competitions ({AMC})}, author = {{Mathematical Association of America}}, publisher = {Mathematical Association of America}, year = {2023}}

@article{he2024olympiadbench,
  title={Olympiadbench: A challenging benchmark for promoting agi with olympiad-level bilingual multimodal scientific problems},
  author={He, Chaoqun and Luo, Renjie and Bai, Yuzhuo and Hu, Shengding and Thai, Zhen Leng and Shen, Junhao and Hu, Jinyi and Han, Xu and Huang, Yujie and Zhang, Yuxiang and others},
  journal={arXiv preprint arXiv:2402.14008},
  year={2024}
}

@inproceedings{lightman2023let,
  title={Let's verify step by step},
  author={Lightman, Hunter and Kosaraju, Vineet and Burda, Yuri and Edwards, Harrison and Baker, Bowen and Lee, Teddy and Leike, Jan and Schulman, John and Sutskever, Ilya and Cobbe, Karl},
  booktitle={The Twelfth International Conference on Learning Representations},
  year={2023}
}

@article{sun2025simpledeepsearcher,
  title={Simpledeepsearcher: Deep information seeking via web-powered reasoning trajectory synthesis},
  author={Sun, Shuang and Song, Huatong and Wang, Yuhao and Ren, Ruiyang and Jiang, Jinhao and Zhang, Junjie and Bai, Fei and Deng, Jia and Zhao, Wayne Xin and Liu, Zheng and others},
  journal={arXiv preprint arXiv:2505.16834},
  year={2025}
}

@article{li2025websailor,
  title={Websailor: Navigating super-human reasoning for web agent},
  author={Li, Kuan and Zhang, Zhongwang and Yin, Huifeng and Zhang, Liwen and Ou, Litu and Wu, Jialong and Yin, Wenbiao and Li, Baixuan and Tao, Zhengwei and Wang, Xinyu and others},
  journal={arXiv preprint arXiv:2507.02592},
  year={2025}
}

@article{chen2025xbench,
  title={xbench: Tracking agents productivity scaling with profession-aligned real-world evaluations},
  author={Chen, Kaiyuan and Ren, Yixin and Liu, Yang and Hu, Xiaobo and Tian, Haotong and Xie, Tianbao and Liu, Fangfu and Zhang, Haoye and Liu, Hongzhang and Gong, Yuan and others},
  journal={arXiv preprint arXiv:2506.13651},
  year={2025}
}

@article{phan2025humanity,
  title={Humanity's last exam},
  author={Phan, Long and Gatti, Alice and Han, Ziwen and Li, Nathaniel and Hu, Josephina and Zhang, Hugh and Zhang, Chen Bo Calvin and Shaaban, Mohamed and Ling, John and Shi, Sean and others},
  journal={arXiv preprint arXiv:2501.14249},
  year={2025}
}

@inproceedings{wu2025webwalker,
  title={Webwalker: Benchmarking llms in web traversal},
  author={Wu, Jialong and Yin, Wenbiao and Jiang, Yong and Wang, Zhenglin and Xi, Zekun and Fang, Runnan and Zhang, Linhai and He, Yulan and Zhou, Deyu and Xie, Pengjun and others},
  booktitle={Proceedings of the 63rd Annual Meeting of the Association for Computational Linguistics},  
  year={2025}
}

@article{mialon2023gaia,
  title={Gaia: a benchmark for general ai assistants},
  author={Mialon, Gr{\'e}goire and Fourrier, Cl{\'e}mentine and Swift, Craig and Wolf, Thomas and LeCun, Yann and Scialom, Thomas},
  journal={arXiv preprint arXiv:2311.12983},
  year={2023}
}

@inproceedings{xanh2020_2wikimultihop,
    title = "Constructing A Multi-hop {QA} Dataset for Comprehensive Evaluation of Reasoning Steps",
    author = "Ho, Xanh  and
      Duong Nguyen, Anh-Khoa  and
      Sugawara, Saku  and
      Aizawa, Akiko",
    booktitle = "Proceedings of the 28th International Conference on Computational Linguistics",
    year = "2020",
}

@inproceedings{yang2018hotpotqa,
  title={HotpotQA: A dataset for diverse, explainable multi-hop question answering},
  author={Yang, Zhilin and Qi, Peng and Zhang, Saizheng and Bengio, Yoshua and Cohen, William and Salakhutdinov, Ruslan and Manning, Christopher D},
  booktitle={Proceedings of the 2018 Conference on Empirical Methods in Natural Language Processing},
  year={2018}
}

@article{loshchilov2017decoupled,
  title={Decoupled weight decay regularization},
  author={Loshchilov, Ilya and Hutter, Frank},
  journal={arXiv preprint arXiv:1711.05101},
  year={2017}
}

@article{hu2026seeupo,
  title={SeeUPO: Sequence-Level Agentic-RL with Convergence Guarantees},
  author={Hu, Tianyi and Fu, Qingxu and Chen, Yanxi and Liu, Zhaoyang and Ding, Bolin},
  journal={arXiv preprint arXiv:2602.06554},
  year={2026}
}

@article{shao2024deepseekmath,
  title={Deepseekmath: Pushing the limits of mathematical reasoning in open language models},
  author={Shao, Zhihong and Wang, Peiyi and Zhu, Qihao and Xu, Runxin and Song, Junxiao and Bi, Xiao and Zhang, Haowei and Zhang, Mingchuan and Li, YK and Wu, Yang and others},
  journal={arXiv preprint arXiv:2402.03300},
  year={2024}
}

@article{zheng2025group,
  title={Group sequence policy optimization},
  author={Zheng, Chujie and Liu, Shixuan and Li, Mingze and Chen, Xiong-Hui and Yu, Bowen and Gao, Chang and Dang, Kai and Liu, Yuqiong and Men, Rui and Yang, An and others},
  journal={arXiv preprint arXiv:2507.18071},
  year={2025}
}

\newpage

\appendix
\section{Related Work}
\label{related work}

\paragraph{Policy optimization in agentic RL.}
Recent post-training of large language models (LLMs) has shifted from supervised fine-tuning toward reinforcement learning with verifiable rewards (RLVR). Among existing approaches, GRPO \citep{shao2024deepseekmath}, built upon PPO, reduces variance through token-level importance ratios and group-relative advantage estimation, and has become a representative algorithm for agentic RL. Its sequence-level refinement, GSPO \citep{zheng2025group}, further defines importance ratios and clipping operations at the sequence level, leading to more stable training dynamics on models such as Qwen3. However, when these RL methods are directly transferred from single-turn to multi-turn agentic settings, a structural mismatch emerges between rollout and training. Existing methods often treat the entire multi-turn interaction as a single concatenated trajectory and distribute a single scalar reward uniformly across all tokens, thereby ignoring the actual contribution of each turn. More importantly, the rollout may operate on truncated or summarized contexts, whereas the training recomputes importance ratios over the full interaction history. This discrepancy induces a mismatch in conditioning distributions and weakens the assumptions under which PPO-style optimization is expected to remain stable. In long-horizon multi-turn scenarios, these issues can manifest as training divergence or eventual policy collapse.

A noteworthy latest work is SeeUPO \citep{hu2026seeupo}, which first models multi-turn interaction as a sequentially executed multi-agent bandit problem. Under the heterogeneous-agent RL framework \citep{JMLR:v25:23-0488}, SeeUPO updates turn-level virtual agents in reverse execution order $(T \to T-1 \to \cdots \to 1)$, thereby inheriting the monotonic-improvement property and proving convergence to the globally optimal policy. This provides strong motivation for our proposed RTPO design, particularly its reverse-order turn updates. However, each turn-level agent in SeeUPO is still trained on the full chat history rather than on the turn-level conditioning context $c_k$ actually observed by the model during rollout; therefore, the policy-forward mismatch is not eliminated. Moreover, within each turn, SeeUPO corrects downstream advantages using token-level importance-sampling reweighting. As the reverse training recursion proceeds and the number of involved turns accumulates, the variance of the importance-sampling ratio can grow multiplicatively. Although clipping bounds this ratio from above, it also introduces systematic bias into the per-turn advantage estimate.

\paragraph{Tree-based credit assignment.}
To mitigate the sparse-credit problem caused by flattened trajectories in multi-turn agentic RL training, a recent line of work reorganizes rollouts into tree structures with shared prefixes, thereby constructing finer-grained credit signals. ARPO \citep{dong2026agentic} adaptively tree branches at high-entropy nodes following tool calls and, through advantage attribution estimation, applies branch averaging to shared-prefix tokens while estimating advantages independently for tokens on disjoint branches. Afterward, TreeGRPO \citep{ji2026tree} abstracts each turn as a tree node and combines intra-tree and inter-tree group-relative advantages, theoretically establishing gradient-level equivalence with step-level DPO \citep{rafailov2023direct}. This tree-based formulation is extended in SEEA-R1 \citep{tian2026seear}, which integrates MCTS \citep{10.5555/1777826.1777833,6145622,article} into embodied-agent settings and trains a multimodal generative reward model to densify sparse outcome rewards. Tree-structured credit assignment has also been explored in step/token levels. TreeRPO \citep{yang2025treerpo} adopts an $N$-ary tree for mathematical reasoning and constructs step-level rewards through bottom-up Bellman expectations. TreeAdv \citep{cao2026treeadv} employs entropy-triggered branching and redistributes leaf advantages to tokens using inverse-descendant-count weighting. Similarly, ATPO \citep{cao2026atpo} operates under a hierarchical MDP and uses Bellman error and $Q$-value variance as uncertainty measures for adaptive expansion, while applying visit-count-based down-weighting to suppress update imbalance caused by repeated nodes.

Although these tree-based rollouts improve the granularity of credit assignment in multi-turn training, they do not fully eliminate the propagation of trajectory-level credit mismatch to individual turns through the reconstructed tree. For instance, in Tree-GRPO and SEEA-R1, advantages are still derived from leaf returns, typically in the form $A_i = \frac{R_i - \mathrm{mean}(R)}{\mathrm{std}(R)}$. As a result, prefix nodes that appear on multiple paths inherit trajectory-level signals from all descendant outcomes, causing the gradient direction at shared tokens to be perturbed by trajectory-level scalars originating from different rollouts. In TreeAdv, the $1/|S|$ normalization attenuates the signal more aggressively near the root, precisely suppressing early decisions where informative gradients are often most needed. In addition, ARPO averages multi-branch advantages on shared tokens, which can dilute the signal linearly with the number of branches, while ATPO relies on a learned critic and is therefore exposed to critic-induced bias. Therefore, under the rollout–training contextual mismatch settings, shared-prefix tokens may still carry trajectory-level credit signals with significant residual bias.

\paragraph{Policy gradient correction.}
Another line of work studies stable training from the perspective of correcting (off-)policy drift through interventions. SimpleTIR \citep{xuesimpletir}, by decomposing the policy gradient on softmax logits, attributes gradient explosion in multi-turn tool-integrated reasoning (TIR) to the accumulation of low-probability tokens under the distributional shift induced by tool feedback. It proposes filtering out entire trajectories that contain \emph{void turns}: turns produce neither a complete code block nor a final answer to block harmful gradients. This method is empirically effective and plug-and-play; however, the void-turn criterion is tightly coupled with the code-execution setting of mathematical reasoning and is difficult to transfer to other scenarios, such as web search. It also inevitably discards a non-trivial fraction of training data. RAGEN \citep{wang2025ragen} discovers the inconsistency between rollout and training engines for mismatch correction through numerical mechanisms such as truncated importance sampling. However, these approaches mainly address numerical discrepancies while overlooking the contextual mismatch between rollout conditioning and training recomputation. In contrast, our proposed RTPO does not rely on trajectory dropping or post-hoc numerical correction. Instead, it restructures the training paradigm from a turn-boundary MDP perspective, maintaining the on-policy optimization throughout the entire training process.

\section{Full Theoretical Analysis: Training Instability}
\label{full-theory}
Here, we provide the full theoretical analysis of training instability in multi-turn agentic RL:
\textbf{Single flattened-trajectory policy optimization (B.1):} Standard rollouts are typically generated under truncated or summarized contexts for efficiency, whereas training re-evaluates tokens under concatenated full-history contexts without truncation. This rollout--training context mismatch induces biased importance-sampling (IS) ratios, thereby undermining stable policy optimization. 
\textbf{Trajectory-only credit assignment (B.2):} In multi-turn interactions, a given state may admit multiple valid actions, requiring accurate credit assignment for each action. However, credit computed from a single trajectory-level advantage can obscure the contribution of individual turns and therefore cannot provide a proper comparison among alternative actions. 
\textbf{Long-horizon off-policy drift (B.3):} In long-horizon tasks, standard policy optimization with per-token clipping is insufficient to correct policy drift under asynchronous training. Since later states depend on earlier generated actions and tool feedback, once the policy drifts, the later-turn states visited during rollout may no longer match those induced by the current policy.

\paragraph{Preliminaries.}
A multi-turn interaction with TIR is represented as a standard trajectory $(q,l_0,f_0,\ldots,l_{n-1},f_{n-1})$ spanning $n$ turns, where $q$ is the initial prompt, $l_k$ is the response generated, and $f_k$ is the corresponding external tool feedback returned by the environment for each turn $k\in\{0,\dots,n-1\}$. We formulate this process as a \textit{Hierarchical Markov Decision Process (H-MDP)}, which captures turn-level (high-level) planning and token-level (low-level) execution in multi-turn agentic RL. At the turn level, for each turn $k$, the agent state $S_k$ represents all interaction history $\left(q, l_0, f_0, \ldots, l_{k-1}, f_{k-1}\right)$, available before the current turn, including previous responses, tool calls, and corresponding environment feedback. The corresponding turn-level action is the decision of what response to produce at the current turn, denoted by $l_k \in \mathcal{A}_H$. After executing this action and receiving environmental feedback $f_k$, the turn-level state evolves as $S_{k+1}=S_k \circ (l_k,f_k)$, where $\circ$ denotes concatenation. 

Each turn-level action (response) $l_k$ is generated autoregressively at the token level as $l_k=(a_{k,1}, a_{k,2},\ldots, a_{k, T_k})$, where $a_{k,t}\in\mathcal{A}_L$ denotes the token generated at step $t$, and $T_k$ is the number of tokens in the response at turn $k$. The corresponding low-level state is $s_{k,t}=(S_k, a_{k,1},\ldots, a_{k,t-1})\in\mathcal{S}_L$, which consists of the turn history $S_k$ together with the token prefix generated so far in the current turn. Since the token only serves to generate the turn-level action (response) and does not itself receive intermediate reward, we set the low-level reward to $R_L=0$ and set $\gamma_H=\gamma_L=1$ as the discount factor. 

Our hierarchical formulation differs from existing policy optimization methods, which do not explicitly distinguish turn-level planning from token-level execution. Instead, they flatten the entire multi-turn interaction into a single token sequence, i.e., a single trajectory, and optimize it using clipped policy optimization methods in the PPO family. Let $q\sim\mathcal{D}$ be an input prompt sampled from the task distribution, and let $\{g_i\}_{i=1}^{G}$ be a group of $G$ trajectories of this form sampled from the old policy $\pi_{\theta_{\mathrm{old}}}$ conditioned on $q$, where $\theta$ denotes the current policy parameters and $\theta_{\mathrm{old}}$ denotes the rollout policy parameters. For trajectory $g_i$, we denote its turn-$k$ state and response by $S_{i,k}$ and $l_{i,k}$, respectively:

\begin{equation}\label{eq:flat} J^{\text{flat}}(\theta) = \mathbb{E}_{\substack{q \sim \mathcal{D} \\ \{g_i\}_{i=1}^G \sim \pi_{\theta_{\mathrm{old}}}(\cdot \mid q)}} \left[ \frac{1}{G}\sum_{i=1}^{G} \frac{1}{\sum_{t'} m_{i,t'}} \sum_t m_{i,t} L_{i,t}^{\text{CLIP}}(\theta) \right] \end{equation}
where $i$ indexes the sampled trajectory, $t$ indexes flattened token positions, and $L_{i,t}^{\text{CLIP}}(\theta) = \min\!\left(\rho_{i,t}(\theta) A_i,\, \operatorname{clip}\!\left(\rho_{i,t}(\theta),\,1-\epsilon,\,1+\epsilon\right) A_i\right)$ is the standard clipped surrogate objective with clipping threshold $\epsilon$. $a_{i,t}$ is the token generated at flattened position $t$ in the $i$-th sampled trajectory, $m_{i,t}$ is a binary mask indicating whether that token contributes to the policy gradient, and $\rho_{i,t}(\theta)=\frac{\pi_\theta(a_{i,t}\mid x_{i,<t})}{\pi_{\theta_{\mathrm{old}}}(a_{i,t}\mid x_{i,<t})}$ is the IS ratio. The context $x_{i,<t}$ is the flattened interaction context preceding token $a_{i,t}$, including previously generated response tokens and inserted environment feedback. The scalar $A_i$ denotes the trajectory-level group-relative advantage, which is uniformly assigned to all unmasked tokens in trajectory $g_i$.

\subsection{Single flattened-trajectory policy optimization: mismatch from rollout to training}
\label{sec:append_truncation}
In multi-turn interactions, the accumulated history can become too long for the model to process in full. For example, an agent may search for information, call a tool, revise its plan based on the returned result, and repeat this process over many turns. By later turns, the prompt, previous responses, and tool feedback may already span tens of thousands of tokens. In practice, rollouts therefore often rely on a truncated or summarized context rather than the complete interaction history \citep{wang2025openhands,lindenbauer2025complexity,wu2025resum}. Therefore, existing methods typically reconstruct the whole interaction as a single concatenated sequence, or flattened trajectory, during training. This flattened training formulation evaluates each generated token under the concatenated prefix, rather than under the original context that was actually used during rollout. Consequently, tokens generated at turn $k$ are optimized under a conditioning context that can differ from the rollout context that produced them. We refer to this discrepancy as a rollout-to-training mismatch induced by flattened-trajectory policy optimization. We next analyze how this mismatch distorts the likelihood-ratio estimation underlying clipped policy optimization in the PPO/GRPO family.

\paragraph{Policy optimization mismatch across turns.}
Under flat training, the optimization mismatch across turns arises because the same generated token is conditioned on different contexts during rollout and training. We now formalize this mismatch and show that it induces a biased IS ratio. Let $\bar{x}_k$ denote the full interaction history before turn $k$, and let $\phi:\bar{\mathcal{X}}\rightarrow\mathcal{Z}$ be an observation map that truncates or summarizes history beyond the model's effective context length. During rollout, token-level actions are sampled conditioned on $\phi(\bar{x}_k)$, so the true sampling distribution is $\pi_{\theta_{\mathrm{old}}}(a_{k,t}\mid \phi(\bar{x}_k))$. In clipped policy optimization of the PPO/GRPO family, the denominator of the IS ratio must match this rollout distribution. However, under flat training, the same token is re-evaluated under the concatenated full-history context, yielding $\pi_{\theta_{\mathrm{old}}}(a_{k,t}\mid \bar{x}_k)$. This leads to the mismatch
\begin{equation}
\label{eq:is-mismatch}
\rho_{k,t}^{\mathrm{flat}}
=
\frac{\pi_\theta(a_{k,t}\mid \bar{x}_k)}
     {\pi_{\theta_{\mathrm{old}}}(a_{k,t}\mid \phi(\bar{x}_k))}
\neq
\frac{\pi_\theta(a_{k,t}\mid \phi(\bar{x}_k))}
     {\pi_{\theta_{\mathrm{old}}}(a_{k,t}\mid \phi(\bar{x}_k))}
=
\rho_{k,t}^{\mathrm{true}}.
\end{equation}
The denominator used in flat training, $\pi_{\theta_{\mathrm{old}}}(a_{k,t}\mid \bar{x}_k)$, does not match with the true rollout sampling probability, $\pi_{\theta_{\mathrm{old}}}(a_{k,t}\mid \phi(\bar{x}_k))$. Consequently, $\rho_{k,t}^{\mathrm{flat}}$ is a biased estimate of the correct IS ratio. Since clipped policy-gradient updates in the PPO/GRPO family depend on the IS ratio, the bias propagates into the gradient estimate and can distort the update direction, leading to instability in multi-turn agentic RL training. This mismatch becomes more severe for tokens generated in later turns, since the amount of history omitted during rollout typically grows with $k$, while flat training continues to re-evaluate these tokens under the concatenated training prefix. In the worst case, the resulting discrepancy in token probability can become extremely large.

\paragraph{State aliasing and projected suboptimality.}
Beyond gradient bias, truncated contexts also introduce \textit{state aliasing}, where distinct full-history states are mapped to the same truncated representation. When $\phi$ is non-injective, different interaction histories may collapse into an identical observation $z_k=\phi(\bar{x}_k)$ \citep{10.5555/1777826.1777833}. The induced process over $z_k$ may not preserve the Markov property of the original full-history process \citep{allen2021learning}. As a result, any policy conditioned only on $z_k$ is confined to the observation-induced policy class $\Pi\phi$. Even if optimized exactly within this restricted class, such a policy can achieve only the projected optimum $V^*_{\Pi_\phi}$ \citep{pmlr-v108-abel20a}, which can be strictly lower than the true optimum $V^*$. Therefore, projected suboptimality can arise, with $V^*_{\Pi_\phi} < V^*$. Under observation inconsistency, flattened-trajectory training may be limited to the projected optimum $V^*_{\Pi_\phi}$, which is strictly suboptimal relative to the full-history optimum $V^*$.

\subsection{Trajectory-only credit assignment: mismatch across low- and high-quality turns}
\label{sec:append_credit}

In multi-turn interactions, trajectories generated from the same query can reach substantially different states by turn $k$. As a result, a trajectory-level return no longer provides a reliable credit signal for evaluating actions taken at that turn. Under flattened-trajectory training, credit is assigned only at the trajectory level and then shared across tokens within a single trajectory, which mismatches the turn-level structure of the decision process. Here, we analyze how a trajectory advantage entangles the contribution of the current turn with both upstream state effect and downstream stochasticity.

\paragraph{A trajectory advantage entangles turn-level credit with context effects.}
A single trajectory advantage is typically defined as $A_i = R_i - \bar{R}$, where $R_i$ denotes the final return of trajectory $g_i$ and $\bar{R}$ is the average return over the sampled group. This trajectory advantage is then assigned to all turns in the trajectory, without identifying each turn's individual contributions. To expose the turn-level credit hidden in this trajectory return, we analyze its population counterpart, $R_i-\mathbb{E}[R]$, by introducing the conditional expectations $\mathbb{E}[R \mid S_{i,k}, l_{i,k}]$ and $\mathbb{E}[R \mid S_{i,k}]$:
\begin{equation}
\label{eq:traj-decompose}
R_i - \mathbb{E}[R]
=
\underbrace{R_i - Q^\pi(S_{i,k}, l_{i,k})}_{\xi_{\text{down}}:\; \text{downstream stochasticity}}
+
\underbrace{A^\pi(S_{i,k}, l_{i,k})}_{\text{true turn-}k\text{ advantage}}
+
\underbrace{V^\pi(S_{i,k}) - \mu_R}_{\xi_{\text{up}}:\; \text{upstream state effect}},
\end{equation}
where $Q^\pi(S_{i,k}, l_{i,k})=\mathbb{E}[R \mid S_{i,k}, l_{i,k}]$, $V^\pi(S_{i,k})=\mathbb{E}[R \mid S_{i,k}]$, and $\mu_R=\mathbb{E}[R]$. The true turn-$k$ advantage $A^\pi(S_{i,k}, l_{i,k})$ is therefore only one component of the trajectory advantage. The term $\xi_{\text{down}}$ captures downstream stochasticity after turn $k$, while $\xi_{\text{up}}$ captures variation induced by the upstream state reached before turn $k$. Whenever $\xi_{\text{down}} + \xi_{\text{up}}$ dominates $A^\pi(S_{i,k}, l_{i,k})$ in magnitude, the sign of the trajectory advantage can disagree with that of the true turn advantage, i.e., $\mathrm{sign}(A_i) \neq \mathrm{sign}(A^\pi(S_{i,k}, l_{i,k}))$, thereby reversing the gradient direction for turn $k$. Therefore, a trajectory advantage is not a valid turn-level credit signal for turn $k$, because it entangles the context effect of the current turn with both upstream state effects and downstream stochasticity. When these two terms dominate, the resulting policy update can assign incorrect credit to the current turn and may even reverse the intended policy gradient direction.

\paragraph{Cross-trajectory baseline with state bias.}
Since group-based policy optimization methods generate multiple trajectories from the same input query by sampling several rollouts from the current or old policy, they typically compute advantages of the form $A_i = R_i - b$, where $R_i$ is the return of trajectory $g_i$ and $b = \frac{1}{G}\sum_{j=1}^{G} R_j$ is the group baseline. To determine whether the compared trajectories provide a valid turn-level credit signal, we examine whether they share the same turn-level state $S_{i,k}$. If they do, the baseline compares alternative outcomes from the same context and is therefore matched. Otherwise, the baseline mixes returns from different states and no longer reflects the local effect of the current turn, leading to mismatched credit assignment.

Formally, viewing $A_i$ as an estimator of the true turn-$k$ advantage $A^\pi(S_{i,k},l_{i,k})$ gives
\begin{equation}
\label{eq:baseline-general}
\mathbb{E}[A_i \mid S_{i,k}, l_{i,k}]
=
\frac{G-1}{G}
\Big(
Q^\pi(S_{i,k}, l_{i,k})
-
\underbrace{\mathbb{E}[R_j \mid S_{i,k}]}_{\text{depends on grouping}}
\Big),
\qquad j\neq i .
\end{equation}

In the case of same-state grouping, where all trajectories in the group share the same state at turn $k$, we have $\mathbb{E}[R_j \mid S_{i,k}] = V^\pi(S_{i,k})$. In this case, the group baseline is anchored to the correct local decision context, and the resulting estimator differs from the true turn-level advantage only by the multiplicative factor $(G-1)/G$. However, under cross-state grouping, trajectories within the same group may already reach different states by turn $k$. Then the baseline is no longer tied to the local state $S_{i,k}$, but is effectively centered around the global mean value as return, i.e., $\mathbb{E}[R_j] = \mu_R$. This introduces the additional bias:
\begin{equation}
\label{eq:cross-state-bias}
\mathrm{Bias}_{\text{cross}}
=
\frac{G-1}{G}\big(V^\pi(S_{i,k}) - \mu_R\big).
\end{equation}

This bias is determined entirely by the upstream trajectory and has no causal connection to the action taken at turn $k$. In multi-turn interactions, trajectories often diverge into semantically distinct environment states by turn $k$. For example, one trajectory may issue a search query while another is executing code. As a result, the variance of $V^\pi(S_{i,k})$ across states can be large, making the cross-state bias comparable to, or even larger than, the true turn advantage $A^\pi(S_{i,k},l_{i,k})$. Therefore, a cross-trajectory baseline cannot guarantee a valid local comparison signal for turn-level credit assignment.

\subsection{Long-horizon policy drift: PPO clipping mismatch asynchronous turns}
\label{append_PPO_clipping}
In long-horizon multi-turn training, policy updates can become asynchronous across turns, inducing turn policy drift. In practice, PPO-style methods use token clipping to constrain the IS ratio and limit policy drift during updates. However, this clipping mechanism is designed for near-on-policy updates with synchronized rollout data and does not explicitly account for turn-wise discrepancies when different parts of a trajectory are generated or optimized under different policy versions. Over long horizons, such mismatches can accumulate, shifting the update away from the near-on-policy learning and toward an off-policy setting.

This occurs because, in asynchronous multi-turn training, policy updates may be performed before all trajectories have completed their rollouts. For instance, shorter trajectories may complete first and immediately contribute to an update, while longer trajectories are still being generated under an older policy. Specifically, short trajectories may finish first and update the policy from $\theta_0$ to $\theta_1$, while long trajectories are still being rolled out under $\pi_{\theta_0}$. By the time these longer trajectories are used for training, the current policy $\pi_{\theta_1}$ no longer matches the policy that generated them.

In principle, this mismatch can be corrected by the full trajectory IS ratio:
\begin{equation}
\label{eq:full-is}
\omega_i
=
\prod_{t=1}^{T_i} \rho_{i,t}
=
\prod_{t=1}^{T_i}
\frac{\pi_{\theta_1}(a_{i,t} \mid x_{i,<t})}{\pi_{\theta_0}(a_{i,t} \mid x_{i,<t})},
\end{equation}
where $\rho_{i,t}$ is the token IS ratio, $x_{i,<t}$ is the flattened multi-turn interaction context preceding token $a_{i,t}$, and $T_i$ is the total number of tokens in the flattened trajectory $g_i$. Weighting the loss by $\omega_i$ would yield an unbiased correction to the policy objective. However, clipped policy optimization in the PPO/GRPO family does not employ an IS ratio for the full trajectory. Instead, it clips each token ratio $\rho_{i,t}$ independently:
\begin{equation}
\label{eq:ppo-clip}
\mathcal{L}_{\text{clip}}
=
\sum_{t=1}^{T_i}
\min\!\big(
\rho_{i,t}\, A_i,
\;
\operatorname{clip}(\rho_{i,t},\, 1{-}\epsilon,\, 1{+}\epsilon)\, A_i
\big).
\end{equation}

Therefore, if one composes the clipped token-level ratios into a trajectory-level correction, it generally differs from the true full-trajectory IS ratio:
\begin{equation}
\label{eq:clip-product}
\prod_{t=1}^{T_i} \operatorname{clip}(\rho_{i,t},\, 1{-}\epsilon,\, 1{+}\epsilon)
\neq
\prod_{t=1}^{T_i} \rho_{i,t}
=
\omega_i.
\end{equation}

Because PPO clipping is nonlinear, the product of clipped token-level ratios does not equal the true trajectory-level importance weight. In particular, each clipped ratio lies in $[1{-}\epsilon,\, 1{+}\epsilon]$, so their product is restricted to $[(1{-}\epsilon)^T,\,(1{+}\epsilon)^T]$, whereas the true $\omega$ can in principle take any value in $(0,\infty)$. Whenever $\omega$ falls outside this interval, per-token clipping necessarily yields a biased trajectory correction. Even when $\omega$ lies within the interval, the product of individually clipped ratios will generally differ from $\omega$ as soon as any $\rho_t$ is clipped, since clipping and multiplication do not commute. This discrepancy compounds with trajectory length $T$: as more tokens are clipped, the gap between $\prod_t \mathrm{clip}(\rho_t)$ and $\omega$ can grow progressively larger. Removing clipping and using the exact $\omega$ is not a practical solution, because its variance grows exponentially with $T$, making gradient estimates increasingly uninformative over the long horizons typical of multi-turn agentic RL. Hence, token-level clipping cannot faithfully reproduce the trajectory-level correction required for long-horizon policy drift.

\section{Method: RTPO Theoretical Proofs}
\label{sec:appendix-theory}

\subsection{Notation, Formal Problem Setup, and Technical Assumptions}
\label{app:notation}

\paragraph{Notation convention.}
We create a multi-turn agentic episode, which consists of $K$ turn-level interactions, written as $(q,l_0,f_0,\ldots,l_{K-1},f_{K-1})$, where $q$ is the initial prompt, $l_k$ is the complete response generated by the agent at turn $k$, and $f_k$ is the external tool or environment feedback returned after executing $l_k$. The turn index is denoted by $k\in\{0,\ldots,K-1\}$, while the token index within a response is denoted by $t$. When multiple rollouts are sampled, we use $i$ or $j$ to index the rollout or sibling trajectory. Thus, $S_{i,k}$ denotes the turn-$k$ state in rollout $i$, and $a_{i,k,t}$ denotes the $t$-th token generated at turn $k$ in rollout $i$. When no rollout index is needed, we can write $S_k$, $l_k$, and $a_{k,t}$ for a generic trajectory. We distinguish between a turn-level macro-action and its token-level realization. The macro-action at turn $k$ is denoted by $u_k\equiv l_k$, where $l_k=(a_{k,1},\ldots,a_{k,T_k})$ is the complete response and $T_k$ is the number of generated tokens in that response. The notation $u_k$ is used in the MDP and value-function definitions, while $l_k$ emphasizes that the macro-action is implemented as a language-model response. In addition, the reward in a task is assigned at the turn or trajectory level. Tokens inside a response are treated as the low-level realization of the turn-level macro-action and do not receive separate intermediate rewards. Thus, the low-level token reward is set to zero. We write $r_k$ for the immediate turn-level reward at turn $k$. In sparse-reward tasks, we typically have $r_k=0$ for $k<K-1$, and the final task reward is observed only after the terminal turn. 

\paragraph{Turn-boundary MDP for rollout-training match.}
A $K$-turn agentic episode is modelled at turn boundaries as $\mathcal{M}=\langle \bar{\mathcal{X}},\mathcal{A}_H,P_H,R_H,\gamma_H\rangle$, where $\bar{\mathcal{X}}$ is the augmented turn-boundary state space, $\mathcal{A}_H$ is the high-level action space of complete responses, $P_H$ is the transition kernel induced by executing a complete response and receiving tool feedback, $R_H$ is the turn-level reward function, and $\gamma_H$ is the turn-level discount factor. The augmented state at turn $k$ is $\bar{x}_k=(S_k,k)\in\bar{\mathcal{X}}$, where $S_k=(q,l_0,f_0,\ldots,l_{k-1},f_{k-1})$ is the interaction history available before the current turn. Adding the turn index $k$ in $\bar{x}_k$ can make the process Markov over a finite horizon, since the remaining number of turns can affect both the available decisions and the continuation value. At turn $k$, the agent selects a macro-action $u_k\equiv l_k\in\mathcal{A}_{H,k}(S_k)$, where $\mathcal{A}_{H,k}(S_k)$ denotes the set of feasible complete responses at state $S_k$ and turn $k$. The macro-action is realized autoregressively as $l_k=(a_{k,1},\ldots,a_{k,T_k})$ and consumes $\tau_k=T_k$ token-generation steps. After executing $u_k$ and receiving environment feedback $f_k$, the turn-level history is updated by concatenation as $S_{k+1}=S_k\circ(l_k,f_k)$. Equivalently, since $u_k\equiv l_k$, one may write $S_{k+1}=S_k\circ(u_k,f_k)$. In Sec~\ref{sec:method}, we adopt this MDP terminology for the turn-boundary process. Strictly speaking, because each macro-action may span a variable number of token-level steps $\tau_k$, this process can alternatively be formulated as a finite-horizon semi-MDP. This distinction does not affect our analysis, since policy optimization and credit assignment are defined exclusively over turn-boundary states and macro-actions.

To maintain rollout--training consistency under conditional contexts, the policy need not condition on the full interaction history $S_k$. Instead, it may condition on a compressed or truncated context, provided that the same conditional context is used consistently during both rollout generation and policy optimization. Instead, the actual conditioning context at turn $k$ is $c_k=\psi(S_k)$, where $\psi$ may be the identity map, a truncation operator, or a summarization operator. This distinction is important in long-context multi-turn training: rollout may be performed under a truncated or summarized context, while training may otherwise recompute log-probabilities under a different context. RTPO avoids this mismatch by recording the exact context $c_k$ used during rollout and reusing the same $c_k$ during training. The turn-level policy factorizes as $\pi_\theta=(\pi_{\theta,0},\ldots,\pi_{\theta,K-1})$, where each sub-policy maps the turn-level context $c_k$ to a complete response. Each sub-policy is implemented autoregressively:
\begin{equation}
\pi_{\theta,k}(u_k\mid c_k)
=
\prod_{t=1}^{T_k}
\pi_\theta(a_{k,t}\mid c_k,a_{k,<t}).
\label{eq:subpolicy}
\end{equation}
Here, $a_{k,<t}=(a_{k,1},\ldots,a_{k,t-1})$ is the token prefix generated within the current turn. Because the same $c_k$ is used in rollout and training, the token-level importance-sampling (IS) ratio for turn $k$ is evaluated under matched conditioning contexts:
\begin{equation}
\rho_{k,t}
=
\frac{
\pi_\theta(a_{k,t}\mid c_k,a_{k,<t})
}{
\pi_{\theta_{\mathrm{old}}}(a_{k,t}\mid c_k,a_{k,<t})
}.
\end{equation}

\paragraph{Policy optimization with recursive optimality.}
A policy sequence $\pi^{\mathrm{rec}}=(\pi_0^*,\ldots,\pi_{K-1}^*)$ is recursively optimal if, for every turn $k\in\{0,\ldots,K-1\}$, the turn-$k$ sub-policy $\pi_k^*$ maximizes the turn-level augmented value given that all downstream sub-policies have already been optimized and fixed. Formally, for each $k$, $\pi_k^*\in\arg\max_{\pi_k}\tilde{Q}_k^{\pi^*}(S_k,u_k)$, where the downstream policies $\pi_{k+1}^*,\ldots,\pi_{K-1}^*$ are treated as fixed during the optimization of $\pi_k^*$. This is a backward-induction notion of optimality: the last turn is optimized first, then the preceding turn is optimized assuming the last-turn policy is fixed, and so on until the first turn.

Furthermore, we establish convergence to recursive optimality in a tabular finite-horizon setting. Assumptions below (A1--A10) are not intended to model the full neural implementation, but instead serve to isolate the theoretical effect of reverse-order turn-level policy optimization under standard stochastic approximation conditions.
\begin{enumerate}[label=(A\arabic*),leftmargin=*,itemsep=2pt]
\item \textit{Finite turn-boundary spaces.} The augmented state space $\bar{\mathcal{X}}$ is finite, and for every turn $k$ and every reachable state $S_k$, the feasible high-level action set $\mathcal{A}_{H,k}(S_k)$ is finite.
\item \textit{Proper finite-horizon episodes.} The number of turns $K$ is finite. For every policy and every turn $k$, the macro-action duration satisfies $\mathbb{E}_{\pi_k}[\tau_k]<\infty$. Thus, each turn terminates almost surely in a finite expected token length.
\item \textit{Discounting or finite-horizon boundedness.} Either $\gamma_H\in(0,1)$, or the problem is finite-horizon with $K<\infty$ and $\gamma_H\in(0,1]$. The latter case includes the undiscounted finite-horizon setting $\gamma_H=1$.
\item \textit{Tabular value representation.} The value estimate $Q_k(S,u)$ is stored separately for each turn-state-action tuple $(S,u,k)$. This assumption avoids approximation error and allows the proof to focus on the stochastic approximation dynamics induced by reverse-order updates.
\item \textit{Reverse-order training with downstream freezing.} Training is proceeding in the order $k=K-1,K-2,\ldots,0$. During turn $k$, the turn-$k$ policy is updated while the downstream policies $\pi_{k+1:K-1}$ are held fixed. After turn $k$ is completed, the turn-$k$ policy is also frozen before moving to turn $k-1$.
\item \textit{GLIE exploration within each turn.} Within turn $k$, the exploration schedule is greedy in the limit with infinite exploration (GLIE): every feasible turn-level action is selected with strictly positive probability infinitely often, while the policy becomes greedy in the limit. This ensures that all relevant action values at turn $k$ are sufficiently sampled before the policy is frozen.
\item \textit{State coverage.} Every reachable boundary state $S_k$ that can arise under the training process is visited infinitely often during turn $k$. This condition ensures that the tabular value estimate for each relevant state-action pair receives infinitely many updates.
\item \textit{Robbins--Monro step sizes.} For every tuple $(S,u,k)$, the learning rates satisfy $\sum_{n=1}^{\infty}\alpha_n(S,u,k)=\infty$ and $\sum_{n=1}^{\infty}\alpha_n^2(S,u,k)<\infty$. These are the standard stochastic approximation step-size conditions.
\item \textit{Bounded rewards.} The turn-level rewards are uniformly bounded: $|r_k|\le r_{\max}$ almost surely for all $k$.
\item \textit{Bounded iterates.} The tabular value iterates remain uniformly bounded: $|Q_k(S,u)|\le Q_{\max}$ throughout learning. This assumption is standard in stochastic approximation analyses and can be enforced by projection if necessary.
\item \textit{Decision sufficiency of the conditioning context.}
The observation mapping $\psi\colon S_k \mapsto c_k$ preserves all 
decision-relevant information for following turns. In other words, states that are indistinguishable under $\psi$ share 
the same optimal action-value function, so the optimal policy at turn 
$k$ depends on $S_k$ only through $c_k$. This implies that 
$\tilde{Q}_k^{\pi^*}$ can be written as a function of $(c_k, u)$ 
without loss, and the completeness condition 
$\Pi_{\mathrm{global}} \subseteq \{(\pi_0, \ldots, \pi_{K-1}) : 
\pi_k(\cdot \mid c_k)\}$ in Theorem~\ref{thm:recursive-app}(c) is 
automatically satisfied.
\end{enumerate}

Assumptions (A1)–(A10) are standard regularity conditions for the convergence of reinforcement learning algorithms; identical or closely analogous conditions appear in the foundational convergence proofs of Q-learning \citep{watkins1992q,tsitsiklis1994asynchronous,6796861}, on-policy GLIE control \citep{10.1023/A:1007678930559}, and the systematic treatment in \citet{bertsekas2025neuro}. They are not specific to RTPO but rather constitute the minimal set of conditions under which any stochastic-approximation-based value-learning algorithm is known to converge. These assumptions (A1--A11) are intentionally stronger than those required in the practical neural network implementation. They are used to make the convergence argument mathematically clean in the tabular setting. In the realistic RTPO implementation, the policy is represented by a shared neural language model rather than by independent tabular sub-policies. Therefore, ``freezing'' a turn-level sub-policy should be interpreted operationally: after a turn is completed, subsequent turns mask out the corresponding turn tokens from the loss, so that those turn-level decisions no longer receive gradients. The tabular analysis should thus be read as an idealized counterpart that clarifies the role of reverse-turn optimization and downstream-policy freezing, rather than as a claim of global convergence for arbitrary neural function approximation. The extension to neural function approximation is discussed separately in Proposition \ref{app:nn-extension}. We do not rely on the universal approximation property alone to claim convergence of the neural algorithm; instead, the tabular result serves as a principled limiting case that motivates the reverse-turn training design.

\paragraph{Verification of decision sufficiency (A11) in RTPO.}
In the practical RTPO implementation, $\psi$ corresponds to the 
chat-template truncation operator that removes the model's internal 
reasoning trace (the content of \texttt{<think>} blocks) while 
preserving all externally observable elements: tool calls, tool results, 
and final answers. Two states $S_k$ and $S_k'$ that differ only in 
their internal reasoning traces satisfy $\psi(S_k) = \psi(S_k')$. 
Since the environment transition kernel $P_H$ depends exclusively on the 
executed tool calls and the returned feedback, not on the model's 
internal reasoning, the next-state distribution and thus the 
continuation value $F_k^{\pi^*}(S_{k+1})$ are identical for $S_k$ and 
$S_k'$. The turn-level reward $r_k$ likewise depends only on the 
externally observable action. Therefore, 
$\tilde{Q}_k^{\pi^*}(S_k, u) = \tilde{Q}_k^{\pi^*}(S_k', u)$ for all 
$u$, and Assumption~(A11) is satisfied. 
When $\psi$ is the identity map (full-history conditioning), (A11) 
holds trivially.

\paragraph{Turn-level value estimation.}
We formulate the interaction as a hierarchical MDP, which decomposes the value of each turn into the local effect of the current macro-action and the downstream completion value, following a MAXQ-style value decomposition. The downstream continuation value after executing turn $k$ is defined as
\begin{equation}
F_k^\pi(S_{k+1})
:=
\mathbb{E}_\pi\!\left[
  \sum_{j=k+1}^{K-1}
  \gamma_H^{\sum_{m=k+1}^{j-1}\tau_m}
  r_j
  \;\middle|\;
  S_{k+1}
\right],
\qquad
F_{K-1}^\pi \equiv 0 .
\label{eq:Fk-def}
\end{equation}
The base case $F_{K-1}^\pi\equiv 0$ reflects that there are no downstream turns after the last turn. The exponent $\sum_{m=k+1}^{j-1}\tau_m$ accounts for the number of token-generation steps elapsed between the next state $S_{k+1}$ and the future reward $r_j$. When $\gamma_H=1$, this reduces to the undiscounted finite-horizon setting used in many sparse-reward agentic RL tasks. The boundary-level augmented action-value function is:
\begin{equation}
\tilde{Q}_k^\pi(S_k,u_k)
:=
\mathbb{E}\!\left[
r_k+\gamma_H^{\tau_k}F_k^\pi(S_{k+1})
\;\middle|\;
S_k,u_k
\right].
\label{eq:Q-tilde}
\end{equation}
This value measures the expected return obtained by choosing the complete turn-$k$ response $u_k$ at state $S_k$, followed by downstream policy execution from $S_{k+1}$ onward. In sparse-reward settings, the immediate term $r_k$ is often zero for nonterminal turns, so the quality of a turn-level action is primarily reflected through the downstream continuation value. The corresponding optimal value satisfies $\tilde{Q}_k^*(S_k,u_k)=\mathbb{E}[r_k+\gamma_H^{\tau_k}F_k^*(S_{k+1})\mid S_k,u_k]$, where $F_k^*$ is obtained by the usual backward Bellman optimality recursion over turn boundaries.

\subsection{Proof of Theorem 1: Convergence to Recursive Optimality}
\label{app:convergence-proof}
We formalize Theorem~1 from Sec~\ref{sec:turn-opt} before presenting the proof.

\setcounter{theorem}{0}
\begin{tcolorbox}[
    colback=gray!5,
    colframe=black!60,
    title=Reverse-Order Training to Recursive Optimality,
    fonttitle=\bfseries,
    breakable
]
\begin{theorem}[\textbf{Convergence to Recursive Optimality}]
\label{thm:recursive-app}
Under Assumptions~(A1)--(A10), including finite state and macro-action spaces, Robbins--Monro step sizes, sufficient exploration at each turn, exact freezing of optimized downstream policies, and bounded rewards, the reverse-order backward induction of turn-level policy optimization satisfies the following properties:

\begin{enumerate}[label=(\alph*),leftmargin=*,itemsep=4pt]

\item \textbf{Per-turn convergence under fixed downstream policies.}
For each turn $k$, once the downstream policies $\pi_{k+1:K-1}^*$ are fixed, the downstream continuation value becomes a stationary function:
\begin{equation}
F_k^{\pi^*}(S_{k+1})
=
\mathbb{E}_{\pi^*_{k+1:K-1}}\!\left[
  \sum_{j=k+1}^{K-1}
  \gamma_H^{\sum_{m=k+1}^{j-1}\tau_m} r_j
  \;\middle|\;
  S_{k+1}
\right].
\label{eq:thm-F-fixed}
\end{equation}
Therefore, the turn-$k$ target
\begin{equation}
Y_k
:=
r_k+\gamma_H^{\tau_k}F_k^{\pi^*}(S_{k+1})
\label{eq:thm-Yk}
\end{equation}
is stationary conditional on $(S_k,u_k)$, with conditional mean
\begin{equation}
\mathbb{E}[Y_k\mid S_k,u_k]
=
\tilde{Q}_k^{\pi^*}(S_k,u_k).
\label{eq:thm-Yk-mean}
\end{equation}
The corresponding tabular update
\begin{equation}
Q_k^{(n+1)}(S,u)
\leftarrow
(1-\alpha_n)Q_k^{(n)}(S,u)+\alpha_n Y_k
\label{eq:thm-single-turn-update}
\end{equation}
is a single-step stochastic approximation problem and converges to $\tilde{Q}_k^{\pi^*}$ under the stated assumptions. Hence, the learned per-turn policy satisfies $\hat{\pi}_k \xrightarrow{\mathrm{a.s.}} \pi_k^*$.

\item \textbf{Recursive optimality under reverse-order training.}
Applying the per-turn convergence argument in reverse order,
\begin{equation}
k=K{-}1,K{-}2,\ldots,0,
\label{eq:thm-reverse-order}
\end{equation}
yields a sequence of optimized and frozen downstream policies. At each turn $k$, the per-turn optimizer converges to
\begin{equation}
\pi_k^*(S_k)
\in
\arg\max_{u_k\in\mathcal{A}_{H,k}(S_k)}
\tilde{Q}_k^{\pi^*}(S_k,u_k),
\label{eq:thm-recursive-selector}
\end{equation}
given the fixed downstream policies $\pi_{k+1}^*,\ldots,\pi_{K-1}^*$. The resulting policy sequence
\begin{equation}
\pi^{\mathrm{rec}}
=
(\pi_0^*,\ldots,\pi_{K-1}^*)
\label{eq:thm-recursive-policy}
\end{equation}
is recursively optimal.

\item \textbf{Global optimality under context-sufficient macro-action completeness.}
If the turn-level macro-action spaces are complete with respect to the conditioning contexts $c_k=\psi(S_k)$, i.e., every globally feasible trajectory-level policy can be represented by a sequence of turn-level policies acting on these contexts,
\begin{equation}
\Pi_{\mathrm{global}}
\subseteq
\left\{
(\pi_0,\ldots,\pi_{K-1})
:
\pi_k(\cdot\mid c_k),
\;
u_k\in\mathcal{A}_{H,k}(S_k)
\right\},
\label{eq:thm-completeness}
\end{equation}
Then, recursive optimality is equivalent to global optimality over the full trajectory:
\begin{equation}
V^{\pi^{\mathrm{rec}}}(\bar{x}_0)
=
V^*(\bar{x}_0).
\label{eq:thm-global-optimality}
\end{equation}

\end{enumerate}
\end{theorem}
\end{tcolorbox}

\begin{proof}
We prove parts~(a)--(c) of Theorem~1 as follows.

\paragraph{Proof of Theorem~1(a): per-turn convergence under fixed downstream policies.}
Fix a turn $k$ and suppose the downstream sub-policies $\pi_{k+1}^*,\ldots,\pi_{K-1}^*$ are already optimized and frozen. By the definition in Theorem~\ref{thm:recursive-app}, the downstream continuation value is
\begin{equation}
F_k^{\pi^*}(S_{k+1})
=
\mathbb{E}_{\pi^*_{k+1:K-1}}\!\left[
  \sum_{j=k+1}^{K-1}
  \gamma_H^{\sum_{m=k+1}^{j-1}\tau_m} r_j
  \;\middle|\;
  S_{k+1}
\right].
\label{eq:proof-Fk}
\end{equation}
Because the downstream policies $\pi_{k+1:K-1}^*$ are frozen, $F_k^{\pi^*}(S_{k+1})$ is a fixed function of $S_{k+1}$ throughout turn $k$. Define the augmented one-step target
\begin{equation}
Y_k
=
r_k+\gamma_H^{\tau_k}F_k^{\pi^*}(S_{k+1}).
\label{eq:proof-target}
\end{equation}
Conditional on $(S_k,u_k)$, the distribution of $r_k$, $\tau_k$, and $S_{k+1}$ is induced by the time-invariant turn-boundary transition kernel $P_H$. Since $F_k^{\pi^*}$ is fixed throughout turn $k$, the conditional distribution of $Y_k$ given $(S_k,u_k)$ is stationary.

The conditional mean of the target is exactly the augmented turn-level action value:
\begin{equation}
\mathbb{E}[Y_k\mid S_k,u_k]
=
\mathbb{E}\!\left[
r_k+\gamma_H^{\tau_k}F_k^{\pi^*}(S_{k+1})
\;\middle|\;
S_k,u_k
\right]
=
\tilde{Q}_k^{\pi^*}(S_k,u_k).
\label{eq:proof-target-mean}
\end{equation}
The target also has a bounded second moment. Since rewards are bounded, $|r_j|\le r_{\max}$ almost surely. Hence, for any policy $\pi$,
\begin{align}
\bigl|F_k^\pi(S_{k+1})\bigr|
&=
\left|
\mathbb{E}_\pi\!\left[
  \sum_{j=k+1}^{K-1}
  \gamma_H^{\sum_{m=k+1}^{j-1}\tau_m} r_j
  \;\middle|\;
  S_{k+1}
\right]
\right| \notag \\
&\le
\mathbb{E}_\pi\!\left[
  \sum_{j=k+1}^{K-1}
  \gamma_H^{\sum_{m=k+1}^{j-1}\tau_m} |r_j|
  \;\middle|\;
  S_{k+1}
\right] \notag \\
&\le
(K-k-1)r_{\max}
\le
K r_{\max},
\label{eq:proof-F-bound}
\end{align}
where we used $\gamma_H\in(0,1]$ and the finite horizon $K$. Therefore,
\begin{equation}
|Y_k|
\le
|r_k|+\gamma_H^{\tau_k}|F_k^{\pi^*}(S_{k+1})|
\le
(K+1)r_{\max},
\label{eq:proof-Y-bound}
\end{equation}
so $Y_k$ has uniformly bounded conditional second moment.

During turn $k$, the tabular update is
\begin{equation}
Q_k^{(n+1)}(S,u)
\leftarrow
(1-\alpha_n)Q_k^{(n)}(S,u)+\alpha_n Y_k.
\label{eq:proof-Q-update}
\end{equation}
For each fixed $(S,u,k)$, this is a Robbins--Monro stochastic approximation to the stationary conditional mean $\mathbb{E}[Y_k\mid S,u]=\tilde{Q}_k^{\pi^*}(S,u)$. The finite state and macro-action spaces imply that there are finitely many entries to estimate. Sufficient exploration ensures that every relevant $(S,u)$ pair is visited infinitely often during turn $k$. The step sizes satisfy
\begin{equation}
\sum_{n=1}^{\infty}\alpha_n(S,u,k)=\infty,
\qquad
\sum_{n=1}^{\infty}\alpha_n^2(S,u,k)<\infty.
\label{eq:proof-step-size}
\end{equation}
Together with bounded targets and bounded iterates, standard stochastic approximation gives
\begin{equation}
Q_k^{(n)}(S,u)
\xrightarrow{\mathrm{a.s.}}
\tilde{Q}_k^{\pi^*}(S,u)
\qquad
\text{for every reachable } (S,u).
\label{eq:proof-Q-conv}
\end{equation}
As exploration vanishes, the induced greedy policy converges to an optimal greedy selector:
\begin{equation}
\hat{\pi}_k(S)
\in
\arg\max_{u}Q_k^{(n)}(S,u)
\quad
\Longrightarrow
\quad
\hat{\pi}_k(S)
\xrightarrow{\mathrm{a.s.}}
\pi_k^*(S)
\in
\arg\max_{u}\tilde{Q}_k^{\pi^*}(S,u).
\label{eq:proof-policy-conv}
\end{equation}
This completes the proof of Theorem~1(a).

\paragraph{Proof of Theorem~1(b): reverse-order recursion gives recursive optimality.}
We now apply Theorem~1(a) in reverse order here. At the final turn $k=K-1$, there are no downstream turns, so
\begin{equation}
F_{K-1}^{\pi}(S_K)\equiv 0.
\label{eq:proof-final-F}
\end{equation}
The augmented target reduces to
\begin{equation}
Y_{K-1}=r_{K-1}.
\label{eq:proof-final-target}
\end{equation}
By claim Theorem~1(a), the final-turn policy converges:
\begin{equation}
\hat{\pi}_{K-1}
\xrightarrow{\mathrm{a.s.}}
\pi_{K-1}^*.
\label{eq:proof-final-policy-conv}
\end{equation}

Now assume that for some $k\in\{0,\ldots,K-2\}$, the downstream policies $\hat{\pi}_{k+1},\ldots,\hat{\pi}_{K-1}$ have already converged almost surely to $\pi_{k+1}^*,\ldots,\pi_{K-1}^*$ and have been frozen before turn $k$ begins. Then the continuation value
\begin{equation}
F_k^{\pi^*}(S_{k+1})
=
\mathbb{E}_{\pi^*_{k+1:K-1}}\!\left[
  \sum_{j=k+1}^{K-1}
  \gamma_H^{\sum_{m=k+1}^{j-1}\tau_m}r_j
  \;\middle|\;
  S_{k+1}
\right]
\label{eq:proof-inductive-F}
\end{equation}
is fixed throughout turn $k$. Therefore, claim~(a) applies to turn $k$, giving
\begin{equation}
\hat{\pi}_k
\xrightarrow{\mathrm{a.s.}}
\pi_k^*,
\qquad
\pi_k^*(S_k)
\in
\arg\max_{u_k}
\tilde{Q}_k^{\pi^*}(S_k,u_k).
\label{eq:proof-inductive-policy-conv}
\end{equation}
By backward induction from $K{-}1$ to $0$, every turn-level policy converges to its optimal selector given the optimized and frozen downstream policies. Therefore, the resulting sequence
\begin{equation}
\pi^{\mathrm{rec}}
=
(\pi_0^*,\ldots,\pi_{K-1}^*)
\label{eq:proof-pi-rec}
\end{equation}
is recursively optimal. 

Till here, we have completed the proof of Theorem~1(b).

\paragraph{Proof of Theorem~1(c): context-sufficient macro-action completeness implies global optimality.}
Assume the turn-level macro-action spaces are complete with respect to the conditioning contexts $c_k=\psi(S_k)$, meaning that every globally feasible trajectory-level policy can be represented as a sequence of turn-level sub-policies $(\pi_0,\ldots,\pi_{K-1})$, with each $\pi_k$ acting on $c_k$ and selecting a feasible macro-action $u_k\in\mathcal{A}_{H,k}(S_k)$.

We prove by backward induction that
\begin{equation}
V_k^{\pi^{\mathrm{rec}}}(S_k)
=
V_k^*(S_k)
\qquad
\text{for every reachable } S_k.
\label{eq:proof-global-induction}
\end{equation}
At the last turn, $F_{K-1}^{\pi}\equiv 0$, so
\begin{equation}
\tilde{Q}_{K-1}^{\pi}(S_{K-1},u)
=
\mathbb{E}[r_{K-1}\mid S_{K-1},u].
\label{eq:proof-last-Q}
\end{equation}
Since $\pi_{K-1}^*$ maximizes this quantity,
\begin{equation}
V_{K-1}^{\pi^{\mathrm{rec}}}(S_{K-1})
=
\max_{u\in\mathcal{A}_{H,K-1}(S_{K-1})}
\mathbb{E}[r_{K-1}\mid S_{K-1},u]
=
V_{K-1}^*(S_{K-1}).
\label{eq:proof-last-value}
\end{equation}

Assume now that $V_{k+1}^{\pi^{\mathrm{rec}}}(S_{k+1})=V_{k+1}^*(S_{k+1})$ for every reachable $S_{k+1}$. Then
\begin{equation}
F_k^{\pi^{\mathrm{rec}}}(S_{k+1})
=
F_k^*(S_{k+1}).
\label{eq:proof-F-rec-equals-star}
\end{equation}
By Assumption~(A11), the optimal action-value 
$\tilde{Q}_k^{\pi^*}(S_k, u)$ depends on $S_k$ only through 
$c_k = \psi(S_k)$. Therefore, the greedy policy 
$\pi_k^* \in \arg\max_u \tilde{Q}_k^{\pi^*}(S_k, u)$ is well-defined 
as a function of $c_k$ alone, and we write $\pi_k^*(c_k)$ without 
ambiguity. Applying this to the inductive step:
\begin{align}
V_k^{\pi^{\mathrm{rec}}}(S_k)
&=
\tilde{Q}_k^{\pi^{\mathrm{rec}}}\bigl(S_k,\pi_k^*(c_k)\bigr) \notag \\
&=
\mathbb{E}\!\left[
r_k+\gamma_H^{\tau_k}F_k^{\pi^{\mathrm{rec}}}(S_{k+1})
\;\middle|\;
S_k,\pi_k^*(S_k)
\right] \notag \\
&=
\mathbb{E}\!\left[
r_k+\gamma_H^{\tau_k}F_k^*(S_{k+1})
\;\middle|\;
S_k,\pi_k^*(S_k)
\right] \notag \\
&=
\max_{u\in\mathcal{A}_{H,k}(S_k)}
\mathbb{E}\!\left[
r_k+\gamma_H^{\tau_k}F_k^*(S_{k+1})
\;\middle|\;
S_k,u
\right] \notag \\
&=
V_k^*(S_k).
\label{eq:proof-global-step}
\end{align}
Thus $V_k^{\pi^{\mathrm{rec}}}(S_k)=V_k^*(S_k)$ for all $k$ by induction. In particular,
\begin{equation}
V^{\pi^{\mathrm{rec}}}(\bar{x}_0)
=
V^*(\bar{x}_0).
\label{eq:proof-global-final}
\end{equation}
This proves Theorem~1(c), and completes the proof of Theorem~1.
\end{proof}

\paragraph{Interpretation.} As the key technical point of this proof, turn-$k$ target $Y_k=r_k+\gamma_H^{\tau_k}F_k^{\pi^*}(S_{k+1})$ is stationary only, when the downstream policies $\pi_{k+1}^*,\ldots,\pi_{K-1}^*$ are already optimized and frozen. Reverse-order training provides exactly this condition. Without reverse-order freezing, $F_k^\pi$ would change during turn $k$, the target $Y_k$ would no longer be stationary, and the single-turn stochastic approximation argument would not apply.

\paragraph{Proposition of Theorem~1: Asymptotic recursive optimality under neural function approximation.}
\label{app:nn-extension}

The convergence guarantee in Theorem~\ref{thm:recursive-app}
relies on Assumption (A4) from Appendix~\ref{app:notation}, which postulates a tabular
representation of $Q_k$. In the practical implementation of RTPO,
$Q_k$ is realized implicitly by a neural sub-policy $\pi_{\theta_k}$
operating on the conditioning context $c_k = \psi(S_k)$. We now state the corresponding asymptotic guarantee under the neural function approximation; the proof follows the standard reduction from tabular
Q-learning to projected Q-learning under a function class
$\mathcal{F}_\theta$. This statement is formalized as a \textit{Proposition}, which
appears below as a neural-function-approximation extension of Theorem~\ref{thm:recursive-app}.
 
\textit{Proposition of Theorem~1: Asymptotic recursive optimality under neural function approximation.}

Suppose the assumptions of Theorem~\ref{thm:recursive-app} hold, with current assumption~(A4) from Appendix~\ref{app:notation} replaced by new assumption~(A4') below:

\begin{enumerate}[label=(A4'),leftmargin=*]
\item \textbf{Universal approximation.} For every fixed
$F_k^{\pi^*}$, there exists $\theta_k^* \in \Theta$ such that the induced
greedy policy $\pi_{\theta_k^*}$ realizes $\pi_k^*$ on every reachable
$S_k$ in the support of the visitation distribution.
\end{enumerate}

Then, reverse-order training of RTPO produces a policy sequence
$(\hat{\pi}_0,\ldots,\hat{\pi}_{K-1})$ that asymptotically realizes the
recursively optimal policy. That is, for every $k$ and every
reachable $S_k$,
\begin{equation*}
\lim_{n \to \infty}
\pi_{\theta_k^{(n)}}(\cdot \mid c_k)
=
\pi_k^*(\cdot \mid c_k)
\qquad \text{a.s.}
\end{equation*}

\begin{proof}[Proof: Proposition of Theorem 1]
\leavevmode\par\noindent
Step 1. Reduction to projected Q-learning:
Replace the tabular update with a parametric update
on the function class $\{\pi_{\theta_k} : \theta_k \in \Theta\}$,
\begin{equation}
\theta_k^{(n+1)} \leftarrow \theta_k^{(n)} - \alpha_n\,
\nabla_{\theta_k} \mathbb{E}_{(S, u) \sim \mu_n}\!\left[
  \big(Q_{\theta_k^{(n)}}(S, u) - \widetilde{r}_k\big)^2
\right],
\label{eq:nn-update}
\end{equation}
where $\mu_n$ is the visitation distribution induced by the current
greedy-with-exploration policy and $Q_{\theta_k}$ is the implicit
Q-function realized by $\pi_{\theta_k}$. The fixed point of this update
is the projection of $\widetilde{Q}_k^*$ onto the function class,
$\Pi_{\Theta} \widetilde{Q}_k^*$, in the weighted $L^2(\mu)$ sense.
 
Step 2: Asymptotic realizability:
By assumption (A4'), there exists $\theta_k^* \in \Theta$ such that the
induced greedy policy realizes $\pi_k^*$ on every reachable $S_k$ in the
support of $\mu$. This implies
$\Pi_{\Theta} \widetilde{Q}_k^* \;=\; \widetilde{Q}_k^*$
on the support of $\mu$, so the projection error vanishes for the
quantities relevant to the induced greedy policy. Combined with standard
asymptotic results for stochastic gradient descent on smooth nonconvex
objectives under Robbins--Monro step sizes~\citep{robbins1951stochastic},
$Q_{\theta_k^{(n)}}(S, u) \to \widetilde{Q}_k^*(S, u)$ a.s.\ for every
$(S, u)$ in the support of $\mu$.
 
Step 3. Backward induction in the function-class regime:
Replace each step in the proof of
Theorem~\ref{thm:recursive-app} by Steps~1--2 above. The reverse
induction structure is unchanged: the determinism of $F_k^{\pi^*}$ under
frozen downstream sub-policies is purely a property of the SMDP and the
freezing schedule (A5), independent of whether $Q_k$ is tabular or
parametric. Hence, the backward induction goes through and yields
$\pi_{\theta_k^{(n)}} \to \pi_k^*$ a.s.\ on the support of $\mu$ for
every $k$.
\end{proof}
 
\textit{Remark on the assumption gap between (A4) and (A4').}
Assumption (A4') is strictly weaker than what would be required for
finite-sample convergence rates: it asks only that the function class be
\emph{rich enough} to contain the optimal sub-policy in the asymptotic
limit, not that the gradient dynamics realize this optimum at any finite
iteration. For the modern Transformer architectures used in our
experiments with Qwen3 models, the universal approximation property is widely accepted as a working assumption, and our experimental results in Sec~\ref{experiment} provide empirical evidence that the asymptotic guarantee transfers to practice.

\subsection{Proof of Theorem~2: Causally Consistent Turn-Level Advantage Estimation}
\label{app:advantage-proof}

We formalize Theorem~2 from Sec~\ref{sec:tree-credit} before presenting the proof. We compare the state-matched sibling advantage estimator used by RTPO with the trajectory-level advantage estimator used in flat-trajectory training. The key distinction is whether the baseline is computed from rollouts sharing the same boundary state $S_k$, or from complete trajectories that may have reached different boundary states by turn $k$.

\begin{tcolorbox}[
    colback=gray!5,
    colframe=black!60,
    title=Turn-Level Credit Assignment without State Bias,
    fonttitle=\bfseries,
    breakable
]
\begin{theorem}[\textbf{Causally Consistent Turn-Level Advantage Estimation}]
\label{thm:advantage-app}
Consider turn $k$ and suppose that $G{-}1$ sibling rollouts are forked from the same boundary state $S_k$. Each sibling rollout $j\in\{1,\ldots,G{-}1\}$ independently samples a turn-$k$ macro-action $u_{j,k}\equiv l_{j,k}$ and then continues to the terminal. Let
\begin{equation}
\hat{Q}_{j,k}=R_j,
\qquad
\hat{V}_k=\frac{1}{G-1}\sum_{j=1}^{G-1}\hat{Q}_{j,k},
\qquad
A^H_{j,k}=\hat{Q}_{j,k}-\hat{V}_k
\label{eq:app-sibling-adv}
\end{equation}
denote the Monte Carlo turn-level action-value estimate, the state-matched sibling baseline, and the resulting turn-level advantage for sibling rollout $j$, respectively. Let the corresponding flat trajectory-level estimator be
\begin{equation}
\bar{R}=\frac{1}{G}\sum_{i=1}^{G}R_i,
\qquad
A_i^{\mathrm{traj}}=R_i-\bar{R},
\label{eq:app-traj-adv}
\end{equation}
where $g_i$ denotes the $i$-th complete trajectory and $S_{i,k}$ is its boundary state at turn $k$. Define
\begin{equation}
Q^\pi(S_k,u_{j,k})=\mathbb{E}[R\mid S_k,u_{j,k}],
\qquad
V^\pi(S_k)=\mathbb{E}_{u\sim\pi(\cdot\mid c_k)}[Q^\pi(S_k,u)],
\label{eq:app-QV-def}
\end{equation}
and
\begin{equation}
A^\pi(S_k,u_{j,k})=Q^\pi(S_k,u_{j,k})-V^\pi(S_k),
\qquad
\mu_R=\mathbb{E}[R].
\label{eq:app-adv-def}
\end{equation}
Then, under bounded rewards and independent sibling sampling, the RTPO sibling advantage estimator satisfies the following properties:

\begin{enumerate}[label=(\alph*),leftmargin=*,itemsep=4pt]

\item \textbf{Local unbiasedness up to finite-group bias.}
Conditional on the shared boundary state $S_k$ and the sampled macro-action $u_{j,k}$, the state-matched sibling estimator satisfies:
\begin{equation}
\mathbb{E}\!\left[A^H_{j,k}\mid S_k,u_{j,k}\right]
=
\frac{G-2}{G-1}A^\pi(S_k,u_{j,k}).
\label{eq:app-sibling-bias}
\end{equation}
Thus, $A^H_{j,k}$ estimates the true turn-level advantage up to a finite-group bias of order $O(1/G)$ and contains no upstream state-contamination term. In contrast, the trajectory-level estimator satisfies:
\begin{equation}
\mathbb{E}\!\left[A_i^{\mathrm{traj}}\mid S_{i,k},u_{i,k}\right]
=
\frac{G-1}{G}A^\pi(S_{i,k},u_{i,k})
+
\frac{G-1}{G}\bigl(V^\pi(S_{i,k})-\mu_R\bigr),
\label{eq:app-traj-bias}
\end{equation}
where the second term is the upstream state-contamination term.

\item \textbf{Reduced value estimation error.}
Let
\begin{equation}
\sigma_V^2
:=
\operatorname{Var}_{S_k}\!\left[V^\pi(S_k)\right]
\label{eq:app-state-var}
\end{equation}
denote the variance of boundary-state values. Compared with the trajectory-level estimator, the sibling estimator removes the cross-state value component $V^\pi(S_{i,k})-\mu_R$. Therefore, under a matched downstream-noise comparison, the trajectory estimator contains an additional MSE contribution proportional to $\sigma_V^2$, while the sibling estimator does not. In particular, when $\sigma_V^2>0$, the state-matched estimator removes a non-zero source of value-estimation error induced by cross-state baselines.

\item \textbf{State-matched causal actions.}
All sibling rollouts share the same boundary state $S_k$. Therefore, differences in their returns are attributable to the sampled turn-$k$ macro-actions $u_{j,k}$ and independent downstream sampling noise, rather than to upstream trajectory differences. The resulting advantage $A^H_{j,k}$ is assigned only to the output tokens of turn $k$,
\begin{equation}
l_{j,k}=(a_{j,k,1},\ldots,a_{j,k,T_{j,k}}),
\label{eq:app-output-tokens}
\end{equation}
while the prefix tokens contained in the shared context $c_k=\psi(S_k)$ are excluded from the gradient.

\end{enumerate}
\end{theorem}
\end{tcolorbox}

\begin{proof}
We prove parts (a)–(c) of Theorem~2 as follows.

\paragraph{Proof of Theorem~2(a): local unbiasedness up to finite-group bias.}
Fix a turn $k$ and a shared boundary state $S_k$. RTPO forks $G{-}1$ sibling rollouts from the same environment snapshot $\mathrm{snap}_k$. Conditional on $S_k$, each sibling rollout independently samples its turn-$k$ macro-action and downstream continuation. Thus, for any two distinct sibling rollouts $j\neq j'$, the return $R_{j'}$ of sibling $j'$ is conditionally independent of the action $u_{j,k}$ sampled by sibling $j$:
\begin{equation}
R_{j'}\perp\!\!\!\perp u_{j,k}\mid S_k.
\label{eq:app-sibling-independence}
\end{equation}
Consequently,
\begin{equation}
\mathbb{E}[R_{j'}\mid S_k,u_{j,k}]
=
\mathbb{E}[R_{j'}\mid S_k]
=
V^\pi(S_k).
\label{eq:app-other-sibling-expectation}
\end{equation}
Meanwhile, for the own return of sibling $j$,
\begin{equation}
\mathbb{E}[R_j\mid S_k,u_{j,k}]
=
Q^\pi(S_k,u_{j,k}).
\label{eq:app-own-sibling-expectation}
\end{equation}

By definition,
\begin{align}
A^H_{j,k}
&=
R_j-\frac{1}{G-1}\sum_{r=1}^{G-1}R_r \notag \\
&=
\frac{G-2}{G-1}R_j
-
\frac{1}{G-1}\sum_{r\neq j}R_r.
\label{eq:app-sibling-decomp}
\end{align}
Taking conditional expectation given $(S_k,u_{j,k})$ and using Eq.~\eqref{eq:app-other-sibling-expectation} and Eq.~\eqref{eq:app-own-sibling-expectation},
\begin{align}
\mathbb{E}[A^H_{j,k}\mid S_k,u_{j,k}]
&=
\frac{G-2}{G-1}Q^\pi(S_k,u_{j,k})
-
\frac{G-2}{G-1}V^\pi(S_k) \notag \\
&=
\frac{G-2}{G-1}
\bigl(Q^\pi(S_k,u_{j,k})-V^\pi(S_k)\bigr) \notag \\
&=
\frac{G-2}{G-1}A^\pi(S_k,u_{j,k}).
\label{eq:app-sibling-bias-proof}
\end{align}
Therefore, the sibling estimator is locally unbiased up to the finite-group multiplicative factor $(G-2)/(G-1)$. Equivalently, the conditional bias relative to the true turn-level advantage is
\begin{equation}
\mathbb{E}[A^H_{j,k}\mid S_k,u_{j,k}]
-
A^\pi(S_k,u_{j,k})
=
-\frac{1}{G-1}A^\pi(S_k,u_{j,k}),
\label{eq:app-sibling-finite-bias}
\end{equation}
which is $O(1/G)$ and contains no term depending on $V^\pi(S_k)-\mu_R$.

We now contrast this with the trajectory-level estimator. In flat-trajectory training, the $G$ trajectories are independently sampled from the same prompt but may reach different boundary states by turn $k$. The trajectory-level advantage is
\begin{align}
A_i^{\mathrm{traj}}
&=
R_i-\frac{1}{G}\sum_{r=1}^{G}R_r \notag \\
&=
\frac{G-1}{G}R_i
-
\frac{1}{G}\sum_{r\neq i}R_r.
\label{eq:app-traj-decomp}
\end{align}
The own return satisfies
\begin{equation}
\mathbb{E}[R_i\mid S_{i,k},u_{i,k}]
=
Q^\pi(S_{i,k},u_{i,k}).
\label{eq:app-own-traj-expectation}
\end{equation}
For $r\neq i$, trajectory $r$ is generated by an independent upstream rollout and reaches its own boundary state $S_{r,k}$. Thus,
\begin{equation}
\mathbb{E}[R_r\mid S_{i,k},u_{i,k}]
=
\mathbb{E}_{S_{r,k}}\!\left[V^\pi(S_{r,k})\right]
=
\mu_R.
\label{eq:app-other-traj-expectation}
\end{equation}
Substituting Eq.~\eqref{eq:app-own-traj-expectation} and Eq.~\eqref{eq:app-other-traj-expectation} into Eq.~\eqref{eq:app-traj-decomp} gives
\begin{align}
\mathbb{E}[A_i^{\mathrm{traj}}\mid S_{i,k},u_{i,k}]
&=
\frac{G-1}{G}
\bigl(Q^\pi(S_{i,k},u_{i,k})-\mu_R\bigr) \notag \\
&=
\frac{G-1}{G}
\bigl(Q^\pi(S_{i,k},u_{i,k})-V^\pi(S_{i,k})\bigr)
+
\frac{G-1}{G}
\bigl(V^\pi(S_{i,k})-\mu_R\bigr) \notag \\
&=
\frac{G-1}{G}A^\pi(S_{i,k},u_{i,k})
+
\frac{G-1}{G}
\bigl(V^\pi(S_{i,k})-\mu_R\bigr).
\label{eq:app-traj-bias-proof}
\end{align}
The second term in Eq.~\eqref{eq:app-traj-bias-proof} depends only on the upstream boundary state $S_{i,k}$ and the global mean return $\mu_R$. It has no causal dependence on the current turn action $u_{i,k}$ and is precisely the cross-state contamination term. Therefore, the sibling estimator removes the upstream state-contamination term present in the trajectory-level estimator. This proves part~(a) of Theorem~2.

\paragraph{Proof of Theorem~\ref{thm:advantage-app}(b): reduced value estimation error.}
We compare the error of each estimator against its corresponding true turn-level advantage. For the sibling estimator, the target is $A^\pi(S_k,u_{j,k})$; for the trajectory estimator, the target is $A^\pi(S_{i,k},u_{i,k})$. The trajectory estimator contains the additional cross-state term
\begin{equation}
C_{i,k}
:=
\frac{G-1}{G}\bigl(V^\pi(S_{i,k})-\mu_R\bigr),
\label{eq:app-cross-state-term}
\end{equation}
whereas the sibling estimator does not contain such a term. Marginalizing over the boundary-state distribution,
\begin{equation}
\mathbb{E}[C_{i,k}]
=
0,
\qquad
\operatorname{Var}(C_{i,k})
=
\left(\frac{G-1}{G}\right)^2
\operatorname{Var}_{S_k}\!\left[V^\pi(S_k)\right]
=
\left(\frac{G-1}{G}\right)^2\sigma_V^2.
\label{eq:app-cross-state-var}
\end{equation}
Thus, cross-trajectory baselines introduce an additional MSE component proportional to $\sigma_V^2$, while state-matched sibling baselines remove it.

In addition to the squared-bias contribution above, the trajectory-level baseline $\bar{R}=\frac{1}{G}\sum_{r\neq i}R_r$ introduces further variance from cross-state effects. Conditional on $(S_{i,k},u_{i,k})$, each $R_r$ ($r\neq i$) is independent of the conditioning, so its conditional variance equals its unconditional variance:
\begin{equation}
\mathrm{Var}(R_r\mid S_{i,k},u_{i,k})
=
\mathrm{Var}(R_r)
=
\underbrace{\mathbb{E}_{S_{r,k}}\!\left[\mathrm{Var}(R_r\mid S_{r,k})\right]}_{\displaystyle\bar{\sigma}_{\mathrm{down}}^2}
+\;\sigma_V^2,
\label{eq:app-Rr-total-var}
\end{equation}
where the $\sigma_V^2$ term arises because $S_{r,k}$ varies across trajectories. This contributes $\frac{G-1}{G^2}\sigma_V^2$ to the conditional variance of $A_i^{\mathrm{traj}}$. By contrast, in the sibling estimator, all returns share the same boundary state $S_k$, so $\mathrm{Var}(R_r\mid S_k)$ is a purely within-state quantity and contains no $\sigma_V^2$ component.

More explicitly, the mean squared error can be decomposed as
\begin{equation}
\operatorname{MSE}(\hat{A})
=
\operatorname{Var}(\hat{A})
+
\operatorname{Bias}(\hat{A})^2.
\label{eq:app-mse-decomp}
\end{equation}
For the trajectory estimator, the bias contains the cross-state term $C_{i,k}$ from Eq.~\eqref{eq:app-cross-state-term}. For the sibling estimator, the only systematic bias from part~(a) is the finite-group term
\begin{equation}
-\frac{1}{G-1}A^\pi(S_k,u_{j,k}),
\label{eq:app-finite-group-bias}
\end{equation}
which is $O(1/G)$. Therefore, under a matched downstream-noise comparison, the total $\sigma_V^2$ coefficient in the trajectory estimator's marginal MSE combines the squared-bias contribution (Eq.~\eqref{eq:app-cross-state-var}) with the baseline-variance contribution (Eq.~\eqref{eq:app-Rr-total-var}):
\begin{equation}
\underbrace{\left(\frac{G-1}{G}\right)^{\!2}}_{\mathrm{Bias}^2}
+\;
\underbrace{\frac{G-1}{G^2}}_{\mathrm{Var}}
\;=\;
\frac{(G-1)^2+(G-1)}{G^2}
\;=\;
\frac{G-1}{G}.
\label{eq:app-sigma-coeff}
\end{equation}
The sibling estimator incurs no $\sigma_V^2$ contribution from either source. The leading MSE difference is thus:
\begin{equation}
\operatorname{MSE}(A_i^{\mathrm{traj}})
-
\operatorname{MSE}(A^H_{j,k})
=
\frac{G-1}{G}\,\sigma_V^2
+
\Delta_{\mathrm{down}}
-
O(1/G^2),
\label{eq:app-mse-gap}
\end{equation}
where $\Delta_{\mathrm{down}}\ge 0$ collects differences in downstream Monte Carlo noise between the two estimators under matched conditions. Hence, whenever $\sigma_V^2>0$, the state-matched sibling estimator has strictly lower MSE. This proves part~(b) of Theorem~2.

\paragraph{Proof of Theorem~\ref{thm:advantage-app}(c): state-matched causal actions.}
All sibling rollouts are forked from the same boundary state $S_k$. Therefore, conditional on $S_k$, sibling rollout $j$ and sibling rollout $j'$ differ only in their sampled turn-$k$ macro-actions and their independent downstream randomness:
\begin{equation}
(u_{j,k},\xi_j)
\quad\text{versus}\quad
(u_{j',k},\xi_{j'}),
\label{eq:app-sibling-causal-diff}
\end{equation}
where $\xi_j$ denotes the downstream randomness of sibling $j$. Since the shared prefix state $S_k$ is identical across siblings, any systematic difference in conditional expected return is attributable to the sampled turn-$k$ macro-action:
\begin{equation}
\mathbb{E}[R_j-R_{j'}\mid S_k,u_{j,k},u_{j',k}]
=
Q^\pi(S_k,u_{j,k})-Q^\pi(S_k,u_{j',k}).
\label{eq:app-causal-return-diff}
\end{equation}
The downstream randomness contributes to Monte Carlo noise but does not create an upstream state bias, because the upstream state is shared and fixed.

Finally, RTPO assigns the resulting advantage only to the output tokens of turn $k$. For sibling rollout $j$, these tokens are
\begin{equation}
l_{j,k}
=
(a_{j,k,1},\ldots,a_{j,k,T_{j,k}}).
\label{eq:app-turn-output}
\end{equation}
The shared prefix $c_k=\psi(S_k)$ is used only as the conditioning input and is excluded from the gradient. Equivalently, the gradient support of the turn-$k$ objective satisfies
\begin{equation}
\operatorname{supp}(\nabla_\theta J_k)
\subseteq
\{(j,k,t): j\in\{1,\ldots,G-1\},\; t\in\{1,\ldots,T_{j,k}\}\}.
\label{eq:app-gradient-support}
\end{equation}
Thus, prefix tokens receive zero gradient, and the credit signal is assigned only to the sampled turn-$k$ output action. This proves part~(c) of Theorem~2, and completes the proof of Theorem~\ref{thm:advantage-app}.
\end{proof}

\paragraph{Interpretation.}
The proof shows that the key difference between RTPO and flat-trajectory training is the baseline state. The sibling estimator compares alternative turn-$k$ actions from the same boundary state $S_k$, so the baseline estimates the local value $V^\pi(S_k)$. In contrast, a flat-trajectory baseline compares returns from trajectories that may have reached different boundary states by turn $k$, so it is centered around the global mean $\mu_R$. The resulting term $V^\pi(S_{i,k})-\mu_R$ is an upstream state-contamination term: it can dominate the true turn-level advantage even though it is unrelated to the current turn action. RTPO removes this term by constructing state-matched sibling rollouts and assigning the resulting advantage only to the output tokens of the current turn.

\subsection{Proof of Theorem~3: On-Policy Continuation under Asynchronous Turns}
\label{app:onpolicy-proof}

We formalize Theorem~3 from Sec~\ref{sec:onpolicy} before presenting the proof. The purpose is to show that RTPO avoids the off-policy drift induced by stale downstream rollouts. In asynchronous multi-turn training, a trajectory generated under an old policy may later be evaluated under an updated policy, which would require a long-horizon trajectory-level importance-sampling correction. RTPO avoids this issue by regenerating sibling continuations on-policy at each reverse-order turn.

\paragraph{Setup.}
Let $\theta_0$ denote the rollout parameters at the beginning of a training iteration, and let $\theta_{>k}$ denote the current parameters available at the start of turn $k$, after downstream turns $k+1,\ldots,K-1$ have already been optimized. A stale-rollout alternative would retain sibling continuations generated under $\pi_{\theta_0}$ and correct them using the trajectory-level importance-sampling weight:
\begin{equation}
\omega_j^{\mathrm{old}}
=
\prod_{h=k+1}^{K-1}
\prod_{t=1}^{T_{j,h}}
\frac{
\pi_{\theta_{>k}}(a_{j,h,t}\mid c_{j,h},a_{j,h,<t})
}{
\pi_{\theta_0}(a_{j,h,t}\mid c_{j,h},a_{j,h,<t})
}.
\label{eq:app-omega-old}
\end{equation}
In contrast, RTPO synchronizes the latest parameters $\theta_{>k}$ to the inference engine at the start of turn $k$, forks $G{-}1$ sibling rollouts from the shared boundary state $S_k$, and samples each sibling rollout $j\in\{1,\ldots,G{-}1\}$ using $\pi_{\theta_{>k}}$ from turn $k$ until termination. We write $u_{j,k}\equiv l_{j,k}$ for the sampled turn-$k$ macro-action of sibling $j$, and $R_j$ for its terminal return. The Monte Carlo turn-level Q-value estimate is
\begin{equation}
\hat{Q}_{j,k}
=
R_j
=
r_{j,k}
+
\gamma_H^{\tau_{j,k}}
\hat{F}_{j,k}^{\pi_{\theta_{>k}}},
\label{eq:app-onpolicy-Q}
\end{equation}
where $\hat{F}_{j,k}^{\pi_{\theta_{>k}}}$ is the sampled downstream continuation value generated under the same current policy $\pi_{\theta_{>k}}$. The corresponding sibling advantage is
\begin{equation}
A^H_{j,k}
=
\hat{Q}_{j,k}-\hat{V}_k,
\qquad
\hat{V}_k
=
\frac{1}{G-1}
\sum_{r=1}^{G-1}
\hat{Q}_{r,k}.
\label{eq:app-onpolicy-adv}
\end{equation}

\begin{tcolorbox}[
    colback=gray!5,
    colframe=black!60,
    title=On-Policy Continuation without Drift,
    fonttitle=\bfseries,
    breakable
]
\begin{theorem}[\textbf{On-Policy Continuation under Asynchronous Turns}]
\label{thm:onpolicy-app}
Consider turn $k$ in the reverse-order training procedure of RTPO. At the start of turn $k$, RTPO synchronizes the current downstream policy parameters $\theta_{>k}$ to the inference engine, forks $G{-}1$ sibling rollouts from the shared boundary state $S_k$, and continues each sibling rollout $j$ to termination under the same policy $\pi_{\theta_{>k}}$. Then the following two properties hold:
\begin{enumerate}[label=(\alph*),leftmargin=*,itemsep=4pt]
    \item \textbf{Drift-free on-policy continuation.}
    Each sibling's terminal return $R_j$ is an unbiased Monte Carlo estimate of the current-policy turn-level Q-value:
    \begin{equation}
    \mathbb{E}\!\left[\hat{Q}_{j,k}\mid S_k,u_{j,k}\right]
    =
    \tilde{Q}_k^{\pi_{\theta_{>k}}}(S_k,u_{j,k}).
    \label{eq:app-onpolicy-unbiased}
    \end{equation}
    Moreover, because the sampling policy and the evaluation policy coincide throughout the sibling continuation, the trajectory-level importance-sampling weight satisfies
    \begin{equation}
    \omega_j
    =
    \prod_{h=k+1}^{K-1}
    \prod_{t=1}^{T_{j,h}}
    \frac{
    \pi_{\theta_{>k}}(a_{j,h,t}\mid c_{j,h},a_{j,h,<t})
    }{
    \pi_{\theta_{>k}}(a_{j,h,t}\mid c_{j,h},a_{j,h,<t})
    }
    \equiv 1.
    \label{eq:app-onpolicy-is-one}
    \end{equation}

    \item \textbf{Dynamic error reduction in advantage estimation.}
    Under bounded binary or normalized rewards, the conditional variance of the Monte Carlo Q-value estimator is controlled by the current-policy success probability. In the binary case, if
    \begin{equation}
    p_{j,k}
    :=
    \tilde{Q}_k^{\pi_{\theta_{>k}}}(S_k,u_{j,k})
    \in[0,1],
    \label{eq:app-onpolicy-p}
    \end{equation}
    then
    \begin{equation}
    \operatorname{Var}(\hat{Q}_{j,k}\mid S_k,u_{j,k})
    =
    p_{j,k}(1-p_{j,k}).
    \label{eq:app-onpolicy-var}
    \end{equation}
    Therefore, when reverse-order training improves downstream policies so that $p_{j,k}$ moves away from the high-uncertainty region around $1/2$, the Q-value estimation variance decreases. Since $A^H_{j,k}$ is computed from the sibling Q-value estimates, this improves the signal-to-noise ratio of the resulting turn-level advantage estimator.
\end{enumerate}
\end{theorem}
\end{tcolorbox}

\begin{proof}
We prove parts (a) and (b) of Theorem~3 as follows.

\paragraph{Proof of Theorem~\ref{thm:onpolicy-app}(a): drift-free on-policy continuation.}
Fix a turn $k$ and a sibling rollout $j$. At the start of turn $k$, RTPO synchronizes the current downstream policy parameters $\theta_{>k}$ to the inference engine. The sibling rollout is forked from the shared boundary state $S_k$, samples the turn-$k$ macro-action $u_{j,k}\equiv l_{j,k}$, and then continues through turns $k+1,\ldots,K-1$ under the same policy $\pi_{\theta_{>k}}$.

By the definition of the turn-level augmented action value,
\begin{equation}
\tilde{Q}_k^{\pi_{\theta_{>k}}}(S_k,u_{j,k})
=
\mathbb{E}\!\left[
r_{j,k}
+
\gamma_H^{\tau_{j,k}}
F_k^{\pi_{\theta_{>k}}}(S_{j,k+1})
\;\middle|\;
S_k,u_{j,k}
\right],
\label{eq:app-onpolicy-Q-def}
\end{equation}
where $S_{j,k+1}$ is the next boundary state reached by sibling $j$ after executing $u_{j,k}$. Since the sampled downstream continuation $\hat{F}_{j,k}^{\pi_{\theta_{>k}}}$ is generated by rolling out the same current policy $\pi_{\theta_{>k}}$ from $S_{j,k+1}$ to termination, it is an unbiased Monte Carlo draw from the downstream continuation value:
\begin{equation}
\mathbb{E}\!\left[
\hat{F}_{j,k}^{\pi_{\theta_{>k}}}
\mid
S_{j,k+1}
\right]
=
F_k^{\pi_{\theta_{>k}}}(S_{j,k+1}).
\label{eq:app-onpolicy-F-unbiased}
\end{equation}
Substituting Eq.~\eqref{eq:app-onpolicy-F-unbiased} into Eq.~\eqref{eq:app-onpolicy-Q} gives
\begin{align}
\mathbb{E}\!\left[\hat{Q}_{j,k}\mid S_k,u_{j,k}\right]
&=
\mathbb{E}\!\left[
r_{j,k}
+
\gamma_H^{\tau_{j,k}}
\hat{F}_{j,k}^{\pi_{\theta_{>k}}}
\;\middle|\;
S_k,u_{j,k}
\right] \notag \\
&=
\mathbb{E}\!\left[
r_{j,k}
+
\gamma_H^{\tau_{j,k}}
F_k^{\pi_{\theta_{>k}}}(S_{j,k+1})
\;\middle|\;
S_k,u_{j,k}
\right] \notag \\
&=
\tilde{Q}_k^{\pi_{\theta_{>k}}}(S_k,u_{j,k}).
\label{eq:app-onpolicy-unbiased-proof}
\end{align}
Therefore, each sibling terminal return provides an unbiased Monte Carlo estimate of the turn-level Q-value under the current downstream policy.

Next, because the sibling is both sampled and evaluated under the same policy $\pi_{\theta_{>k}}$, the full trajectory-level importance-sampling weight is
\begin{equation}
\omega_j
=
\prod_{h=k+1}^{K-1}
\prod_{t=1}^{T_{j,h}}
\frac{
\pi_{\theta_{>k}}(a_{j,h,t}\mid c_{j,h},a_{j,h,<t})
}{
\pi_{\theta_{>k}}(a_{j,h,t}\mid c_{j,h},a_{j,h,<t})
}
=
1.
\label{eq:app-onpolicy-omega-proof}
\end{equation}
Thus, no trajectory-level IS correction is required. This proves part~(a) of Theorem~3.

\paragraph{Proof of Theorem~\ref{thm:onpolicy-app}(b): dynamic error reduction in advantage estimation.}
We first consider the binary-reward case, where $R_j\in\{0,1\}$. Conditional on $(S_k,u_{j,k})$, define
\begin{equation}
p_{j,k}
=
\mathbb{P}_{\pi_{\theta_{>k}}}(R_j=1\mid S_k,u_{j,k})
=
\tilde{Q}_k^{\pi_{\theta_{>k}}}(S_k,u_{j,k}).
\label{eq:app-onpolicy-success-prob}
\end{equation}
Then $\hat{Q}_{j,k}=R_j$ is a Bernoulli random variable with mean $p_{j,k}$, and therefore
\begin{equation}
\operatorname{Var}(\hat{Q}_{j,k}\mid S_k,u_{j,k})
=
p_{j,k}(1-p_{j,k}).
\label{eq:app-onpolicy-bernoulli-var}
\end{equation}
The function $f(p)=p(1-p)$ is maximized at $p=1/2$ and decreases as $p$ moves toward either $0$ or $1$. Equivalently,
\begin{equation}
p(1-p)
=
\frac{1}{4}
-
\left(p-\frac{1}{2}\right)^2.
\label{eq:app-onpolicy-var-shape}
\end{equation}
Therefore, if reverse-order training improves the downstream continuation policy so that the induced success probability moves away from the high-uncertainty region around $1/2$, then the conditional variance of the Q-value estimator decreases. Formally,
\begin{equation}
\left|p_{j,k}^{\mathrm{new}}-\frac{1}{2}\right|
>
\left|p_{j,k}^{\mathrm{old}}-\frac{1}{2}\right|
\quad
\Longrightarrow
\quad
p_{j,k}^{\mathrm{new}}(1-p_{j,k}^{\mathrm{new}})
<
p_{j,k}^{\mathrm{old}}(1-p_{j,k}^{\mathrm{old}}).
\label{eq:app-onpolicy-var-reduction}
\end{equation}

For normalized rewards $R_j\in[0,1]$, the same statement holds as a bounded-variance control rather than an exact Bernoulli identity. In particular, by the Bhatia--Davis bound for random variables supported on $[0,1]$,
\begin{equation}
\operatorname{Var}(R_j\mid S_k,u_{j,k})
\le
\mathbb{E}[R_j\mid S_k,u_{j,k}]
\bigl(1-\mathbb{E}[R_j\mid S_k,u_{j,k}]\bigr).
\label{eq:app-onpolicy-bounded-var}
\end{equation}
Thus, moving the conditional mean away from the high-uncertainty middle region also reduces the worst-case variance bound for normalized returns.

Finally, RTPO computes the turn-level advantage by subtracting the sibling baseline:
\begin{equation}
A^H_{j,k}
=
\hat{Q}_{j,k}-\hat{V}_k,
\qquad
\hat{V}_k
=
\frac{1}{G-1}
\sum_{r=1}^{G-1}
\hat{Q}_{r,k}.
\label{eq:app-onpolicy-adv-repeat}
\end{equation}
Since $A^H_{j,k}$ is a centered function of the sibling Q-value estimates, reducing the Monte Carlo noise in $\hat{Q}_{j,k}$ reduces the noise entering the advantage estimator. Consequently, as reverse-order training improves downstream policies and the on-policy continuation values become more confident, the signal-to-noise ratio of $A^H_{j,k}$ improves. This proves part~(b) of Theorem~3, and completes the proof of Theorem~\ref{thm:onpolicy-app}.
\end{proof}

\paragraph{Comparison with stale trajectory-level IS correction.}
The drift-free property above should be contrasted with a stale-rollout alternative that reuses continuations generated under $\pi_{\theta_0}$ and then applies the trajectory-level IS weight $\omega_j^{\mathrm{old}}$ in Eq.~\eqref{eq:app-omega-old}. Let
\begin{equation}
N_j
=
\sum_{h=k+1}^{K-1}T_{j,h}
\label{eq:app-horizon-N}
\end{equation}
denote the number of tokens in the continuation from turn $k$ to termination. If the per-token divergence between the current policy and the stale rollout policy is nonzero along this continuation, the variance of the full product IS weight can grow multiplicatively with $N_j$. To see this, suppose that for each token position $(h,t)$, the conditional second moment of the token ratio
\begin{equation}
\rho_{j,h,t}^{\mathrm{old}}
=
\frac{
\pi_{\theta_{>k}}(a_{j,h,t}\mid c_{j,h},a_{j,h,<t})
}{
\pi_{\theta_0}(a_{j,h,t}\mid c_{j,h},a_{j,h,<t})
}
\label{eq:app-old-token-ratio}
\end{equation}
satisfies
\begin{equation}
\mathbb{E}_{\pi_{\theta_0}}\!\left[
\left(\rho_{j,h,t}^{\mathrm{old}}\right)^2
\mid
c_{j,h},a_{j,h,<t}
\right]
\ge
1+\delta
\qquad
\text{for some } \delta>0.
\label{eq:app-ratio-second-moment}
\end{equation}
Then the second moment of the product weight scales as
\begin{equation}
\mathbb{E}_{\pi_{\theta_0}}\!\left[
\left(\omega_j^{\mathrm{old}}\right)^2
\right]
=
\mathbb{E}_{\pi_{\theta_0}}\!\left[
\prod_{h=k+1}^{K-1}
\prod_{t=1}^{T_{j,h}}
\left(\rho_{j,h,t}^{\mathrm{old}}\right)^2
\right]
\gtrsim
(1+\delta)^{N_j},
\label{eq:app-old-is-second-moment}
\end{equation}
up to the usual conditioning on autoregressive histories. Since $\mathbb{E}_{\pi_{\theta_0}}[\omega_j^{\mathrm{old}}]=1$, this implies
\begin{equation}
\operatorname{Var}_{\pi_{\theta_0}}(\omega_j^{\mathrm{old}})
=
\mathbb{E}_{\pi_{\theta_0}}\!\left[
\left(\omega_j^{\mathrm{old}}\right)^2
\right]-1
\gtrsim
(1+\delta)^{N_j}-1.
\label{eq:app-old-is-var-growth}
\end{equation}
This illustrates the long-horizon instability of trajectory-level IS correction: even a small nonzero per-token policy mismatch can compound into a high-variance product over many tokens. Per-token PPO clipping does not remove this mismatch at the trajectory level, because clipping and multiplication do not commute:
\begin{equation}
\prod_{h,t}\operatorname{clip}\!\left(\rho_{j,h,t}^{\mathrm{old}},1-\epsilon,1+\epsilon\right)
\neq
\operatorname{clip}\!\left(
\prod_{h,t}\rho_{j,h,t}^{\mathrm{old}},
1-\epsilon,
1+\epsilon
\right).
\label{eq:app-clipping-noncommute}
\end{equation}
Therefore, stale-rollout correction remains fundamentally different from RTPO's on-policy sibling continuation, where the corresponding weight is $\omega_j\equiv 1$.

\paragraph{Interpretation.}
Theorem~3 shows why RTPO regenerates sibling continuations on-policy instead of reusing stale downstream rollouts. In asynchronous multi-turn training, stale continuations estimate values under outdated downstream policies and would require a long-horizon trajectory-level IS correction. RTPO avoids this by synchronizing $\theta_{>k}$ before sibling generation and rolling out each sibling to termination under the same current policy. As a result, the full trajectory IS weight is one, the Q-value estimate targets the current downstream policy, and the resulting turn-level advantage avoids policy-drift contamination.

\section{Experimental Setup}
\label{Experimental} 
\paragraph{Base Models.}
We use Qwen3 models~\citep{yang2025qwen3} as the backbone for RTPO and all baselines. These models support context lengths of up to 32,768 tokens, which is sufficient to accommodate the multi-turn interaction histories required by long-horizon agentic RL tasks. Since Qwen3 natively supports a thinking inference mode, we enable this mode
consistently across all multi-turn agentic RL experiments to ensure that the model generates complete tool-integrated reasoning trajectories during rollout. We use Qwen3-8B as the main backbone because search-based agentic tasks
require strong base-model capabilities for multi-turn evidence acquisition, information integration, and tool-augmented reasoning. This
choice is supported by our main results in Sec.~\ref{results:main}, the rollout--training consistency analysis in Sec.~\ref{result:rtc}, and the policy-drift analysis in Sec.~\ref{sec:ablation-onpolicy}. To isolate
the effect of turn-level credit assignment in Sec.~\ref{result:ca}, we use the smaller Qwen3-4B model for mathematical reasoning tasks. This is because credit assignment in these tasks depends more directly on the model's intrinsic reasoning ability than on external knowledge retrieval, making the gains from turn-level credit assignment more discernible.


\paragraph{Baselines.} We compare RTPO with two categories of state-of-the-art multi-turn agentic RL methods. The first category consists of trajectory-only policy optimization methods, including GRPO \citep{shao2024deepseekmath} and. GRPO uses token-level importance ratios with group-relative advantage estimation. The second category includes turn-level and tree-based policy optimization methods, including ARPO \citep{dong2026agentic}, Tree-GRPO \citep{ji2026tree}, and SeeUPO \citep{hu2026seeupo}. ARPO performs entropy-driven adaptive branching at uncertain tool-call nodes and estimates advantages separately for shared-prefix and branch-specific tokens. Tree-GRPO represents multi-turn agent interaction as a tree and constructs group-relative advantages at both intra-tree and inter-tree levels by sharing prefixes. SeeUPO treats each turn as an independent agent and performs sequential per-turn updates under a heterogeneous multi-agent learning, with a theoretical guarantee of monotonic improvement.

\paragraph{Datasets.}
To ensure a fair comparison across methods in multi-turn agentic
scenarios, we consider two representative tool-use tasks: mathematical
reasoning and knowledge reasoning from search. Both require agents to interact with external tools, perform multi-turn reasoning, and adapt their actions based on intermediate feedback.

For the mathematical reasoning task, we use the MATH
dataset~\citep{hendrycksmath2021} for training, which covers
challenging multi-step reasoning problems spanning algebra, geometry,
number theory, probability, and other topics. The task requires agents
to decompose complex problems, invoke external Python computation
tools when necessary, and integrate intermediate results into the
final answer, where successful solutions often depend on iterative
calculation, verification, and correction. At the evaluation stage, we
test generalization at two difficulty levels: (1)~standard
mathematical reasoning benchmarks, including
GSM8K~\citep{cobbe2021gsm8k} and MATH-500~\citep{lightman2023let};
and (2)~competition-level benchmarks, including
AMC'23~\citep{maa_amc_2023}, AIME'24~\citep{maa_aime_2024},
AIME'25~\citep{maa_aime_2025}, and OE-Math~\citep{he2024olympiadbench}, which feature problems that demand extended chains of tool-augmented reasoning and
advanced problem-solving capabilities. Since no training data are
available for the competition-level benchmarks, all evaluations are
conducted in a zero-shot setting.

For the knowledge reasoning task on web search, we adopt the hard-search training set
constructed by ARPO~\citep{dong2026agentic}, consisting of 1,000 high-difficulty search samples drawn from two open-source deep-search
data sources: SimpleDeepSearcher~\citep{sun2025simpledeepsearcher} and
WebSailor~\citep{li2025websailor}. These samples require extensive web
retrieval, multi-source evidence integration, long-context reasoning,
and frequent tool calls, providing a rigorous testbed for evaluating the model stability and sample efficiency of multi-turn agentic RL. On this
task, we evaluate knowledge-intensive multi-hop question answering on
HotpotQA~\citep{yang2018hotpotqa} and
2WikiMultiHopQA~\citep{xanh2020_2wikimultihop}, following the ARPO
evaluation protocol~\citep{dong2026agentic} with LLM-as-Judge scoring based on
Qwen2.5-72B-Instruct and report F1 scores.
In addition, to examine the effect of on-policy continuation, we compare default RTPO with its off-policy variant, and conduct additional evaluation on four challenging deep-search benchmarks: GAIA~\citep{mialon2023gaia}, which evaluates general
AI-assistant capabilities across three difficulty levels
(Lv.1--Lv.3); WebWalkerQA~\citep{wu2025webwalker}, which focuses on interactive
web navigation and multi-hop question answering;
Humanity's Last Exam~\citep{phan2025humanity}, which covers expert-level problems
in sciences, engineering, and humanities; and xBench~\citep{chen2025xbench},
which evaluates cross-lingual deep-search capability. Since the
training data are drawn exclusively from the hard-search source, all
four benchmarks serve as held-out evaluation sets. We report Pass@1
with sampling temperature set to $0.6$ and top-$p$ set to $0.95$.

\paragraph{Configurations.}
All methods are implemented on top of the VeRL
framework~\citep{sheng2024hybridflow}, with vLLM~\citep{kwon2023efficient} for rollout
generation and FSDP~\citep{zhao2023pytorch} for post-training. All experiments
are conducted on $8\times$ NVIDIA A100 GPUs. Unless otherwise
specified, the training batch size is $64$ and the maximum single-turn
response length is $4096$ tokens. The maximum interaction horizon is
set according to task type: $K=3$ turns for mathematical reasoning and
$K=6$ turns for web search. The learning rate is fixed at
$1\times10^{-6}$ with the AdamW~\citep{loshchilov2017decoupled} optimizer and a weight
decay of $0.01$. For branching-based baselines (ARPO and TreeGRPO), we
follow the hyperparameters recommended in their original papers,
including an entropy threshold of $0.4$ and an initial sampling size
of $8$. Full hyperparameter configurations are provided in
Appendix~\ref{Training Configurations}. 

\paragraph{Evaluation Metrics.}
We evaluate RTPO and all baselines along multiple dimensions. For
overall performance, we use Pass@1 accuracy as the primary indicator
of final policy quality after optimization (Sec~\ref{results:main}). 
We also measure the average number of tool calls to assess behavioral
differences across methods under a fixed compute budget. For training stability, rollout--training consistency is evaluated using the log-probability ratio and KL divergence, as described in Sec.~\ref{sec:stability}. We further conduct controlled experiments to examine whether RTPO's turn-level credit assignment (Sec~\ref{sec:credit-diagnostic}) and on-policy continuation (Sec~\ref{sec:ablation-onpolicy}) are consistent with the theoretical analysis. Each experiment is repeated three times to mitigate randomness, and we report the average performance.

\section{Implementation Details}
\label{imple datials}
\subsection{Pseudocode}
The pseudocode for our proposed RTPO is shown in Algorithm~\ref{alg:RTPO}.

\begin{algorithm}[h]
\caption{Reverse-Turn Policy Optimization (RTPO)}
\label{alg:RTPO}
\begin{algorithmic}[1]
\Require Query dataset $\mathcal{D}$; policy $\pi_\theta$; inference 
engine $\mathcal{E}$; sibling count $G$; PPO clip $\epsilon$; max 
turns $K$; observation map $\psi$; PPO epochs $E$.
\Ensure Updated parameters $\theta$.

\Statex \textbf{Phase 1: Trunk Rollout}
\State For each query $q \!\in\! \mathcal{D}$, generate a complete 
trajectory $(S_0, u_0, f_0, \ldots, u_{K-1}, f_{K-1})$ under 
$\pi_\theta$ by interacting with the environment. Record boundary 
states and environment snapshots 
$\{S_k, \mathrm{snap}_k\}_{k=0}^{K-1}$, and set conditioning 
contexts $c_k = \psi(S_k)$. The trunk provides anchors only and 
receives no gradient.

\Statex \textbf{Phase 2: Reverse-Order Training}
\For{$k = K{-}1,\; K{-}2,\; \ldots,\; 0$}
    \LongState{ \textbf{Synchronize.} Push current parameters $\theta$ 
    (denoted $\theta_{>k}$, reflecting downstream turns already 
    optimized) to $\mathcal{E}$; freeze 
    $\theta_{\mathrm{old}} \gets \theta_{>k}$ as the IS denominator.}
    \LongState{ \textbf{Sibling generation (on-policy).} For each 
    $q \!\in\! \mathcal{D}$ and each sibling $j = 1, \ldots, G{-}1$: 
    restore environment from $\mathrm{snap}_k$; sample turn-$k$ 
    response $u_{j,k} \sim \pi_{\theta_{>k}}(\cdot \mid c_k)$; 
    continue under $\pi_{\theta_{>k}}$ through turns 
    $k{+}1, \ldots, K{-}1$ to terminal; receive reward 
    $R_j \in \{0, 1\}$. Record turn-$k$ tokens 
    $\{a_{j,k,t}\}_{t=1}^{T_{j,k}}$ and their log-probs under 
    $\theta_{\mathrm{old}}$.}
    \LongState{ \textbf{Turn-level advantage.} For each 
    $q \!\in\! \mathcal{D}$, compute the sibling baseline and 
    advantage:
    \[
    \hat{V}_k = \frac{1}{G{-}1}\sum_{j=1}^{G-1} R_j, \qquad 
    A_{j,k}^{H} = R_j - \hat{V}_k.
    \]
    All siblings share $S_k$, so $A_{j,k}^{H}$ is free of upstream 
    state contamination 
    (Theorem~\ref{thm:advantage-app}).}
    \State \textbf{PPO update (turn-$k$ tokens only).}
    \For{epoch $= 1, \ldots, E$}
        \For{each mini-batch from sibling turn-$k$ tokens}
            \State Compute IS ratios 
            $\rho_{j,k,t} = \pi_\theta(a_{j,k,t} \mid c_k, 
            a_{j,k,<t}) \,/\, \pi_{\theta_{\mathrm{old}}}(a_{j,k,t} 
            \mid c_k, a_{j,k,<t})$.
            \State Update $\theta$ via the clipped objective:
            \[
            J_k(\theta) = \frac{1}{G{-}1}\sum_{j=1}^{G-1}
            \frac{1}{T_{j,k}}\sum_{t=1}^{T_{j,k}} 
            \min\!\Big(\rho_{j,k,t}\, A_{j,k}^{H},\;
            \mathrm{clip}\bigl(\rho_{j,k,t},\, 1{-}\epsilon,\, 
            1{+}\epsilon\bigr)\, A_{j,k}^{H}\Big).
            \]
        \EndFor
    \EndFor
    \Statex  $\triangleright$ \textit{Turn $k$ complete; 
    $\pi_{\theta,k}$ frozen; next turn inherits updated $\theta$.}
\EndFor
\end{algorithmic}
\end{algorithm}


\subsection{Training Details}
\label{Training Configurations}

We provide the detailed hyperparameters for all experiments in Table~\ref{tab:all-hparams}. Unless otherwise noted, the maximum prompt/response lengths are kept 
the same across all experiments. For computational fairness, the global sampling budget is set to 
$80\text{k}$ rollouts for mathematical reasoning 
($5\text{k}$ prompts $\times$ $16$ rollouts per prompt) and 
$16\text{k}$ for knowledge reasoning 
($1\text{k}$ prompts $\times$ $16$ rollouts per prompt). For the Qwen3 series, we
use a sampling temperature of $0.9$, the AdamW optimizer
with a learning rate of $1\times10^{-6}$, weight decay $0.01$, and an
initial KL coefficient of $0.02$. The log-probability clipping range
is set to $[-10, 10]$. The training batch size and the rollout batch
size are both set to $64$. The maximum prompt length is $8192$ tokens and the maximum single-turn response length is $4096$ tokens for both
tasks. The maximum interaction horizon is set to $K=3$ turns for
mathematical reasoning and $K=6$ turns for web search. All methods
are allocated a uniform budget of $16$ rollouts per prompt (see
Table~\ref{tab:all-hparams}). 

\begin{table} [h]
\centering
\small
\begin{tabular}{llcl}
\toprule
\textbf{Method} & \textbf{Parameters} & \textbf{Total Rollouts} & \textbf{Structure} \\
\midrule
GRPO & \texttt{n\_agent}$=16$ & 16 independent chains & Chain \\
SeeUPO & \texttt{rollout\_n}$=16$ & 16 chains + turn-level update & Chain \\
ARPO & $N{=}8,\; M{=}16$ & 8 initial + entropy branch to 16 & Entropy tree \\
TreeGRPO & $M{=}4,\; N{=}3,\; L{=}1$ & 4 trees $\times$ 4 leaves $=$ 16 & Random tree \\
RTPO & sibling rollouts & Global Max budget & State-matched tree \\
\bottomrule
\end{tabular}
\caption{Hyperparameter settings for all experiments. }
\label{tab:all-hparams}
\end{table}

For GRPO, we sample $16$ independent
chains per prompt. For SeeUPO, we generate $16$ independent chains
with sequential turn-level updates in reverse execution order. For
ARPO, we set the initial sampling size to $N=8$ and the global rollout
budget to $M=16$, with entropy weight $\beta=0.2$, base probability
$\alpha=0.5$, and branching threshold $\tau=0.5$; rollouts that do not
trigger entropy-driven branching are supplemented with independent
chains until the budget of $16$ is reached. For TreeGRPO, we set the
number of initial trees to $M=4$, the number of expansion nodes per
iteration to $N=3$, and the number of expansion iterations to $L=1$,
yielding $M\times(L\times N+1)=16$ rollouts per prompt via random node
expansion. 
For RTPO, we first generate multiple trunk trajectories per prompt to 
establish turn-boundary states: $4$ trunks for mathematical reasoning 
($K{=}3$) and $2$ trunks for web search ($K{=}6$), since in practice 
most trunks do not reach the maximum interaction horizon and terminate 
early with fewer tool calls. At each realized turn boundary $S_k$, 
$G{-}1=2$ sibling rollouts are forked and continued to termination. 

To ensure computational fairness with baselines, RTPO enforces the 
same global sampling budget as all other methods 
($80\text{k}$ for mathematical reasoning, $16\text{k}$ for knowledge 
reasoning), so that the total number of generated rollouts across all 
prompts does not exceed that of any baseline. Per-prompt rollout 
counts may vary depending on early termination, but the global budget 
constraint guarantees that RTPO consumes no more rollout compute than 
the $16$-chain baselines in aggregate.
All methods use a discount factor of $\gamma=1$. 

All experiments are conducted on a single node
equipped with $8\times$ NVIDIA A100 80GB GPUs,
$256$ CPU cores, and $256$GB of system memory,
running Ubuntu 22.04 LTS. The math task requires approximately 15 hours of training, whereas the knowledge task requires approximately 5 hours.

\subsection{Qwen3 Chat Template}
\label{sec:qwen_chat_template}

\begin{figure}[h]
    \centering
    \includegraphics[width=1.0\linewidth]{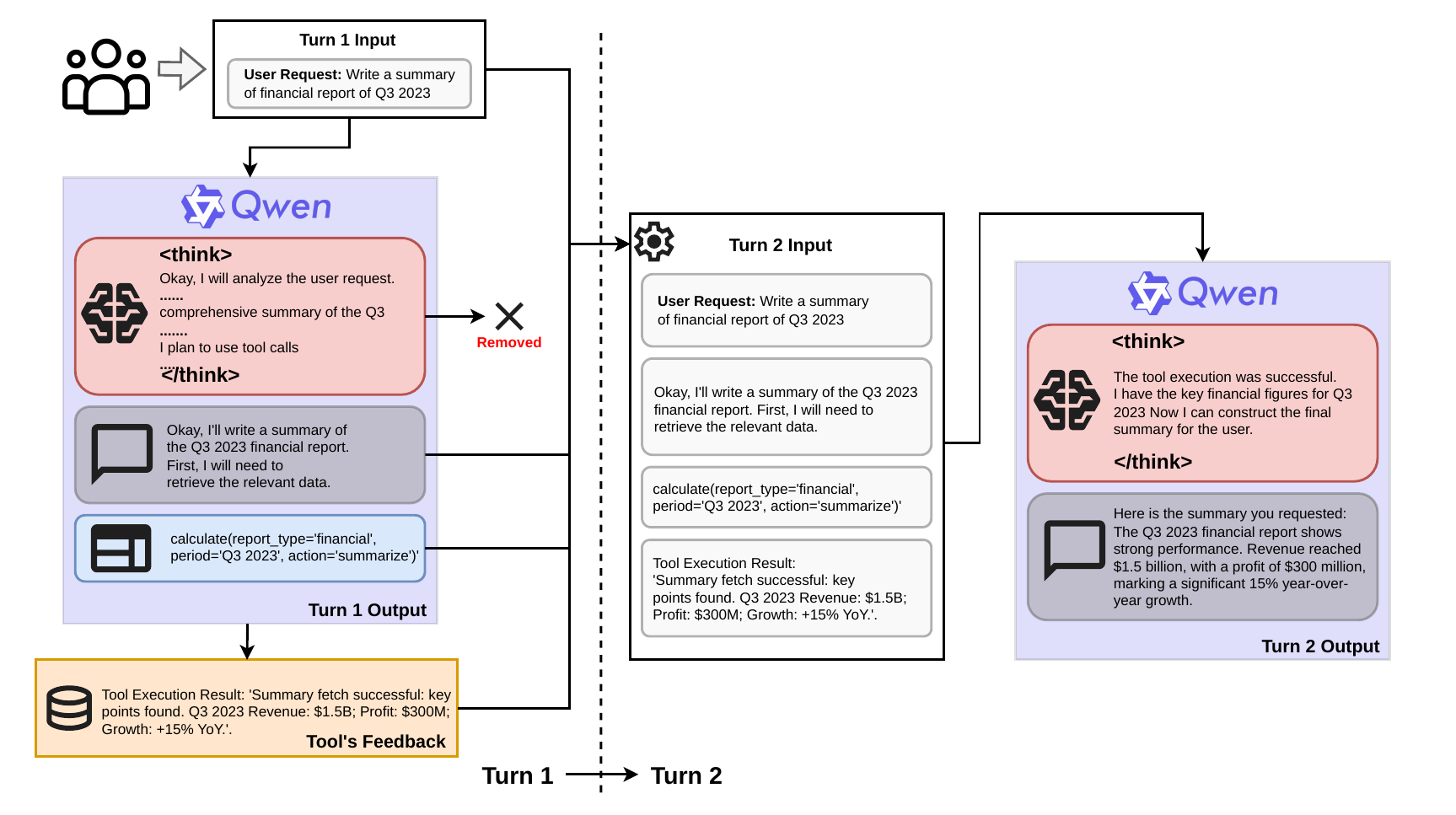}
    \caption{Examples of Qwen3 chat template.}
    \label{fig:chat message}
\end{figure}

At each interaction turn, Qwen3 receives the current context and generates three components: (A) an internal reasoning trace enclosed by \texttt{<think>...</think>}, (B) textual commentary after \texttt{</think>}, and (C) a tool-call instruction, e.g., \texttt{<tool\_call>...</tool\_call>}. The tool call is executed by the external environment, which returns the result as an observation. Before constructing the next-turn input, the chat-template filter removes the \texttt{<think>...</think>} segment, while retaining the textual commentary, the tool-call instruction, and the appended tool-execution result. Thus, the input to the second turn consists of: (1) the original user request, (2) the previous textual commentary, (3) the previous tool call, and (4) the tool result. This filtered context enables the model to condition on the action history and environment feedback without carrying redundant internal reasoning traces, keeping the context length manageable while preserving decision-relevant information. The process is illustrated in Figure~\ref{fig:chat message}.

\section{Additional Results and Insights}
\subsection{Discussion of Main Results}
\label{dis for main}

For the main experiments, we use Qwen3-8B as the base model and evaluate on two categories of multi-turn tool-use tasks: mathematical reasoning and web search. Qwen3 models automatically produce chain-of-thought content wrapped in \texttt{<think>...</think>} tags during inference. In our setup, each turn's input is constructed using the standard Qwen3 chat template, which post-processes prior assistant content by stripping out the \texttt{<think>...</think>} segments so that subsequent turns observe only the final answer content following \texttt{</think>}. This design keeps the per-turn input context concise and aligns naturally with the turn-boundary MDP formulation underlying RTPO: each turn-level policy is conditioned on the visible context $c_k$ rather than on the full raw interaction history. In addition, we report the complete accuracy results with mean and standard deviation, extending Table~\ref{tab:math_knowledge_toolcall_results} to Table~\ref{tab:acc_only_results}.

\begin{table*}[!ht]
\centering
\scriptsize
\setlength{\tabcolsep}{5.0pt}
\renewcommand{\arraystretch}{0.95}
\resizebox{\textwidth}{!}{
\begin{tabular}{l|cccccc|cc}
\toprule
\multirow{2}{*}{\textbf{Method}}
& \multicolumn{6}{c|}{\textbf{M: Mathematical Reasoning}}
& \multicolumn{2}{c}{\textbf{K: Knowledge Reasoning}} \\
\cmidrule(lr){2-7}
\cmidrule(lr){8-9}
& \textbf{AIME24}
& \textbf{AIME25}
& \textbf{AMC23}
& \textbf{MATH500}
& \textbf{GSM8K}
& \textbf{OE-Math}
& \textbf{HotpotQA}
& \textbf{2Wiki} \\
\midrule

Vanilla
& $3.33_{\pm0.00\%}$
& $13.33_{\pm3.33\%}$
& $12.50_{\pm2.50\%}$
& $51.00_{\pm0.40\%}$
& $76.95_{\pm0.15\%}$
& $24.33_{\pm0.68\%}$
& $\underline{54.52}_{\pm0.72\%}$
& $59.83_{\pm0.61\%}$ \\

GRPO
& $10.00_{\pm3.33\%}$
& $20.00_{\pm3.33\%}$
& $57.50_{\pm2.50\%}$
& $82.33_{\pm0.50\%}$
& $93.86_{\pm0.23\%}$
& $50.74_{\pm1.12\%}$
& $53.39_{\pm0.93\%}$
& $61.85_{\pm0.84\%}$ \\

SeeUPO
& $\underline{20.00}_{\pm3.33\%}$
& $\underline{23.33}_{\pm3.33\%}$
& $\underline{67.50}_{\pm2.50\%}$
& $\underline{84.20}_{\pm0.35\%}$
& $\underline{94.47}_{\pm0.13\%}$
& $\underline{53.86}_{\pm0.45\%}$
& $54.27_{\pm0.55\%}$
& $\underline{63.78}_{\pm0.49\%}$ \\

RTPO
& $\mathbf{33.33}_{\pm0.00\%}$
& $\mathbf{26.67}_{\pm0.00\%}$
& $\mathbf{67.50}_{\pm0.00\%}$
& $\mathbf{86.20}_{\pm0.20\%}$
& $\mathbf{94.84}_{\pm0.08\%}$
& $\mathbf{55.49}_{\pm0.30\%}$
& $\mathbf{64.89}_{\pm0.34\%}$
& $\mathbf{64.35}_{\pm0.27\%}$ \\

\bottomrule
\end{tabular}
}
\caption{Accuracy comparison across eight benchmarks. Values are reported as mean$_{\pm\mathrm{std}}$ over three evaluation runs. \textit{Acc} denotes Pass@1 (\%) for mathematical tasks and best-span F1 (\%) for knowledge reasoning tasks. Standard deviations are reported in percentage points.}
\label{tab:acc_only_results}
\vspace{-10pt}
\end{table*}

\paragraph{Performance and tool-use patterns on mathematical reasoning tasks.}
On mathematical tasks, tool-call frequency exhibits a trend that runs almost opposite to accuracy: Vanilla makes 17 Python calls on AIME24 yet solves only one problem, whereas RTPO solves ten with just four calls. This pattern reflects an intrinsic property of mathematical reasoning: the knowledge and derivations required to solve a math problem are primarily internalized in the model's parameters, while the external Python tool serves as an auxiliary aid for verification and numerical computation rather than as a source of new information. Vanilla's frequent tool use, therefore, largely reflects an inefficient strategy of \emph{repeated verification used to mask reasoning uncertainty}. After RL training, the model gains stronger control over its own reasoning process, and tool calls collapse from redundant repeated verification into precise invocations at critical computation steps; consequently, the number of calls decreases while accuracy increases. RTPO's advantage becomes particularly pronounced on the hardest math benchmarks: it reaches 33.33\% Pass@1 on AIME24, surpassing SeeUPO (20.00\%) and GRPO (10.00\%) by 13.3 and 23.3 absolute points, respectively, and maintains a clear lead on AIME25 and OE-Math. RTPO simultaneously achieves the fewest tool calls and the highest accuracy on these hard problems, indicating that it learns the most efficient tool-use strategy. We further note that the gap between methods on mathematical tasks remains relatively small overall, since math problems typically require only a few interaction turns, and the context discrepancy between rollout and training stays at a manageable scale under short horizons; the robustness of standard training dynamics alone is sufficient for baselines to reach near-optimal performance in this regime. We provide an example to support our discussion, as shown in Figure~\ref{fig:mathcase}.

\begin{figure}[h]
    \centering
    \includegraphics[width=1.0\linewidth]{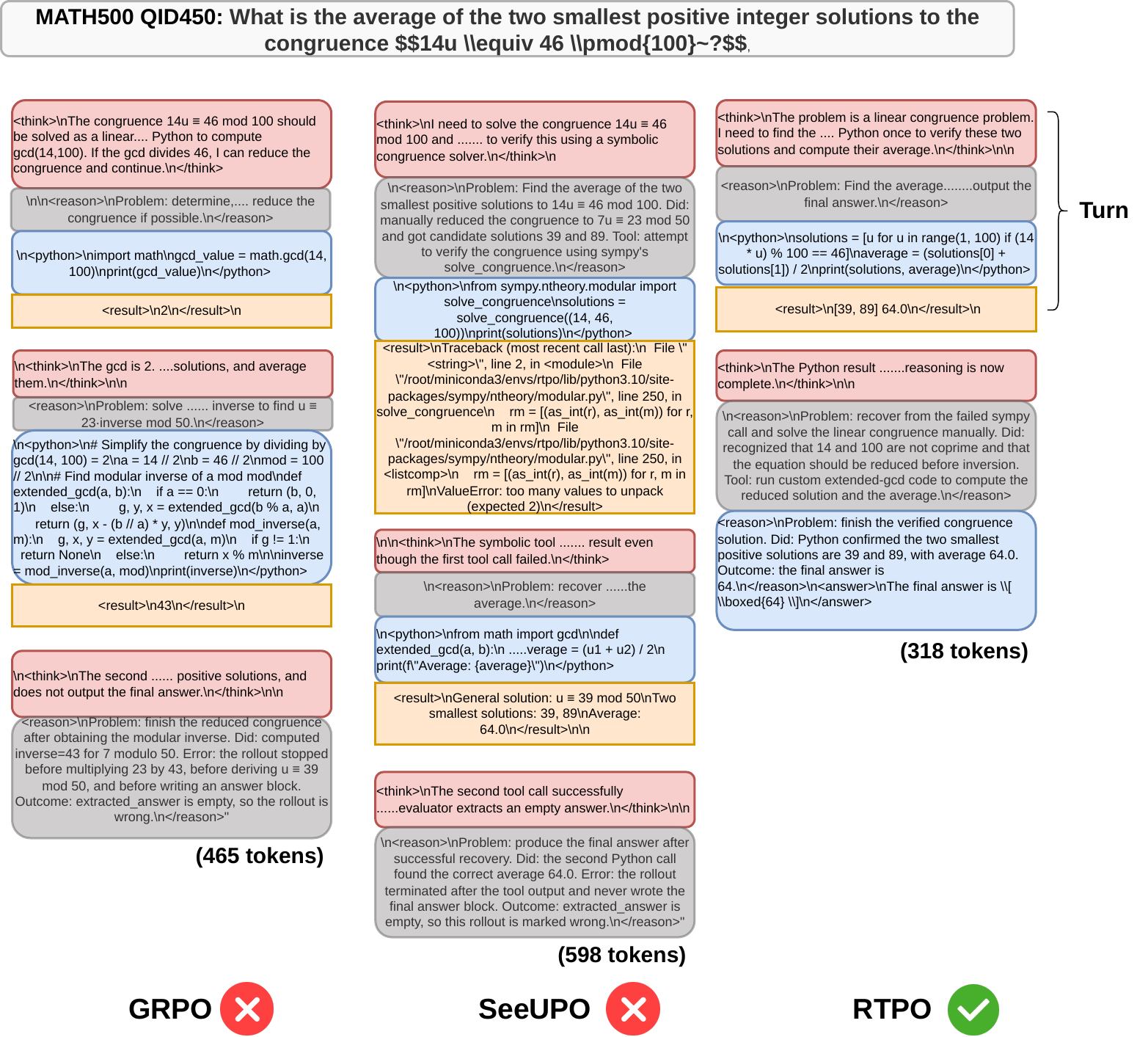}
    \caption{An example of different policy performance on a mathematical-reasoning task through Python tools.}
    \label{fig:mathcase}
\end{figure}

\paragraph{Performance and tool-use patterns on knowledge reasoning tasks.}
On knowledge-intensive question answering, RTPO reaches 64.89 F1 on HotpotQA, exceeding SeeUPO (54.27) and GRPO (53.39) by 10.6 and 11.5 points, respectively, and maintains a clear lead on 2Wiki. In contrast to mathematical tasks, the tool-use pattern on knowledge reasoning exhibits a positive correlation between call count and accuracy: RTPO issues 613 and 867 search calls on HotpotQA and 2Wiki, far exceeding the roughly 200 calls observed for the other methods. This contrast admits a natural explanation grounded in task structure: the factual information required to answer such questions does not reside in the model's parametric knowledge and must be acquired through external retrieval, while multi-hop questions further demand cross-turn integration of multiple pieces of evidence. The lower retrieval counts of GRPO and SeeUPO indicate that they fail to learn to issue sustained follow-up queries and to expand retrieval across multiple turns. The underlying cause is that such hard tasks require longer interaction horizons and more frequent reasoning revision, so the rollout--training context mismatch is amplified as turns accumulate, and baseline methods must spend a larger share of their optimization budget compensating for this drift. RTPO eliminates this source of bias structurally, allowing the entire optimization budget to act directly on the task objective, which yields more stable improvements on long-horizon hard tasks. We provide an example to support our discussion, as shown in Figure~\ref{fig:searchcase}.

\begin{figure}[h]
    \centering
    \includegraphics[width=1.0\linewidth]{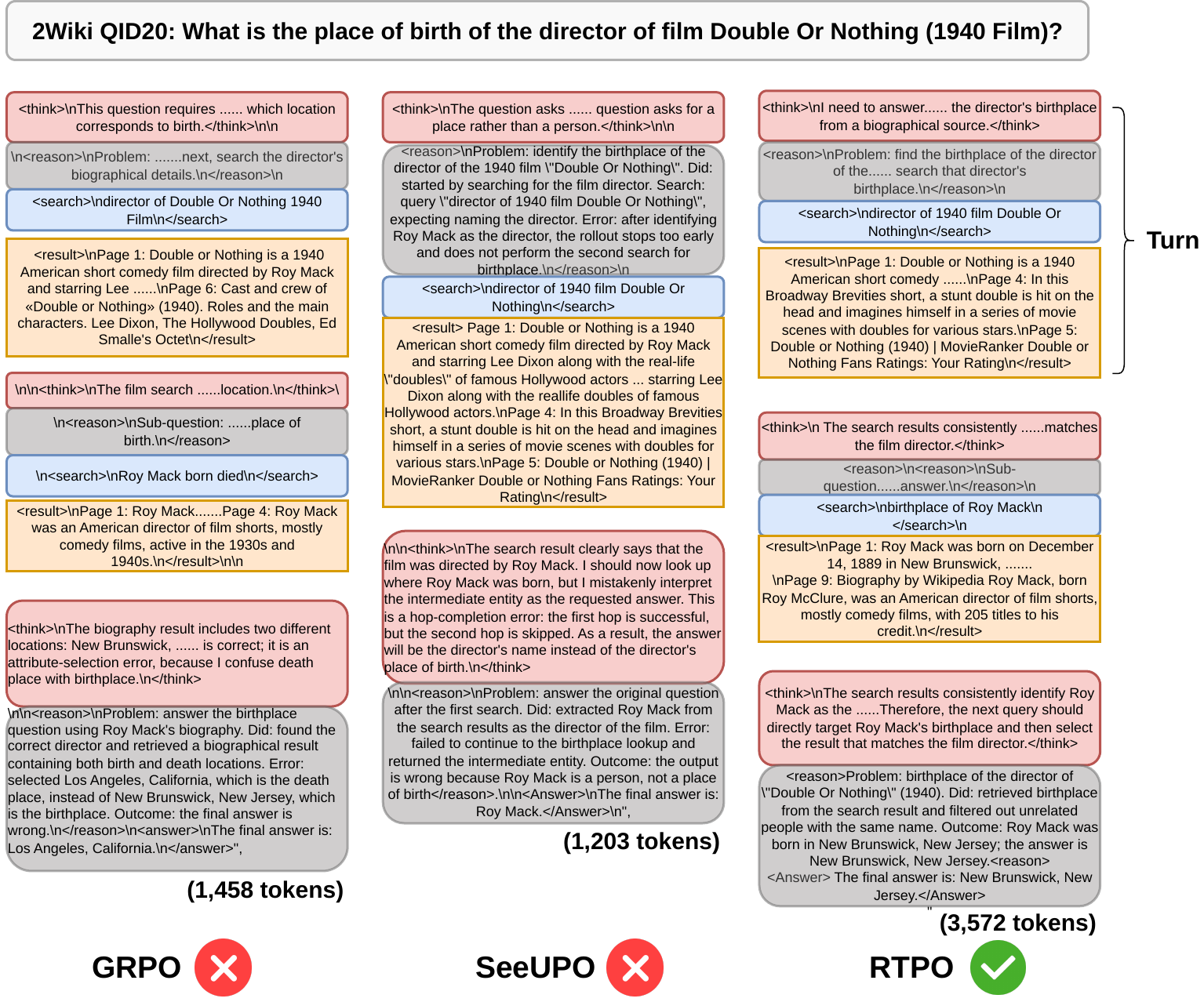}
    \caption{An example of different policy performance on a knowledge-reasoning task through web search tools.}
    \label{fig:searchcase}
\end{figure}

Taking both task categories together, RTPO learns to perform mathematical reasoning with \emph{fewer calls but more reliable internal reasoning}, while performing knowledge reasoning with \emph{denser calls and more thorough external information integration}. This bidirectional adaptation of tool-use behavior indicates that the advantage of RTPO does not stem from simply encouraging or suppressing tool calls, but rather from its turn-level optimization objective, which is able to learn a tool-use strategy matched to the underlying nature of each task.

\subsection{Insights from Rollout--Training Consistency Analysis}
\label{dis for rtc}

Follow-up discussion from Sec~\ref{sec:stability}, a noteworthy observation is that, despite using different conditioning contexts at the rollout and training stages, the baselines' ratios still exhibit a slow drift toward $1$. We find that this phenomenon is closely related to the mechanism revealed by DAgger~\citep{ross2011reduction}: when the training data is continually drawn from the distribution induced by the policy itself, the model gradually adapts to the distribution it actually operates on. In our setting, trajectories generated by summary-rollouts are reinforced under full-history-training, and parameter sharing causes the behaviors under the two conditionings to converge indirectly; the next rollout therefore falls closer to the region considered reasonable under the training context. SCoRe~\citep{kumartraining} observes the same bootstrapped alignment process in multi-turn online RL. In addition, some works point out that the summary can serve as a learnable sufficient statistic: if the summary preserves the decision-relevant information, the optimal policies under the two contexts can converge in an information-theoretic sense~\citep{allen2021learning}; recent work on multi-turn RL further shows that the summary context can evolve into a learnable compact decision state under end-to-end optimization~\citep{lu2025scaling,yu2026memagent}.

However, this empirical alignment differs from the structural consistency of RTPO in three fundamental ways:

(1) Alignment is incomplete and strongly depends on task complexity.
GRPO recovers only to $0.97$ on math after $30$ steps and only to $0.83$ on search after $14$ steps; the longer the horizon and the faster the context accumulates, the harder the alignment becomes, which is reflected in the knowledge reasoning task in Table~\ref{tab:math_knowledge_toolcall_results}.

(2)The alignment process is accompanied by oscillation.
SeeUPO produces a $1.02$ spike at step $9$ on math and exhibits sustained small fluctuations, because the bootstrapped loop is driven by the advantage signal, and the variance of advantage estimates propagates directly into step-to-step jitter of the ratio. RTPO's ratio is structurally guaranteed and is only affected by differences between the training and inference engines.

(3) Empirical alignment consumes additional optimization budget.
Bootstrapped alignment essentially allocates part of the policy's capacity to an implicit objective---pulling the behavior under the summary toward the optimal behavior under the full history. RTPO removes this hidden cost, so that the optimization budget can act directly on the task objective. This is consistent with the lower search-call counts of baselines on long-horizon knowledge reasoning observed in Table~\ref{tab:math_knowledge_toolcall_results}: part of their training dynamics is diverted to \textit{patching the mismatch}.

Overall, the rollout--training ratio reveals not that baselines necessarily fail, but that baselines must rely on training dynamics to compensate for a gap that RTPO does not have by construction, and this compensatory mechanism becomes substantially less effective on long-horizon tasks.

\subsection{Additional Findings for Hard-Search Scenarios}
\label{add subsets}

\begin{table} [b]
\centering
\small
\begin{tabular}{lrrrrrr}
\toprule
\textbf{Benchmark} & \textbf{QS} & \textbf{Off Hit} & \textbf{Off Rate} & \textbf{On Hit} & \textbf{On Rate} & \textbf{$\Delta$} \\
\midrule
\textbf{GAIA} & 103 & 24 & 23.30\% & 30 & 29.13\% & +5.83\% \\
\quad L1 & 39 & 7 & 17.95\% & 10 & 25.64\% & +7.69\% \\
\quad L2 & 52 & 16 & 30.77\% & 18 & 34.62\% & +3.85\% \\
\quad L3 & 12 & 1 & 8.33\% & 2 & 16.67\% & +8.33\% \\
\midrule
\textbf{WebWalker} & 200 & 12 & 6.00\% & 19 & 9.50\% & +3.50\% \\
\midrule
\textbf{XBench} & 100 & 8 & 8.00\% & 15 & 15.00\% & +7.00\% \\
\midrule
\textbf{HLE} & 2096 & 261 & 12.45\% & 254 & 12.12\% & -0.33\% \\
\quad Biology/Medicine & 228 & 37 & 16.23\% & 38 & 16.67\% & +0.44\% \\
\quad Chemistry & 81 & 6 & 7.41\% & 6 & 7.41\% & +0.00\% \\
\quad Computer Science/AI & 220 & 20 & 9.09\% & 24 & 10.91\% & +1.82\% \\
\quad Engineering & 71 & 5 & 7.04\% & 8 & 11.27\% & +4.23\% \\
\quad Humanities/Social Science & 192 & 32 & 16.67\% & 23 & 11.98\% & -4.69\% \\
\quad Math & 947 & 114 & 12.04\% & 112 & 11.83\% & -0.21\% \\
\quad Other & 158 & 29 & 18.35\% & 29 & 18.35\% & +0.00\% \\
\quad Physics & 199 & 18 & 9.05\% & 14 & 7.04\% & -2.01\% \\
\bottomrule
\end{tabular}
\caption{Rechecked output-hit comparison between standard RTPO (on-policy) and its off-policy variant. Values are rounded to the nearest integer for better readability. Output-hit marks a sample correct if the output contains the gold answer or an alias.}
\label{tab:recheck_output_hit}
\end{table}

To quantify the independent contribution of on-policy continuation (Sec.~\ref{sec:onpolicy}), we use Qwen3-8B as the base model and compare standard RTPO with its off-policy variant. Both share the same reverse-order turn-level training and state-matched sibling structure; the only difference lies in how sibling downstream continuations are obtained. RTPO (on-policy) synchronizes the latest parameters $\theta_{>k}$ to the inference engine at the start of turn $k$ and regenerates sibling continuations under $\pi_{\theta_{>k}}$ until termination, so the trajectory-level IS weight is identically one, and the Q-value
estimate is
\begin{equation}
\hat{Q}^{\text{on}}_{j,k}
= r_{j,k} + \gamma^{\tau_{j,k}}\,
  \hat{F}_{j,k}^{\pi_{\theta_{>k}}}.
\end{equation}
RTPO with off-policy variant instead reuses the downstream continuations already generated under $\pi_{\theta_0}$ during the initial rollout stage and
corrects for the staleness via a clamped trajectory-level importance
weight:
\begin{equation}
\hat{Q}^{\text{off}}_{j,k}
= r_{j,k} + \gamma^{\tau_{j,k}}\,
  \bar{\omega}_{j,k}\,
  \hat{F}_{j,k}^{\pi_{\theta_0}},
\qquad
\bar{\omega}_{j,k}
= \mathrm{clamp}\!\Bigg(
  \prod_{h=k+1}^{K-1}\prod_{t=1}^{T_{j,h}}
  \frac{\pi_{\theta_{>k}}(a_{j,h,t}\mid s_{j,h,t})}
       {\pi_{\theta_0}(a_{j,h,t}\mid s_{j,h,t})},\;
  \rho_{\min},\;\rho_{\max}
\Bigg).
\end{equation}
We compare the on-policy RTPO and its off-policy variant on four knowledge-reasoning deep-search benchmarks, as shown in Table~\ref{tab:recheck_output_hit}. We use output-hit accuracy, which marks a sample correct if the full output contains the gold answer or an alias. The benchmarks include GAIA with three difficulty levels, WebWalkerQA, XBench, and HLE with eight subject subsets.

\subsection{Limitations and Future Work}
\label{limits}

Although RTPO provides stronger theoretical guarantees and empirical performance than existing flat-trajectory methods, its algorithmic design has several limitations that motivate future work.

\paragraph{Dependence on trunk quality.}
RTPO uses the boundary states $S_k$ of a trunk trajectory as the forking points for sibling branches. As a result, the training signal at each reverse phase is anchored to the state sequence visited by the trunk. If the trunk makes a poor decision at an early turn, such as $k=0$ or $k=1$, later boundary states may lie in low-value regions where even strong turn-level actions fail to obtain positive terminal rewards. In this case, the turn-level advantages $A_{j,k}^{H}$ may degenerate into near-zero signals, making the corresponding gradient update ineffective. In contrast, beam-style search methods, such as beam search or best-of-$N$, maintain multiple candidate prefixes and can discard low-quality paths earlier. RTPO currently relies on a single trunk anchor and does not explicitly incorporate trunk-level diversity or post-hoc trunk selection. A natural extension is \emph{multi-trunk sampling} or \emph{retroactive trunk selection}, where multiple complete trunks are generated during rollout and a high-reward trunk is selected as the anchor to improve boundary-state coverage. This would increase rollout cost, but would not change the reverse-order training formulation.

\paragraph{Overhead of reverse multi-turn training.}
RTPO decomposes a $K$-turn episode into $K$ sequential turn-level optimization phases. Each phase requires sibling generation, environment restoration, on-policy continuation to termination, and a PPO-style update. Compared with flat-trajectory methods, which perform a single optimization pass over the full trajectory, RTPO incurs additional cost that scales with the number of turns $K$ and the number of sibling rollouts. Moreover, each reverse phase requires synchronizing the latest policy parameters to the inference engine before generating on-policy continuations, introducing additional inference-training latency. In our VeRL-based implementation, this overhead is partially mitigated by batched parallel sibling generation and asynchronous engine scheduling, but it cannot be fully removed. Designing more efficient sibling generation and update schedules is therefore an important direction for future work.

\paragraph{Under-utilization of training tokens.}
RTPO assigns the turn-level advantage in phase $k$ only to the sibling's turn-$k$ output tokens. Prefix tokens in $c_k$ and downstream continuation tokens from turns $k{+}1,\ldots,K{-}1$ receive no gradient. In addition, the trunk trajectory is used only as an anchor and does not directly contribute to gradient updates. Thus, although sibling continuations are necessary for estimating terminal returns, their downstream tokens are discarded during policy optimization. For example, in a $K=5$ episode with an average of 200 tokens per turn, the sibling continuation at phase $k=2$ may generate around 600 downstream tokens, while only the 200 turn-$k$ tokens are used for the PPO update. This reduced token utilization is the cost of causal turn-level credit assignment: by withholding gradients from non-turn-$k$ tokens, RTPO avoids assigning credit to actions that are not causally responsible for the turn-$k$ comparison, as stated in \textit{Theorem~2(c)}. Future work may explore auxiliary objectives, such as language-modeling losses or self-play rewards on downstream tokens, to improve token efficiency while preserving the causal consistency of the turn-level advantage.

\subsection{Broader Impacts}
\label{impact}
The potential positive impact of RTPO is that more stable agentic RL training can reduce failed tool-use trajectories, improve sample efficiency, and support more reliable deployment of LLM agents in research, education, software engineering, and decision-support settings. RTPO may also make multi-turn RL training easier to analyze by separating turn-level decisions from full-trajectory outcomes. By enabling turn-level monitoring, RTPO can further improve our understanding of how agentic workflows learn to plan, search, and use tools over multiple turns. At the same time, stronger turn-refined agentic workflows may increase the capability of LLM agents to act autonomously across long-horizon tasks. If deployed without appropriate safeguards, such systems could produce incorrect outputs with high confidence, misuse external tools, or amplify harmful automation. Therefore, practical deployment should include safety constraints, tool-use monitoring, privacy-preserving data handling, and human oversight, especially in high-stakes domains.

%

\clearpage
\section{Supplementary Theoretical Clarifications, Implementation Details, and Extended Experiments}
\label{sec:supplementary_clarifications}

\subsection{Stateful Tool-Agent Evaluation on \texorpdfstring{$\tau^3$}{tau3}-Airline}
\label{sec:supp_airline}

\paragraph{Experimental setting.}
We evaluate Qwen3-1.7B on a fixed set of 30 training tasks and the 20 held-out
test tasks provided by the $\tau^3$-Airline environment. The agent must query
and modify an airline database through multi-turn tool interactions, and
earlier actions change the state observed in later turns.

\begin{table}[htbp]
    \centering
    \small
    \caption{Online training success rate on $\tau^3$-Airline.}
    \label{tab:supp_airline_training}
    \begin{tabular}{lccccc}
        \toprule
        Method & Step 10 & Step 20 & Step 30 & Step 40 & Step 50 \\
        \midrule
        GRPO                 & 23.44\% & 25.00\% & 31.26\% & 29.69\% & 31.26\% \\
        SeeUPO               & 31.23\% &  9.38\% & 14.83\% &  0.78\% &  0.78\% \\
        Tree-GRPO            & 18.40\% & 14.95\% & 27.48\% & 24.53\% & 14.88\% \\
        REFUEL               & 25.00\% & 31.25\% & 31.25\% & 28.12\% & 21.88\% \\
        \textbf{RTPO ($G=3$)}& 15.63\% & \textbf{37.50\%} & \textbf{50.00\%}
                              & \textbf{56.25\%} & \textbf{40.63\%} \\
        \bottomrule
    \end{tabular}
\end{table}

RTPO obtains the highest online success rate from Step~20 onward and peaks at
56.25\% at Step~40. The reduction to 40.63\% at Step~50 indicates that
continued optimization can produce late-stage degradation when the training
set is small and the reward is sparse.

\begin{table}[htbp]
    \centering
    \small
    \caption{Held-out evaluation on the 20 $\tau^3$-Airline test tasks.
    Pass$^4$ follows the definition used in the main paper.}
    \label{tab:supp_airline_test}
    \resizebox{\textwidth}{!}{
    \begin{tabular}{lcccccc}
        \toprule
        Method & Pass@1 & Pass@4 & Pass$^4$ &
        Normal termination & Generation truncation &
        Avg.\ response tokens \\
        \midrule
        SeeUPO
        & 0.00\% & 0.00\% & 0.00\% & 0.00\% & 98.75\% & 2,023.7 \\
        GRPO
        & 11.25\% & 25.00\% & 0.00\% & 20.00\% & 60.00\% & 2,469.7 \\
        Tree-GRPO
        & 16.25\% & 25.00\% & 10.00\% & 30.00\% & 47.50\% & 2,538.2 \\
        REFUEL
        & 11.25\% & 20.00\% & 0.00\% & 18.75\% & 57.50\% & 2,621.2 \\
        \textbf{RTPO ($G=3$)}
        & \textbf{17.50\%} & \textbf{30.00\%} & \textbf{10.00\%}
        & \textbf{45.00\%} & \textbf{42.50\%} & \textbf{1,688.8} \\
        \bottomrule
    \end{tabular}}
\end{table}

RTPO achieves the highest Pass@1 and Pass@4 and the highest normal
environment-termination rate. Relative to Tree-GRPO, it reduces the
generation-truncation rate from 47.50\% to 42.50\% and the average response
length from 2,538.2 to 1,688.8 tokens. Relative to REFUEL, RTPO improves
Pass@1 from 11.25\% to 17.50\%, improves Pass@4 from 20.00\% to 30.00\%, and
reduces average response length by approximately 35.6\%.

\begin{table}[htbp]
    \centering
    \small
    \caption{Tool-execution quality on the held-out $\tau^3$-Airline tasks.}
    \label{tab:supp_airline_tools}
    \resizebox{\textwidth}{!}{
    \begin{tabular}{lccccc}
        \toprule
        Method & Attempted calls & Executed successfully & Tool-error calls &
        Execution success & Response tokens/success \\
        \midrule
        GRPO      & 118 & 68 & 50 & 57.63\% & 21,952.9 \\
        SeeUPO    &   1 &  0 &  1 &  0.00\% & N/A \\
        Tree-GRPO &   1 &  0 &  1 &  0.00\% & 15,619.5 \\
        \textbf{RTPO}
                  &  83 & \textbf{74} & \textbf{9} & \textbf{89.16\%}
                  & \textbf{9,650.5} \\
        \bottomrule
    \end{tabular}}
\end{table}

RTPO attempts fewer tool calls than GRPO but completes more calls without
execution errors. Tool-error calls decrease from 50 to 9, and the execution
success rate increases from 57.63\% to 89.16\%. RTPO also requires
approximately 9,650.5 response tokens per successful trajectory, which is
38.2\% lower than Tree-GRPO and 56.0\% lower than GRPO.

\subsection{Sensitivity to the Sibling Group Size}
\label{sec:supp_group_size}

The group-size hyperparameter is $G$; at each boundary, RTPO forks
$G-1$ sibling continuations from the same turn-boundary state $S_k$. The
turn-level estimator is
\begin{equation}
    \widehat{A}^{H}_{j,k}
    =
    R_j-\frac{1}{G-1}\sum_{r=1}^{G-1}R_r .
    \label{eq:supp_sibling_estimator}
\end{equation}
Theorem~\ref{thm:advantage-app}(a) gives
\begin{equation}
    \mathbb{E}\!\left[
        \widehat{A}^{H}_{j,k}
        \mid S_k,u_{j,k}
    \right]
    =
    \frac{G-2}{G-1}
    A^\pi(S_k,u_{j,k}) .
    \label{eq:supp_sibling_bias}
\end{equation}
Thus, the finite-group bias is of order $O(1/G)$, and the multiplicative
coefficient approaches one as $G$ increases. Averaging more sibling returns
also reduces sampling noise in the Monte Carlo baseline. Because all
continuations begin from the same boundary state, this comparison does not
reintroduce upstream-state contamination.

\begin{table}[htbp]
    \centering
    \small
    \caption{Effect of sibling group size on signal density and held-out
    performance in $\tau^3$-Airline.}
    \label{tab:supp_group_sensitivity}
    \resizebox{\textwidth}{!}{
    \begin{tabular}{cccccc}
        \toprule
        $G$ & Theoretical coefficient &
        Pass@1 & Pass@4 & Zero-advantage rate & Tokens/success \\
        \midrule
        3 & $1/2$ & 17.50\% & 30.00\% & 89.61\% & 9,650.5 \\
        4 & $2/3$ & \textbf{20.00\%} & 30.00\% &
        \textbf{82.66\%} & \textbf{7,477.9} \\
        \bottomrule
    \end{tabular}}
\end{table}

Increasing $G$ from 3 to 4 reduces the zero-advantage rate from 89.61\% to
82.66\%, improves Pass@1 from 17.50\% to 20.00\%, and reduces inference
tokens per successful trajectory from 9,650.5 to 7,477.9. The larger group
therefore provides denser relative learning signals, but requires more
offline exploration.

\begin{table}[htbp]
    \centering
    \small
    \caption{Performance--cost trade-off for sibling group size.}
    \label{tab:supp_group_cost}
    \resizebox{\textwidth}{!}{
    \begin{tabular}{ccccccc}
        \toprule
        $G$ & GPU-hours & Training-generation tokens & Pass@1 & Pass@4 &
        Avg.\ tool calls & Avg.\ response tokens \\
        \midrule
        3 & 21.48 & 846,207   & 17.50\% & 30.00\% & 1.038 & 1,688.8 \\
        4 & 28.06 & 2,139,039 & 20.00\% & 30.00\% & 0.662 & 1,495.6 \\
        \bottomrule
    \end{tabular}}
\end{table}

Moving from $G=3$ to $G=4$ increases GPU-hours by approximately 30.6\% and
training-generation tokens by approximately 152.8\%, while improving Pass@1
by 2.5 percentage points and leaving Pass@4 unchanged. The resulting policy
uses approximately 36.2\% fewer tool calls and 11.4\% fewer response tokens
at inference time. Consequently, $G=3$ provides the stronger default
cost--performance trade-off, whereas $G=4$ is useful when Pass@1 and concise
inference behavior are prioritized.

For $G=2$, the coefficient in Eq.~\eqref{eq:supp_sibling_bias} is zero, so the
conditional expectation of the relative advantage degenerates to zero.
Therefore, $G=3$, corresponding to two sibling continuations, is the minimum
viable configuration that preserves an informative relative signal.

\subsection{Trunk Quality and Failure Dynamics}
\label{sec:supp_trunk_quality}

\begin{table}[htbp]
    \centering
    \small
    \caption{Aggregate trunk outcomes in the $\tau^3$-Airline training run.}
    \label{tab:supp_trunk_distribution}
    \begin{tabular}{lc}
        \toprule
        Trunk metric & Percentage of all trunks \\
        \midrule
        Final task success                  & 40.00\% \\
        Failed within Turn 1                &  6.88\% \\
        Failed within the first 2 turns     & 22.50\% \\
        Failed within the first 3 turns     & 28.75\% \\
        Failed within the first 5 turns     & 39.38\% \\
        No terminal environment signal      & 10.63\% \\
        Average trunk length                & 5.39 turns \\
        \bottomrule
    \end{tabular}
\end{table}

Only 6.88\% of trunks fail within the first turn, indicating that catastrophic
early failure is not the dominant failure mode. Most errors occur in the
middle or later stages. Moreover, an observed early failure under finite
sampling is not equivalent to a strict dead state, because finite
continuations cannot establish that every possible future policy is unable
to recover.

\begin{table}[htbp]
    \centering
    \small
    \caption{Evolution of failure modes during RTPO training.}
    \label{tab:supp_failure_dynamics}
    \begin{tabular}{lccccc}
        \toprule
        Metric & Step 10 & Step 20 & Step 30 & Step 40 & Step 50 \\
        \midrule
        Success rate
        & 15.63\% & 37.50\% & 50.00\% & 56.25\% & 40.63\% \\
        Normal termination but failure
        & 46.88\% & 53.13\% & 46.88\% & 40.63\% & 59.38\% \\
        No terminal signal
        & 37.50\% & 9.38\% & 3.13\% & 3.13\% & 0.00\% \\
        Failure within Turn 1
        & 3.13\% & 6.25\% & 9.38\% & 6.25\% & 9.38\% \\
        Failure within 2 turns
        & 15.63\% & 15.63\% & 31.25\% & 25.00\% & 25.00\% \\
        Failure within 3 turns
        & 21.88\% & 25.00\% & 34.38\% & 25.00\% & 37.50\% \\
        Failure within 5 turns
        & 25.00\% & 40.63\% & 40.63\% & 31.25\% & 59.38\% \\
        Avg.\ successful turns
        & 4.40 & 3.92 & 1.88 & 2.06 & 1.92 \\
        Avg.\ failed turns
        & 13.30 & 7.30 & 4.06 & 5.50 & 2.84 \\
        \bottomrule
    \end{tabular}
\end{table}

The fraction of trajectories without a terminal environment signal decreases
from 37.50\% at Step~10 to 0\% at Step~50, while the average failed-trajectory
length decreases from 13.30 to 2.84 turns. RTPO therefore substantially
reduces trajectories that stall for a long time or fail to complete the
interaction. Later failures increasingly take the form of fast but incorrect
termination rather than persistent generation until truncation. The drop in
success after Step~40 may reflect over-optimization on a limited task set
under sparse binary rewards.

The implementation does not use heuristic trunk filtering. All trunks are
sampled on-policy from the current model. Filtering low-quality trunks could
reduce sibling-sampling cost, but would alter the actually visited state
distribution and concentrate optimization on manually selected states.
Instead, RTPO retains all trunks and uses same-state sibling returns to
determine whether a boundary supplies an informative relative signal.

\subsection{Selective-Gradient Optimization and Training Efficiency}
\label{sec:supp_efficiency}

\paragraph{Why the loss is restricted to current-turn tokens.}
RTPO applies the policy loss only to the output tokens of the current-turn
siblings. Prefix tokens and the trunk trajectory define the state and
conditioning context that were actually reached, while the downstream
continuation supplies the return used to evaluate the current action.
Assigning the same turn-level advantage to prefix or downstream-continuation
tokens would reintroduce the trajectory-level credit contamination analyzed
in Theorem~\ref{thm:advantage-app}. Although only one turn receives gradients at a particular
reverse stage, every turn is optimized when it becomes the current turn
during the complete reverse sweep. The selective loss is therefore intended
to isolate causal credit rather than to discard particular turns from
training.

\begin{table}[htbp]
    \centering
    \small
    \caption{Training-generation cost and deployment-time response efficiency
    on $\tau^3$-Airline.}
    \label{tab:supp_training_tokens}
    \begin{tabular}{lccc}
        \toprule
        Method & Training-generation tokens & Pass@1 & Response tokens/success \\
        \midrule
        GRPO                 & 604,532 & 11.25\% & 21,952.9 \\
        Tree-GRPO            & 781,325 & 16.25\% & 15,619.5 \\
        \textbf{RTPO ($G=3$)}& 846,207 & \textbf{17.50\%} & \textbf{9,650.5} \\
        \bottomrule
    \end{tabular}
\end{table}

RTPO generates approximately 8.3\% more training tokens than Tree-GRPO and
40.0\% more than GRPO, while attaining the highest Pass@1. Relative to
Tree-GRPO, it reduces response tokens per successful trajectory by 38.2\%;
relative to GRPO, the reduction is approximately 56.0\%. The additional
offline continuation sampling therefore does not translate into more verbose
deployment-time inference.

\paragraph{Wall-clock and GPU-hour overhead.}
Under the same Qwen3-1.7B model, four A100 GPUs, and 160 training trunks, the
measured cost is:

\begin{table}[htbp]
    \centering
    \small
    \caption{End-to-end training cost under the same hardware.}
    \label{tab:supp_wall_clock}
    \begin{tabular}{lcccc}
        \toprule
        Method & Wall-clock & GPU-hours & Generated tokens & Relative cost \\
        \midrule
        GRPO            & 3.80 h & 15.20 & 604,532 & $1.00\times$ \\
        RTPO ($G=3$)    & 5.37 h & 21.48 & 846,207 & $1.41\times$ \\
        \bottomrule
    \end{tabular}
\end{table}

RTPO increases wall-clock time and GPU-hours by approximately 41.3\%, and
training-generation tokens by approximately 40.0\%. The additional cost comes
primarily from sibling-continuation generation, turn-by-turn reverse updates,
and synchronization of the latest policy between reverse stages. RTPO thus
trades higher training-time computation for state-matched credit assignment
and current-policy downstream continuations.

\paragraph{Compute-matched comparison.}
The following comparison uses approximately the same total GPU-hour budget.
It is distinct from the fully trained RTPO result at 21.48 GPU-hours.

\begin{table}[htbp]
    \centering
    \small
    \caption{GRPO and compute-matched RTPO under approximately equal
    training cost.}
    \label{tab:supp_compute_matched}
    \begin{tabular}{lcc}
        \toprule
        Metric & GRPO & RTPO ($G=3$), compute-matched \\
        \midrule
        GPU-hours                    & 15.20   & 15.55 \\
        Pass@1                       & 11.25\% & 11.25\% \\
        Pass@4                       & 25.00\% & \textbf{35.00\%} \\
        Environment completion       & 20.00\% & \textbf{30.00\%} \\
        Generation truncation        & 60.00\% & \textbf{47.50\%} \\
        Total tool calls             & 118     & \textbf{98} \\
        Tool-execution success       & 57.63\% & \textbf{62.24\%} \\
        Average tool calls           & 1.475   & \textbf{1.225} \\
        Average response tokens      & 2,469.7 & \textbf{2,437.7} \\
        \bottomrule
    \end{tabular}
\end{table}

With only approximately 2\% more GPU-hours, RTPO matches GRPO on Pass@1,
improves Pass@4 from 25.00\% to 35.00\%, and increases environment
completion from 20.00\% to 30.00\%. It also reduces truncation, total tool
calls, and average response length while improving tool-execution success.

\paragraph{Matched maximum inference budget.}
All methods below use the same limits of 20 turns, 2,048 tokens per turn, and
8,192 response tokens per trajectory.

\begin{table}[htbp]
    \centering
    \small
    \caption{Performance under the same maximum inference budget.}
    \label{tab:supp_inference_budget}
    \resizebox{\textwidth}{!}{
    \begin{tabular}{lccccc}
        \toprule
        Method & Pass@1 & Pass@4 & Avg.\ tool calls & Avg.\ response tokens &
        Successes/1K tokens \\
        \midrule
        GRPO
        & 11.25\% & 25.00\% & 1.475 & 2,469.7 & 0.046 \\
        Tree-GRPO
        & 16.25\% & 25.00\% & 0.013 & 2,538.2 & 0.064 \\
        \textbf{RTPO ($G=3$)}
        & \textbf{17.50\%} & \textbf{30.00\%} & 1.038 &
        \textbf{1,688.8} & \textbf{0.104} \\
        \bottomrule
    \end{tabular}}
\end{table}

Under the same maximum inference budget, RTPO attains the highest Pass@1,
Pass@4, and number of successful trajectories per 1,000 response tokens,
while producing the shortest average responses. Its gains therefore do not
come from allowing longer test-time trajectories.

\subsection{Tool-Use Behavior on GSM8K}
\label{sec:supp_gsm8k_tools}

The total number of tool calls should be interpreted relative to dataset size:
GSM8K contains 1,319 distinct test problems, while AIME24 and AIME25 each
contain 30. Across all three benchmarks, RTPO improves accuracy while
reducing the total number of calls relative to the vanilla model.

\begin{table}[htbp]
    \centering
    \small
    \caption{Accuracy and total tool calls across mathematical benchmarks.}
    \label{tab:supp_math_tool_calls}
    \begin{tabular}{lcc}
        \toprule
        Benchmark & Vanilla accuracy / calls & RTPO accuracy / calls \\
        \midrule
        AIME24 & 3.33\% / 17   & \textbf{33.33\% / 4} \\
        AIME25 & 13.33\% / 23  & \textbf{26.67\% / 2} \\
        GSM8K  & 76.95\% / 731 & \textbf{94.84\% / 482} \\
        \bottomrule
    \end{tabular}
\end{table}

We further audit the complete trajectories for the first 100 distinct GSM8K
test problems generated by the RTPO Qwen3-4B checkpoint.

\begin{table}[htbp]
    \centering
    \small
    \caption{Audit of Python use in 100 distinct GSM8K trajectories.}
    \label{tab:supp_gsm8k_audit}
    \begin{tabular}{lc}
        \toprule
        Audit metric & Result \\
        \midrule
        Distinct trajectories audited       & 100 \\
        Trajectories using Python            & 36\% \\
        Average calls per trajectory         & 0.36 \\
        Trajectories with zero calls         & 64 \\
        Trajectories with one call           & 36 \\
        Trajectories with two or more calls  & \textbf{0} \\
        Calls with valid Python syntax       & 36/36 \\
        Calls executed successfully          & 36/36 \\
        Correct tool-using trajectories      & 34/36 \\
        Trajectories with repeated calls     & \textbf{0} \\
        \bottomrule
    \end{tabular}
\end{table}

There are no repeated questions and no trajectory invokes Python more than
once. All 36 calls contain valid syntax and execute successfully. In 34 of
the 36 tool-using trajectories, the model has already derived the correct
numerical result before invoking Python. The calls are one-shot arithmetic
checks rather than multi-step tool search, repeated code, or duplicate
execution. For example, on GSM8K QID~55, the model first derives
$30-2=28$ and $28/2=14$, then invokes Python once to verify the result before
returning $\boxed{14}$. Thus, the GSM8K call total reflects many distinct
problems receiving a single low-cost numerical verification, rather than
repeated tool use within a small set of trajectories.

\subsection{On-Policy Continuation and Synchronization Frequency}
\label{sec:supp_sync}

As characterized by Theorem~\ref{thm:onpolicy-app}, standard RTPO synchronizes the rollout model after every reverse stage. When
turn $k$ is optimized, its downstream continuation is therefore generated by
the current policy $\pi_{\theta_{>k}}$, after turns
$k+1,\ldots,K-1$ have been updated. The resulting estimator is
\begin{equation}
    \widehat{Q}^{\mathrm{on}}_{j,k}
    =
    r_{j,k}
    +
    \gamma^{\tau_{j,k}}
    \widehat{F}^{\pi_{\theta_{>k}}}_{j,k},
    \qquad
    \omega_{j,k}\equiv 1 .
    \label{eq:supp_on_policy}
\end{equation}
This construction assigns the turn-$k$ action a return under the current
downstream policy and avoids a product of trajectory-level importance
ratios.

At the opposite synchronization endpoint, the off-policy variant reuses
downstream continuations generated by the initial rollout policy
$\pi_{\theta_0}$ and applies a clamped trajectory-level importance weight:
\begin{align}
    \widehat{Q}^{\mathrm{off}}_{j,k}
    &=
    r_{j,k}
    +
    \gamma^{\tau_{j,k}}
    \overline{\omega}_{j,k}
    \widehat{F}^{\pi_{\theta_0}}_{j,k},
    \label{eq:supp_off_policy}
    \\
    \overline{\omega}_{j,k}
    &=
    \operatorname{clamp}\!\left(
        \prod_{h=k+1}^{K-1}
        \prod_{t=1}^{T_{j,h}}
        \frac{
            \pi_{\theta_{>k}}(a_{j,h,t}\mid s_{j,h,t})
        }{
            \pi_{\theta_0}(a_{j,h,t}\mid s_{j,h,t})
        },
        \rho_{\min},
        \rho_{\max}
    \right).
    \label{eq:supp_off_policy_weight}
\end{align}

\begin{table}[htbp]
    \centering
    \small
    \caption{Output-hit accuracy of off-policy and per-stage on-policy RTPO
    with Qwen3-8B.}
    \label{tab:supp_sync_results}
    \begin{tabular}{lccc}
        \toprule
        Benchmark & Off-policy & On-policy & Difference \\
        \midrule
        GAIA          & 23.30\% & 29.13\% & $+5.83$ pp \\
        GAIA Level 3  &  8.33\% & 16.67\% & $+8.33$ pp \\
        WebWalkerQA   &  6.00\% &  9.50\% & $+3.50$ pp \\
        XBench        &  8.00\% & 15.00\% & $+7.00$ pp \\
        HLE           & 12.45\% & 12.12\% & $-0.33$ pp \\
        \bottomrule
    \end{tabular}
\end{table}

Per-stage on-policy continuation yields consistent gains on the long-horizon
deep-search tasks, with the largest improvement on GAIA Level~3. The two
variants are broadly comparable on HLE, which is dominated by more static and
shorter retrieval. These results localize the cost of stale continuations to
the long-horizon settings targeted by RTPO. Fixed-interval synchronization
and policy-KL-based adaptive synchronization lie between the fully on-policy
and fully off-policy endpoints: they can reduce synchronization cost, but no
longer strictly satisfy $\overline{\omega}_{j,k}\equiv 1$.

\subsection{Budget Consumption Across Training}
\label{sec:supp_budget_dynamics}

The initial policy produces longer trajectories and more tool calls, so early
training stages are more expensive per trunk. The training schedule does not,
however, reserve a fixed number of tool calls for each stage. Every stage
continues under the predefined trunk-sampling and update schedule, and the
actual per-trunk cost decreases as the policy becomes more efficient.

\begin{table}[htbp]
    \centering
    \small
    \caption{Evolution of success, tool use, and generation cost during
    RTPO training.}
    \label{tab:supp_budget}
    \begin{tabular}{ccccc}
        \toprule
        Training point & Task success & Tool calls/trunk &
        Tool-execution success & Generated tokens/trunk \\
        \midrule
        Step 10 & 15.63\% & 7.34 & 74.04\% & 8,013.5 \\
        Step 20 & 37.50\% & 2.06 & 72.73\% & 7,758.8 \\
        Step 30 & 50.00\% & 0.97 & 87.10\% & 3,533.3 \\
        Step 40 & 56.25\% & 1.84 & 94.92\% & 5,469.0 \\
        Step 50 & 40.63\% & 0.94 & 93.33\% & 1,669.3 \\
        \bottomrule
    \end{tabular}
\end{table}

From Step~10 to Step~50, tool calls decrease from 7.34 to 0.94 per trunk,
generated tokens decrease from 8,013.5 to 1,669.3 per trunk, and
tool-execution success increases from 74.04\% to 93.33\%. Tool use therefore
becomes less frequent and more reliable as training progresses.

\subsection{Further Distinctions from Related Methods}
\label{sec:supp_related_work}


\paragraph{SeeUPO.}
SeeUPO and RTPO both use reverse-order sequential updates at the procedural
level, but they differ in motivation, theoretical object, and algorithmic
mechanism. SeeUPO abstracts multi-turn interaction as sequentially executed
multi-agent bandits and uses backward induction to establish monotonic
improvement and global convergence for critic-free backbone algorithms. RTPO
instead begins from three structural inconsistencies in a flattened
multi-turn training pipeline: Rollout--Training Mismatch, Trajectory-Only
Credit Assignment, and Long-Horizon Policy Drift. Its reverse order is one
component of a turn-boundary formulation that is combined with state-matched
sibling comparison and on-policy continuation. The contribution is therefore
not the isolated use of reverse order, but the unified diagnosis,
formalization, algorithm, and guarantees for the three coupled failure
mechanisms.

\paragraph{ArCHer.}
ArCHer addresses delayed reward in long-horizon multi-turn interaction through
a hierarchical actor--critic design. Its high-level component learns
turn-level values with off-policy value-based RL, and its low-level component
uses the critic to train the token-level policy within each turn. RTPO does
not learn an explicit critic; it constructs a turn-level Monte Carlo
advantage from sibling returns sampled from the same boundary state. ArCHer
also does not directly target the rollout--training conditioning mismatch or
the asynchronous downstream-continuation drift analyzed by RTPO.

\paragraph{REFUEL.}
REFUEL and RTPO both avoid an independent critic, but they address different
mismatches. REFUEL uses covariate shift to describe the difference between
training histories generated by a reference policy and deployment histories
generated by the current learner. It iteratively collects self-generated
data and reformulates multi-turn optimization as relative-future regression
tasks. RTPO's Rollout--Training Mismatch occurs within the same sampled batch,
when rollout and likelihood recomputation condition on different contexts
for the same tokens. Thus, REFUEL addresses policy-induced covariate shift
across data-collection stages, whereas RTPO addresses context inconsistency
between rollout and training-time recomputation. The matched
$\tau^3$-Airline results in Table~\ref{tab:supp_airline_test} additionally
show higher final success, more reliable termination, and shorter responses
for RTPO.

\paragraph{R$^3$ and SRL.}
R$^3$ mitigates sparse-reward exploration by moving the curriculum starting
point backward along an expert reasoning trajectory; its reverse mechanism is
a demonstration-based reverse curriculum. SRL also relies on expert
trajectories and derives step-wise supervision from similarity between model
and expert actions. RTPO remains outcome-supervised and on-policy, estimates
turn-level advantages from sibling continuations at the same boundary state,
and jointly addresses conditioning-context mismatch, upstream-state
contamination, and asynchronous continuation drift.


\end{document}